\documentclass[11pt]{article} 
\pdfoutput=1
\usepackage[letterpaper]{geometry}
\usepackage{comment,url,algorithm,algorithmic,relsize}
\usepackage{amssymb,amsfonts,amsmath,amsthm,amscd,dsfont,mathrsfs,mathtools,multirow,microtype,nicefrac,pifont}
\usepackage{float,psfrag,color,url}
\usepackage{upgreek}
\usepackage[dvipsnames]{xcolor}
\usepackage{epstopdf,bbm,mathtools,enumitem}
\usepackage[toc,page]{appendix}
\usepackage[mathscr]{euscript}

\def\balign#1\ealign{\begin{align}#1\end{align}}
\def\baligns#1\ealigns{\begin{align*}#1\end{align*}}
\def\balignat#1\ealign{\begin{alignat}#1\end{alignat}}
\def\balignats#1\ealigns{\begin{alignat*}#1\end{alignat*}}
\def\bitemize#1\eitemize{\begin{itemize}#1\end{itemize}}
\def\benumerate#1\eenumerate{\begin{enumerate}#1\end{enumerate}}

\newenvironment{talign*}
 {\csname align*\endcsname}
 {\endalign}
\newenvironment{talign}
 {\csname align\endcsname}
 {\endalign}

\def\balignst#1\ealignst{\begin{talign*}#1\end{talign*}}
\def\balignt#1\ealignt{\begin{talign}#1\end{talign}}

\let\originalleft\left
\let\originalright\right
\renewcommand{\left}{\mathopen{}\mathclose\bgroup\originalleft}
\renewcommand{\right}{\aftergroup\egroup\originalright}

\def\tinycitep*#1{{\tiny\citep*{#1}}}
\def\tinycitealt*#1{{\tiny\citealt*{#1}}}
\def\tinycite*#1{{\tiny\cite*{#1}}}
\def\smallcitep*#1{{\scriptsize\citep*{#1}}}
\def\smallcitealt*#1{{\scriptsize\citealt*{#1}}}
\def\smallcite*#1{{\scriptsize\cite*{#1}}}

\def\<{\left\langle} 
\def\>{\right\rangle}

\DeclareSymbolFont{rsfs}{U}{rsfs}{m}{n}
\DeclareSymbolFontAlphabet{\mathscrsfs}{rsfs}

\ifdefined\nonewproofenvironments\else
\ifdefined\ispres\else
\newtheorem{theorem}{Theorem}
\newtheorem{lemma}[theorem]{Lemma}
\newtheorem{corollary}[theorem]{Corollary}
\newtheorem{definition}[theorem]{Definition}
\newtheorem{fact}[theorem]{Fact}
\renewenvironment{proof}{\noindent\textbf{Proof.}\hspace*{.3em}}{\qed \vspace{.1in}}
\newenvironment{proof-sketch}{\noindent\textbf{Proof Sketch}
  \hspace*{1em}}{\qed\bigskip\\}
\newenvironment{proof-idea}{\noindent\textbf{Proof Idea}
  \hspace*{1em}}{\qed\bigskip\\}
\newenvironment{proof-of-lemma}[1][{}]{\noindent\textbf{Proof of Lemma {#1}}
  \hspace*{1em}}{\qed\\}
  \newenvironment{proof-of-proposition}[1][{}]{\noindent\textbf{Proof of Proposition {#1}}
  \hspace*{1em}}{\qed\\}
\newenvironment{proof-of-theorem}[1][{}]{\noindent\textbf{Proof of Theorem {#1}}
  \hspace*{1em}}{\qed\\}
\newenvironment{proof-attempt}{\noindent\textbf{Proof Attempt}
  \hspace*{1em}}{\qed\bigskip\\}

\newtheorem*{remark*}{Remark}
\newenvironment{remark}{\noindent\textbf{Remark.}
  \hspace*{0em}}{\smallskip}

\fi

\newtheorem{proposition}[theorem]{Proposition}

\newtheorem{assumption}{Assumption}
\theoremstyle{definition}
\newtheorem{example}[theorem]{Example}

\fi
\makeatletter
\@addtoreset{equation}{section}
\makeatother

\definecolor{OliveGreen}{rgb}{0,0.6,0}

\usepackage[normalem]{ulem}

\RequirePackage[authoryear]{natbib}
\setcitestyle{authoryear,open={(},close={)}} 
\RequirePackage[colorlinks,citecolor=blue,urlcolor=blue]{hyperref}
\RequirePackage{authblk}

\makeatletter
\renewcommand{\paragraph}{%
  \@startsection{paragraph}{4}%
  {\z@}{1.25ex \@plus 1ex \@minus .2ex}{-1em}%
  {\normalfont\normalsize\bfseries}%
}
\makeatother

\usepackage{amsfonts,amscd,dsfont,mathrsfs,mathtools,microtype,nicefrac,pifont}

\usepackage{subfigure}
\usepackage{mathrsfs}
\usepackage{float}
\usepackage{makecell}
\usepackage{multirow}
\usepackage{enumitem}
\usepackage{hyperref}
\usepackage[toc,page]{appendix}
\usepackage{algorithm,algorithmic}
\usepackage{color}

\allowdisplaybreaks

\DeclareMathOperator*{\argmin}{arg\,min}

\DeclareMathOperator{\grad}{ grad\,}

\DeclareMathOperator{\Div}{div\,}

\DeclareMathOperator{\tr}{tr\,}
\DeclareMathOperator{\Ric}{Ric\,}
\DeclareMathOperator{\Cut}{Cut\,}
\DeclareMathOperator{\ID}{ID\,}
\DeclareMathOperator{\Poly}{Poly\,}

\DeclareMathOperator{\Exp}{Exp\,}
\DeclareMathOperator{\Log}{Log\,}
\DeclareMathOperator{\prox}{prox\,}

\usepackage{times}

\title{\textrm{Riemannian Proximal Sampler\\ for High-accuracy Sampling on Manifolds}}

\author[1]{Yunrui Guan}
\author[2]{Krishnakumar Balasubramanian}
\author[1]{Shiqian Ma}
\affil[1]{Department of Computational Applied Mathematics and Operations Research, Rice University.}
\affil[2]{Department of Statistics, University of California, Davis.}
\affil[1]{\texttt{\{yg83,sqma\}}@rice.edu}
\affil[2]{\texttt{\{kbala\}}@ucdavis.edu}
\date{}

\begin{document}

\maketitle

\begin{abstract}

We introduce the \textit{Riemannian Proximal Sampler}, a method for sampling from densities defined on Riemannian manifolds. The performance of this sampler critically depends on two key oracles: the \textit{Manifold Brownian Increments (MBI)} oracle and the \textit{Riemannian Heat-kernel (RHK)} oracle. We establish high-accuracy sampling guarantees for the Riemannian Proximal Sampler, showing that generating samples with \(\varepsilon\)-accuracy requires \(\mathcal{O}(\log(1/\varepsilon))\) iterations in Kullback-Leibler divergence assuming access to exact oracles and \(\mathcal{O}(\log^2(1/\varepsilon))\) iterations in the total variation metric assuming access to sufficiently accurate inexact oracles. Furthermore, we present practical implementations of these oracles by leveraging heat-kernel truncation and Varadhan’s asymptotics. In the latter case, we interpret the Riemannian Proximal Sampler as a discretization of the entropy-regularized Riemannian Proximal Point Method on the associated Wasserstein space. We provide preliminary numerical results that illustrate the effectiveness of the proposed methodology.

\end{abstract}

\section{Introduction}


We consider the problem of sampling from a density $\pi^{X} \propto e^{-f}$ defined on a Riemannian manifold $(M,g)$, where $g$ is the metric on the manifold $M$. Here, the density is defined with respect to the volume measure $dV_g$ and the normalization constant $\int_M e^{-f} dV_g <\infty$ is unknown. Riemannian sampling arises in various domains. In Bayesian inference, it is used for sampling from distributions with complex geometries, such as those encountered in hierarchical Bayesian models, latent variable models, and machine learning applications like Bayesian deep learning~\citep{girolami2011riemann, byrne2013geodesic, patterson2013stochastic, liu2018riemannian, arnaudon2019irreversible, liu2016stochastic, piggott2016geometric, muniz2022higher, lie2023dimension}. In statistical physics, it plays a crucial role in simulating molecular systems with constrained dynamics~\citep{leimkuhler2016efficient}. Additionally, it appears in optimization problems over manifolds, including eigenvalue problems and low-rank matrix approximations~\citep{goyal2019sampling,li2023riemannian, yu2023riemannian, bonet2023spherical} and as a module in Riemannian diffusion models~\citep{de2022riemannian,huang2022riemannian}.

On a Riemannian manifold, Langevin dynamics has the form $dX_{t} = - \grad f(X_{t})dt + \sqrt{2} dB_{t} $ where $\grad$ represents the Riemannian gradient and $B_{t}$ is the manifold Brownian motion. This formulation extends Euclidean Langevin dynamics by incorporating geometric information through the Riemannian metric, enabling more efficient exploration of curved probability landscapes. Unlike the Euclidean case, discretizing manifold Brownian motion is non-trivial, except in a few special cases. \cite{li2023riemannian} considered the case of $M \equiv \mathcal{S}^d\times\cdots\times\mathcal{S}^d$ (i.e., finite product of spheres) and established convergence rates for a simple discretization scheme that discretizes only the drift (gradient term) while requiring exact implementation of manifold Brownian motion increments--feasible on the sphere. \cite{gatmiry2022convergence} extended this approach to general Hessian manifolds, proving convergence results under the same assumption of exact Brownian motion implementation, which is generally infeasible. Both works require the target density to satisfy a logarithmic Sobolev inequality and establish iteration complexity of 
$\text{poly}(1/\varepsilon)$  to obtain an $\varepsilon$-approximate sample in KL-divergence. However, the reliance on exact Brownian motion increments significantly limits the applicability of the results in~\cite{gatmiry2022convergence}.

\cite{cheng2022efficient} studied a practical discretization of Riemannian Langevin diffusion, where both the drift and Brownian motion are discretized. They established an iteration complexity in the 1-Wasserstein distance under a general assumption and in the 
2-Wasserstein distance under a more restrictive condition, which can be seen as an analog of log-concavity. Their complexity is of order $\tilde{\mathcal{O}}(1/\varepsilon^{2})$. A key technical challenge is that, in the absence of convexity (e.g., on a compact manifold), establishing contractivity under the Wasserstein distance is nontrivial -- even for the continuous-time dynamics. This difficulty is overcome through a second-order expansion of the Jacobi equation~\citet[Lemma 29]{cheng2022efficient}. \cite{kong2024convergence} introduced the Lie-group MCMC sampler for sampling from densities on Lie groups, with a primary focus on accelerating sampling algorithms. Their iteration complexity for the 2-Wasserstein distance are also $\text{poly}(1/\varepsilon)$.

\begin{table}[t]
\label{tab:t2}
\begin{centering}
    {\renewcommand{\arraystretch}{1.2}%
    \begin{tabular}{|c|c|c|c|c|}
    \hline Assumption & Source & Setting & Complexity & Metric \\ \hline \hline
    \multirow{1}{*}{\begin{tabular}[c]{@{}c@{}} 
    Log-Sobolev Inequality (LSI) \end{tabular}} 
    & Theorem \ref{Main_Theorem} & Exact MBI, Exact RHK  & $\tilde{\mathcal{O}}(\log \frac{1}{\varepsilon})$ & KL \\ \hline
    \multirow{1}{*}{\begin{tabular}[c]{@{}c@{}} 
    LSI \& Assumption 1 \end{tabular}} 
    & Theorem \ref{TV_Inexact_BM_Inexact_RHK} & Inexact MBI, Inexact RHK  & $\tilde{\mathcal{O}}(\log^2 \frac{1}{\varepsilon})$ & TV \\ \hline
      \end{tabular}
    }
\par\end{centering}
    \caption{A summary of iteration complexity results in this work. Here,  $\varepsilon$ represents the target accuracy. 
    The $\tilde{\mathcal{O}}$ notation hides dependency on all other parameters except for $\varepsilon$.}
\end{table}

In comparison to the above works for Riemannian sampling, for the Euclidean case, high-accuracy algorithms, i.e., algorithms with iteration complexity of $O(\text{polylog}(1/\varepsilon))$ are available under various assumptions (that are essentially based on (strong) log-concavity or isoperimetry); see for example~\cite{lee2021structured,chen2022improved,fan2023improved,he2024separation} for such results for the Euclidean proximal sampler and~\cite{dwivedi2019log,chen2020fast,chewi2021optimal,lee2020logsmooth,wu2022minimax,chen2023does,andrieu2024explicit,altschuler2024faster} for various Metropolized algorithms including Metropolis Random Walk (MRW), Metropolis Adjusted Langevin Algorithm (MALA) and Metropolis Hamiltonian Monte Carlo (MHMC). 

High-accuracy samplers for constrained (Euclidean) sampling, i.e., when the density is supported on convex set $\mathcal{K} \subseteq \mathbb{R}^{d}$ are established for Hit-and-run and Ball-walk based algorithms under various assumptions~\citep{lovasz1999hit,kannan2006blocking,kannan1997random}; see~\citet[Section 1.3]{kook2025renyi} for a detailed overview of related works. \cite{kook2022sampling}, proposed Constrained Riemannian Hamiltonian Monte Carlo (CRHMC), and used Implicit Midpoint Method to integrate the Hamiltonian dynamics and established a high-accuracy guarantee for discretized CRHMC. \cite{noble2023unbiased} proposed Barrier Hamiltonian Monte Carlo (BHMC) for constrained sampling, along with its discretizations and established asymptotic results. \cite{kook2024and} proposed the ``In-and-Out" sampling algorithm that has high-accuracy guarantees for sampling uniformly on a convex body. Recently~\cite{kook2024sampling} obtained state-of-the-art results for sampling from log-concave densities on convex bodies using a proximal sampler designed for this problem.~\cite{srinivasan2024fast} and~\cite{srinivasan2024high} showed that a Metropolized version of the Mirror and preconditioned Langevin Algorithm obtains high-accuracy guarantees, respectively, under certain assumptions.

Given the above, the following natural question arises: 
\begin{center}
   \emph{Can one develop high-accuracy algorithms for sampling on Riemannian manifolds?} 
\end{center}
To the best of our knowledge, no prior work exists on providing an affirmative answer to this question. In this work, we develop the \emph{Riemannian Proximal Sampler} which generalizes the Euclidean Proximal Sampler from~\cite{lee2021structured}. In contrast to the Euclidean case, the algorithm is based on the availability of two oracles: the Manifold Brownian Increment (MBI) oracle and the Riemannian Heat Kernel (RHK) oracle. We show in Theorem~\ref{Main_Theorem} that when the exact oracles available the algorithm achieves high-accuracy guarantee in the Kullback-Liebler divergence. Under the availability of inexact oracles, as characterized in Assumption~\ref{Assumption_Oracle_TV_quality}, we show in Theorem~\ref{TV_Inexact_BM_Inexact_RHK} that the algorithm still achieves high-accuracy guarantees in the total variation metric. Our results are summarized in Table~\ref{tab:t2}. We further develop practical implementations of the aforementioned oracles that satisfy the conditions in Assumption~\ref{Assumption_Oracle_TV_quality} (Section~\ref{Section_Oracle}), and that are connected to entropy-regularized proximal point method on Wasserstein spaces (Section~\ref{Section_Proximal_point_approximation}). We also demonstrate the numerical performance of the algorithms via simulations in Appendix~\ref{sec:sim}.

\section{Preliminaries}\label{prelim}
We first recall certain preliminaries on Riemannian manifolds; additional preliminaries are provided in Appendix~\ref{sec:addprelim}. We refer the readers to~\cite{lee2018introduction} for more details. 

Let $M$ be a Riemannian manifold of dimension $d$ equipped with metric $g$. The manifold $M$ is assumed to be complete, connected Riemannian manifold without boundary. For a point $x \in M$, $T_{x}M$ denotes the tangent space at $x$. For any $v, w \in T_{x}M$, we can write the metric as  $g_{x}(v, w) = \langle v, w\rangle_{g}$. For $x \in M$ and $v \in T_{x}M$, $\exp_{x}(v)$ denotes the exponential map. We use $\grad$ and $dV_{g}$ to represent the Riemannian gradient and the Riemannian volume form respectively.  

For $x \in M$, $\Cut(x)$ denotes the cut locus of $x$.
For $x, y \in M$, we use $d(x, y)$ to denote the geodesic distance between $x$ and $y$.
Let $\Div$ denotes the Riemannian divergence, and Laplace-Beltrami operator $\Delta : C^{\infty} (M) \to C^{\infty} (M)$ 
is defined as the Riemannian divergence of Riemannian gradient: $ \Delta u = \Div (\grad u)$. We use $\nu(t, x, y)$ to denote the density of manifold Brownian motion with time $t$, starting at $x$, evaluated at $y$. 




Let $(M, \mathcal{F})$ be a measurable space.
Note that the Riemannian volume form $dV_{g}$ is a measure. A probability measure $\rho$ and its corresponding probability density function $p$ are related through $d\rho = p dV_{g}$. Given a measurable set $A \in \mathcal{F}$, $P_{\rho}(A)$ denotes the probability assigned to the set $A$ by $\rho$. We have $P_{\rho}(A) = \int_{A} p(x) dV_{g}(x) = \int_{A} d\rho(x)$.

\begin{definition}[TV distance]
Let $\rho_{1}, \rho_{2}$ be probability measures defined on the measurable space $(M, \mathcal{F})$. The total variation distance between $\rho_{1}$ and $\rho_{2}$ 
is defined as 
\begin{equation*}
    \|\rho_{1} - \rho_{2}\|_{TV} := \sup_{A \in \mathcal{F}} |\rho_{1}(A) - \rho_{2}(A)|.
\end{equation*}
\end{definition}

\begin{definition}[KL divergence] Let $\rho_{1}, \rho_{2}$ be probability measures on the measurable space $(M, \mathcal{F})$, with full support. 
The Kullback-Leibler (KL) divergence of $\rho_{1}$ with respect to $\rho_{2}$ is defined as
\begin{equation*}
    H_{\rho_{2}}(\rho_{1}) := \int_{M} \log \frac{d\rho_{1}}{d\rho_{2}} d\rho_{1},
\end{equation*}
where $\frac{d\rho_{1}}{d\rho_{2}}$ is the Radon-Nikodym derivative.
\end{definition}

It is known that $H_{\rho_{2}}(\rho_{1}) \ge 0$ with equality if and only if $\rho_{1} = \rho_{2}$. 
Although the KL divergence is not symmetric, it serve as a ``distance'' function between two probability measures. For instance, the well known Pinsker inequality states that $\|\rho_{2} - \rho_{1}\|_{TV}^{2} \le \frac{1}{2} H_{\rho_{2}}(\rho_{1})$. 


\begin{definition}[Log-Sobolev Inequality (LSI)]
A probability measure $\rho_{2}$ satisfies Log-Sobolev Inequality with parameter $\alpha > 0$ ($\alpha$-$\mathsf{LSI}$) 
if $H_{\rho_{2}}(\rho_{1}) \le \frac{1}{2\alpha} J_{\rho_{2}}(\rho_{1}), \forall \rho_{1}$, where $J_{\rho_{2}}(\rho_{1}) := \int_{M} \|\grad \log \frac{\rho_{1}}{\rho_{2}}\|^{2} d\rho_{1}$ is the relative Fisher information.
\end{definition}

For more details on LSI, see Appendix \ref{LSI_Heat}.
In Euclidean space, such a condition is a relaxation of strongly convex assumption, and is used to establish convergence of sampling algorithms in KL divergence.
See, for example, \cite{vempala2019rapid} (for the Langevin Monte Carlo Algorithm) and \cite{chen2022improved} (for the Euclidean proximal sampler).

\subsection{Curvature}\label{sec:ricci}
We also need notions of curvature on manifolds to present our main results. Let $\mathfrak{X}(M)$ denote the set of all smooth vector fields on $M$.
Define a map called Riemann curvature endomorphism by $R: \mathfrak{X}(M) \times \mathfrak{X}(M) \times \mathfrak{X}(M) \to \mathfrak{X}(M)$ by  $R(X, Y)Z = \nabla_{X} \nabla_{Y} Z - \nabla_{Y} \nabla_{X} Z - \nabla_{[X, Y]} Z.$ While such definition is very abstract, we provide an intuitive explanation of what curvature is.
Intuitively, on a manifold of positive curvature (say, a $2$-dimensional sphere), geodesics tend to ``contract". More precisely, given $x, y \in M$ and $v \in T_{x}M$, we can parallel transport $v$ to $u = P_{x}^{y}v \in T_{y}M$. It is a well-known result that (ignore higher order terms) $d(\exp_{x}t v, \exp_{y} tu) \le (1-\frac{t^{2}}{2}K)d(x, y)$ for some $K$ (which is actually the sectional curvature). From this, we see that for positive curvature, which means $K > 0$, the distance between geodesics would decrease. 

Formally, given $v, w \in T_{p}M$ being linearly independent, the sectional curvature of the plane spanned by $v$ and $w$ can be computed through  $K(v, w) = \frac{\langle R(v, w)w, v \rangle}{|v|^{2}|w|^{2} - \langle v, w \rangle^{2}} $; see  \citet[Proposition 8.29]{lee2018introduction}. 
On the other hand, Ricci curvature can be viewed as the average of sectional curvatures. The Ricci curvature at $x \in M$ along direction $v$ is denoted as $\Ric_{x}(v)$, which is equal to the sum of the sectional curvatures of the 2-planes spanned by $(v, b_{i})_{i = 2}^{d}$ where $v, b_{2}, ..., b_{d}$ is an orthonormal basis for $T_{x}M$; see
\citet[Proposition 8.32]{lee2018introduction}. 

We remark that the Ricci curvature is actually a symmetric 2-tensor field defined as the trace of the
curvature endomorphism on its first and last indices \citep{lee2018introduction}, which sometimes is written as $\Ric_{x}(u, v)$ for $u, v \in T_{x}M$. The previous notation is a shorthand of $\Ric_{x}(v) = \Ric_{x}(v, v)$. When we say Ricci curvature is lower bounded by $\kappa$, we mean $\Ric(v, v) \ge \kappa, \forall v \in T_{x}M, \|v\| = 1$. We end this subsection through some concrete examples. 
\begin{enumerate}
    \item The hypersphere $\mathcal{S}^{d}$ has constant sectional curvature equal to $1$, and constant Ricci curvature $\Ric = (d-1) g, \forall x \in M$ (so that $\Ric_{x}(v) = d-1$ for all unit tangent vector $v \in T_{x}M$).
    \item For $P_{m} \subseteq \mathbb{R}^{m \times m}$, the manifold of positive definite matrices, its sectional curvatures are in the interval $[-\frac{1}{2}, 0]$; see, for example, \cite{criscitiello2023accelerated}. Hence its Ricci curvature is lower bounded by $-\frac{m(m+1)-1}{4}$.
\end{enumerate}

\subsection{Brownian Motion on Manifolds}

Now we briefly discuss Brownian motion on a Riemannian manifold. Recall that in Euclidean space, Brownian motion is described by the Wiener process. Given $x \in \mathbb{R}^{d}$ and $t > 0$, the Brownian motion starting at $x$ with time $t$ has (a Gaussian) density function 
$\nu(t, x, y) = \frac{1}{(2\pi t)^{d/2}}e^{-\frac{\|x - y\|^{2}}{2t}}$.
It solves the heat equation $\frac{\partial}{\partial t} \nu(t, x, y) = \frac{1}{2} \Delta_{y} \nu(t, x, y)$ 
with initial condition $\nu(0, x, y) = \delta_{x}(y)$. 

On a Riemannian manifold, we can describe the density of Brownian motion (heat kernel) through heat equation.
Let $B_{x, t}$ be a random variable denoting manifold Brownian motion starting at $x$ with time $t$ 
and let $\nu(t, x, y)$ be the density of $B_{x, t}$.
The Brownian motion density $\nu(t, x, y)$ is then defined as the minimal solution of the following heat equation:
\begin{equation*}
    \begin{split}
        \frac{\partial}{\partial t} \nu(t, x, y) = \frac{1}{2} \Delta_{y} \nu(t, x, y) \qquad\text{with}\qquad
        \nu(0, x, y) = \delta_{x}(y).
    \end{split}
\end{equation*}
More details can be found in \citet[Chapter 4]{hsu2002stochastic}.
Unlike the Euclidean case, on Riemannian manifold, the heat kernel does not have a closed-form solution in general. However, some properties of the Euclidean hear kernel is preserved on a Riemannian manifold. One such property is the following: 
Consider $M = \mathbb{R}^{d}$ we have $t \log \nu(t, x, y) = t \log \frac{1}{(2\pi t)^{d/2}} - \frac{\|x - y\|^{2}}{2}$. 
As $t \to 0$, we get $\lim_{t \to 0} t \log \nu(t, x, y) = - \frac{\|x - y\|^{2}}{2}$.
On a Riemannian manifold, we have the following result.
\begin{fact}[Varadhan's asymptotic relation
    \citep{hsu2002stochastic}]
    For all $x, y \in M$ with $y \notin \Cut(x)$, we have 
    \begin{enumerate}
        \item[] $\underset{t \to 0}{\lim}~t \log \nu(t, x, y) = - \frac{d(x, y)^{2}}{2}$\quad and\quad$\lim_{t \to 0} t \grad_{y} \log \nu(t, x, y) = \exp_{y}^{-1}(x)$.
    \end{enumerate}
\end{fact}
When evaluation of the heat kernel is required for practical applications, the Varadhan asymptotics aforementioned is used \citep{de2022riemannian}. We illustrate the above with the following example.
\begin{example}
    For $M = \mathcal{S}^{1} \subseteq \mathbb{R}^{2}$, the heat kernel for time $t$ only depends on the spherical distance but not specific points.
    Hence we simply write $\nu_{t}(\varphi) = \nu_{t}(d(x, y)) = \nu(t, x, y)$ where $\varphi = d(x, y)$ is the geodesic distance between $x$ and $y$.
    We have $\nu_{t}(\varphi) = \sum_{n \in \mathbb{Z}} \frac{1}{\sqrt{2\pi t}} \exp(- \frac{(\varphi+2\pi n)^{2}}{2t})$; see, for example, \cite{andersson2013estimates}.
    Here $t$ represent the time of Brownian motion and $\varphi = d(x, y)$ represent the spherical distance between $x, y$. 
    When $x$ is not too large, terms corresponding to $n = 0$ would dominate the sum. 
    Thus we can write 
    $\nu(t, x, y) \approx \frac{1}{\sqrt{2\pi t}} \exp(-\frac{d(x, y)^{2}}{2t} )$ which recovers Varadhan's asymptotics. 
\end{example}

Yet another numerical method for evaluating the heat kernel on manifold is truncation method; see, for example, \citet[Section 5.1]{corstanje2024simulating} and \cite{de2022riemannian}. In many cases, the heat-kernel has an infinite series expansion. For example, a power series expansion of heat kernel on hypersphere is given in \citet[Theorem 1]{zhao2018exact}, 
and more examples can be found in \citet[Example 1-5]{eltzner2021diffusion}. 
Similar results are also available for more general manifolds; see, for example, \cite{azangulov2022stationary} for compact 
Lie groups and their homogeneous space, 
and \cite{azangulov2024stationary} for non-compact symmetric spaces. Hence, a natural approach is to truncate this infinite series at an appropriate level. For example, on $\mathcal{S}^{2} \subseteq \mathbb{R}^{3}$, the heat kernel and its truncation up to the $l$-th term (denoted as $\nu_{l}$) can be written respectively as 
\begin{equation*}
    \nu(t, x, y) = \sum_{i = 0}^{\infty} e^{-\frac{i(i+1)t}{2}} \frac{2i+1}{4\pi} P_{i}^{0}(\langle x, y \rangle_{\mathbb{R}^{3}})~~\text{and}~~\nu_{l}(t, x, y) = \sum_{i = 0}^{l} e^{-\frac{i(i+1)t}{2}} \frac{2i+1}{4\pi} P_{i}^{0}(\langle x, y \rangle_{\mathbb{R}^{3}}),
\end{equation*}
where $P_{i}^{0}$ are Legendre polynomials.

\section{The Riemannian Proximal Sampler}

\begin{algorithm}[t]
    \begin{algorithmic}
    \FOR{$k=0, 1,2,...$}
    \STATE \textbf{Step 1 (MBI):} From $x_{k}$, sample $y_{k} \sim \pi^{Y|X}(\cdot, x_{k})$ which is a manifold Brownian increment. 
    \STATE \textbf{Step 2 (RHK):} From $y_{k}$, sample $x_{k+1} \sim \pi^{X|Y}(\cdot, y_{k}) \propto e^{-f(x)} \nu(\eta, x, y_{k}) $.
    \ENDFOR
    \end{algorithmic}
    \caption{Riemannian Proximal Sampler}
    \label{Manifold_Proximal_Sampler_Ideal} 
\end{algorithm}

We now describe the Riemannian Proximal Sampler, introduced in Algorithm \ref{Manifold_Proximal_Sampler_Ideal}. Similar to the Euclidean proximal sampler~\citep{lee2021structured}, the algorithm has two steps. The first step is sampling from the Manifold Brownian Increment (MBI) oracle. The second step is called the Riemannian Heat-Kernel (RHK) Oracle. Recall that $\nu(\eta, x, y)$ denotes the density of manifold Brownian motion with time $\eta$. Define a joint distribution $\pi_{\eta}(x, y) \propto e^{-f(x)} \nu(\eta, x, y)$. Then, step 2 consists of sampling from the aforementioned distribution. When there is no ambiguity, we omit the step size $\eta$ and simply write $\pi(x, y) \propto e^{-f(x)} \nu(\eta, x, y)$. Algorithm \ref{Manifold_Proximal_Sampler_Ideal} is an idealized algorithm, in the sense that we assume exact access to MBI and RHK oracles. Following~\cite{chen2022improved}, next we provide an intuitive explanation for the algorithm from a diffusion process perspective. 

\paragraph{Step 1:} For fixed $x$, we see that $\pi^{Y|X}(\cdot, x) \propto \nu(\eta, x, \cdot)$ which is the density of Brownian motion starting from $x$ for time $\eta$.
From this we see that the first step of the algorithm is running forward manifold heat flow: $dZ_{t} = dB_{t}$. 

\paragraph{Step 2:} We will illustrate that the second step of the algorithm is running the time-reversed process of the forward process. Consider a stochastic process $Z_{t}: t \ge 0$. 
When we have observations of $x_{\eta} \sim Z_{\eta}$, we can compute the conditional probability of $Z_{0}$ condition on end point $Z_{\eta}$. 
We denote $\mu(x_{0}|x_{\eta})$ as the posterior.
Bayes Theorem says $\mu(x_{0}|x_{\eta}) \propto \mu(x_{0}) L(x_{\eta}|x_{0})$,
where $\mu(x_{0})$ is the prior guess and the likelihood $L$ depends on the model.
We consider the following model (forward heat flow): $dZ_{t} = dB_{t}$ with $Z_{0} \sim \pi^{X} \propto e^{-f(x)}$. 
Then $\mu(x_{0}) = \pi^{X}(x_{0})$ and $L(x_{\eta}|x_{0}) = \nu(\eta, x_{0}, x_{\eta})$.
Thus we get $\mu(x_{0}|x_{\eta}) \propto e^{-f(x_{0})} \nu(\eta, x_{0}, x_{\eta})$, and we observe that $\mu(x_{0}|x_{\eta})$ is exactly $\pi^{X|Y = x_{\eta}}(x_{0}|x_{\eta})$. For the forward heat flow $dZ_{t} = dB_{t}$ with initialization $Z_{0} \sim \pi^{X} \propto e^{-f(x)}$,
there is a well-defined time reversed process $\hat{Z}_{t}^{-}$, 
which satisfies $(Z_{0}, Z_{\eta}) \overset{d}{=} (\hat{Z}_{\eta}^{-}, \hat{Z}_{0}^{-})$. See Appendix \ref{Section_Backward} for more details.
Based on this, for the time-reversed process $\hat{Z}_{t}^{-}$, 
the law of $\hat{Z}_{\eta}^{-}$ condition on $\hat{Z}_{0}^{-} = z$ is
the same as the posterior $\mu(x|z)$ discussed previously, i.e., $\pi^{X|Y = z}(x) \propto e^{-f(x)} \nu(\eta, x, z)$. 
Thus we see that the RHK oracle is, from a diffusion perspective, running the time-reversed process.

Implementing Step 1 and Step 2 is non-trivial on Riemannian manifolds. In Sections ~\ref{Section_Oracle} and~\ref{Section_Proximal_point_approximation} respectively, we discuss two approaches based on heat-kernel truncation and Varadhan's asymptotics. Furthermore, geodesic random walk~\citep{mangoubi2018rapid,schwarz2023efficient} is a popular approach to simulate Manifold Brownian Increments (see~Appendix~\ref{georw}), however to the best of our knowledge (in various metrics of interest) is known only under strong assumptions~\citep{cheng2022efficient,mangoubi2018rapid}.

\section{High-Accuracy Convergence Rates}

In this section, we provide the convergence rates for the Riemannian Proximal Sampler (Algorithm \ref{Manifold_Proximal_Sampler_Ideal}) assuming that the target density satisfies the LSI assumption. Firstly, note that in \citep{lee2021structured} the analysis of Euclidean Proximal Sampler is done assuming the potential function is strongly convex. However, it is known that on a compact manifold, if a function is geodesically convex, then it has to be a constant. Hence assuming the potential $f$ being geodesically convex is not much meaningful. Recently, \cite{cheng2022efficient} discussed an analog of log-concave distribution on manifolds.
Although their setting works for compact manifolds, it requires the Riemannian Hessian of the potential $f$ to be lower 
bounded by some curvature-related value, which is still restrictive. Hence, we adopt the setting as in \cite{chen2022improved}, assuming that the target distribution satisfies the LSI.

In Section \ref{Sec_Exact_oracle}, we consider the case where both steps of Algorithm \ref{Manifold_Proximal_Sampler_Ideal} are implemented exactly, and in Section \ref{Sec_Inexact_oracle}, we consider the case when MBI and RHK oracles are inexact. Regarding notation, we let $\rho_{k}^{X}(x)$, $\rho_{k}^{Y}(y)$ denote the law of $x$ and $y$ generated by Algorithm \ref{Manifold_Proximal_Sampler_Ideal} at $k$-th iteration, 
assuming exact MBI and exact RHK oracles. When the oracles are inexact, we let $\tilde{\rho}_{k}^{X}(x)$, $\tilde{\rho}_{k}^{Y}(y)$ to denote the law of $x$ and $y$ generated by Algorithm \ref{Manifold_Proximal_Sampler_Ideal} at $k$-th iteration.

\subsection{Rates with Exact Oracles}\label{Sec_Exact_oracle}
Our first result is as follows, with the proof provided in Appendix \ref{Sec_Proof_Main_Theorem}. 
\begin{theorem}\label{Main_Theorem}
    Let $M$ be a Riemannian manifold without boundary, i.e., $\partial M = \emptyset$. Assume $\pi^{X}$ satisfies $\alpha$-$\mathsf{LSI}$. 
    Denote the distribution for the $k$-th iteration of Algorithm \ref{Manifold_Proximal_Sampler_Ideal} as $x_{k} \sim \rho_{k}^{X} $.
    For any initial distribution $\rho_{0}^{X}$, for all $\eta > 0$, we have
    \begin{equation*}
        \begin{split}
            H_{\pi^{X}} (\rho_{k}^{X} ) &\le \frac{H_{\pi^{X}} (\rho_{0}^{X})}{(1 + \eta\alpha)^{2k} }, \qquad\text{if the Ricci curvature is non-negative,} \\
            H_{\pi^{X}} (\rho_{k}^{X} ) &\le H_{\pi^{X}} (\rho_{0}^{X}) \left(\frac{\kappa}{\alpha(e^{\kappa \eta} - 1) + \kappa}\right)^{2k}, \qquad\text{otherwise,} \\
        \end{split}
    \end{equation*}
    where $\kappa$ is the lower bound of Ricci curvature.
    In case of negative curvature, we have
    \begin{align*}
    H_{\pi^{X}} (\rho_{k}^{X} ) \le \frac{H_{\pi^{X}} (\rho_{0}^{X})}{(1 + \eta\alpha)^{2k} },\qquad~\text{if}~~\eta \le \frac{1}{|\kappa|}.
    \end{align*}
\end{theorem}

Note that the resulting contraction rate depends on the curvature. If the curvature is non-negative, then we can recover the rate in Euclidean space. But in the case of negative curvature, the rate becomes more complicated, and in order to get the contraction rate as in Euclidean space, we need the step size to be bounded above by some curvature-dependent constant. 

The above result provides a high-accuracy guarantee for the Riemannian Proximal Sampler in KL-divergence. To see that, consider the case when the Ricci curvature is non-negative. Note that to achieve $\varepsilon$ accuracy in KL divergence, we need $\frac{H_{\pi^{X}} (\rho_{0}^{X})}{(1 + \eta\alpha)^{2k} } = \varepsilon$. Taking $\log$ on both sides, we get 
    $k = \mathcal{O}(\frac{\log (H_{\pi^{X}} (\rho_{0}^{X})/\varepsilon)}{\log (1 + \eta\alpha)}) $.
    For small step size $\eta$, we have $\frac{1}{\log(1 + \eta \alpha)} = \mathcal{O}(\frac{1}{\eta \alpha})$. 
    Hence $k = \mathcal{O}(\frac{1}{\eta \alpha} \log \frac{H_{\pi^{X}} (\rho_{0}^{X})}{\varepsilon}) = \tilde{\mathcal{O}}(\frac{1}{\eta}\log \frac{1}{\varepsilon})$. As $\eta$ does not depend on $\varepsilon$, we see that we need $\tilde{\mathcal{O}}(\log \frac{1}{\varepsilon})$ number of iterations.

There are several challenges in obtaining the aforementioned result for the Riemannian Proximal Sampler. In Euclidean space, when a probability distribution $\pi^{X}$ satisfies $\alpha$-$\mathsf{LSI}$, 
its propagation along heat flow $\pi^{X} * \mathcal{N}(0, tI_{d})$ satisfies $\alpha_{t}$-$\mathsf{LSI}$, 
with $\alpha_{t} = \frac{\alpha}{1 + \alpha t}$. 
This fact is very important and leveraged in~\cite{chen2022improved} for proving their convergence rates. A quantitative generalization of such a fact for Riemannian manifolds is not immediate and we establish the required results in Appendix \ref{LSI_Heat}, following \cite{collet2008logarithmic}, under the required Ricci curvature assumptions.


\subsection{Rates with Inexact Oracles}\label{Sec_Inexact_oracle}
Recall that Algorithm \ref{Manifold_Proximal_Sampler_Ideal} is an idealized algorithm, 
where we assumed the availability of the MBI and RHK oracles. Note that given $x \in M$, exact MBI oracle generate samples $y \sim \pi_{\eta}^{Y|X}(\cdot|x)$. 
And given $y \in M$, exact RHK generate samples $x \sim \pi_{\eta}^{X|Y}(\cdot|y)$. In practice, exactly implementing these oracles could be computationally expensive or even impossible. For the Euclidean case, we emphasize that, as the heat kernel has an explicit closed form density (which is the Gaussian), prior works, for example,~\cite {fan2023improved}, only consider inexact Restricted Gaussian Oracles and control the propagated error along iterations.

In this section, we derive rates of convergence in the setting where both the MBI and RHK oracles are implemented inexactly. Specifically, we assume we are able to approximately implement the MBI oracle by generating $y \sim \hat{\pi}_{\eta}^{Y|X}(\cdot|x)$, 
and approximately implement the RKH oracle by generating $x \sim \hat{\pi}_{\eta}^{X|Y}(\cdot|y)$, see Assumption \ref{Assumption_Oracle_TV_quality} below.

\begin{assumption}\label{Assumption_Oracle_TV_quality}
    Denote the output of exact RHK oracle as $\pi_{\eta}^{X|Y}(\cdot|y)$ and inexact RHK oracle as $\hat{\pi}_{\eta}^{X|Y}(\cdot|y)$. Similarly, denote the output of exact MBI oracle as $\pi_{\eta}^{Y|X}(\cdot|x)$ and inexact MBI oracle as $\hat{\pi}_{\eta}^{Y|X}(\cdot|x)$. 
    Let $\zeta_{\mathsf{RHK}}$ and $\zeta_{\mathsf{MBI}}$ be the desired accuracy. We assume that, for inverse step size $\eta^{-1} = \tilde{\mathcal{O}}(\log \frac{1}{\zeta}) $, the RHK and MBI oracle implementations can achieve respectively $\|\hat{\pi}_{\eta}^{X|Y}(\cdot|y) - \pi_{\eta}^{X|Y}(\cdot|y)\|_{TV} \le \zeta_{\mathsf{RHK}}, \forall y$,
    and $\|\hat{\pi}_{\eta}^{Y|X}(\cdot|x) - \pi_{\eta}^{Y|X}(\cdot|x)\|_{TV} \le \zeta_{\mathsf{MBI}}, \forall x$. We then let $\zeta\coloneqq\max\{\zeta_{\mathsf{RHK}},\zeta_{\mathsf{MBI}} \}$.
\end{assumption}

The need for assuming the step size satisfies $\eta^{-1} = \tilde{\mathcal{O}}(\log \frac{1}{\zeta})$ for the approximation quality is as follows. Recall from the discussion below Theorem~\ref{Main_Theorem} that the complexity of Riemannian Proximal Sampler depends on the step size as $\mathcal{O}(\frac{1}{\eta})$. Thus if $\eta$ became too small, for example $\eta^{-1} = \mathcal{O}\left(\frac{1}{\varepsilon}\right)$, then the overall complexity would be $\Poly(\frac{1}{\varepsilon})$, which is not a high-accuracy guarantee.

We also briefly explain the intuition in assuming total variation distance error bound in oracle quality, 
and postpone the detailed discussion to Section \ref{Section_Oracle}.
To guarantee a high quality oracle, we need a high quality approximation of heat kernel. 
As mentioned previously, a popular method is through truncation of infinite series. 
Theoretically, the $L_{2}$ truncation error can be bounded for compact manifold \citep{azangulov2022stationary}, 
which says that the difference between the heat kernel and the approximation of heat kernel are close. 
This naturally imply an error bound in total variation distance, 
which motivates us to consider the propagated error in total variation distance. 

We first start with a result quantifying the error propagated along iterations, under the availability of inexact oracles. The proof of the following result is provided in Appendix \ref{Sec_Proof_Inexact_Theorem}.
\begin{lemma}\label{Lemma_Propagation_Error_TV}
    Let $\rho_{k}^{X}$ denote the law of $X$ through exact oracle implementation of Algorithm \ref{Manifold_Proximal_Sampler_Ideal}, 
    and $\tilde{\rho}_{k}^{X}$ denote the law of $x$ through inexact oracle implementation of Algorithm \ref{Manifold_Proximal_Sampler_Ideal}.
    Under Assumption \ref{Assumption_Oracle_TV_quality}, we have $\|\rho_{k}^{X}(x) - \tilde{\rho}_{k}^{X}(x)\|_{TV} \le k (\zeta_{\mathsf{RHK}} + \zeta_{\mathsf{MBI}})$.
\end{lemma}
Based on this result, we next obtain the following result analogues to Theorem~\ref{Main_Theorem}; the proof is provided in Appendix \ref{Sec_Proof_Inexact_Theorem}.
\begin{theorem}\label{TV_Inexact_BM_Inexact_RHK}
  Similar to Theorem~\ref{Main_Theorem}, let $M$ be a Riemannian manifold without boundary and let $\pi^{X}$ satisfies LSI with constant $\alpha$. Further let Assumption \ref{Assumption_Oracle_TV_quality} hold.
    For any initial distribution $\rho_{0}^{X}$, to reach $\tilde{\mathcal{O}}(\varepsilon)$ 
    total variation distance with oracle accuracy $\zeta = \zeta_{\mathsf{RHK}} = \zeta_{\mathsf{MBI}} = \frac{\varepsilon}{\log^{2} \frac{1}{\varepsilon}}$
    and step size $\frac{1}{\eta} = \tilde{\mathcal{O}}(\log \frac{1}{\varepsilon})$, 
    we need $k = \tilde{\mathcal{O}}(\log^{2} \frac{1}{\varepsilon})$ iterations.
\end{theorem}

\section{Implementation of Inexact Oracles via Heat Kernel Trucation}\label{Section_Oracle}

Theorem~\ref{TV_Inexact_BM_Inexact_RHK} shows that as long we have sufficient accuracy of MBI and RHK oracles satisfying Assumption~\ref{Assumption_Oracle_TV_quality}, we can have a high-accuracy Riemannian sampling algorithm. In this section, we introduce an approximate implementation, based on heat kernel truncation (as introduced in~\ref{prelim}) and rejection sampling. Numerical simulations for this approach are provided in Appendix~\ref{hkimplem}.

First note that for rejection sampling method (in general) there are two key ingredients: a proposal distribution and an acceptance rate. 
Assume we want to generate samples from $\rho$ through rejection sampling.
We choose a suitable proposal distribution denoted as $\mu$, and a suitable scaling constant $K$ 
such that the acceptance rate $K\frac{\rho(x)}{\mu(x)} \le 1, \forall x$.
We generate a random proposal $x \sim \mu$ and $u \in [0, 1]$ being a uniform random number. 
Then we compute $K\frac{\rho(x)}{\mu(x)}$, and accept $x$ if $u \le K\frac{\rho(x)}{\mu(x)}$.

We also introduce the following definition of Riemannian Gaussian distribution, as defined next, which will be used as the proposal distribution in rejection sampling. A Riemannian Gaussian distribution centered at $x^{*}$ with variable $t$ is  $
    \mu(t, x^{*}, x) \propto \mu_{u}(t, x^{*}, x) := \exp\left(-\frac{d(x^{*}, x)^{2}}{2t}\right)$, where $\mu_{u}$ denote an unnormalized version of $\mu$. We use this as our proposal distribution to implement rejection sampling, as exact sampling from such a distribution is  well-studied for certain specific manifolds;
see, for example, \cite{said2017gaussian} for symmetric spaces and \cite{chakraborty2019statistics} for Stiefel manifolds. Furthermore, this notion of a Riemannian Gaussian distribution is also used in the study of differential privacy on Riemannian manifolds due to their practical feasibility~\citep{reimherr2021differential,jiang2023gaussian}. 


\subsection{Implementation of RHK}
 
We first recall the rejection sampling implementation of Restricted Gaussian Oracle (RGO) in the Euclidean setting. Note that, we have $\log \nu_{u}(\eta, x, y_{k}) = -\frac{1}{2\eta} \|x - y_{k}\|^{2}$, 
where $\nu_{u} = \exp(-\frac{1}{2\eta} \|x - y_{k}\|^{2})$ is an unnormalized heat kernel (or the Gaussian density) in Euclidean space. 
Then we have $\pi_{\eta}^{X|Y}(\cdot, y_{k}) \propto e^{-f(x) - \frac{1}{2\eta} \|x - y_{k}\|^{2}} $. 
Then, the RGO is implemented through rejection sampling. Specifically, we can first find the minimizer 
$ x^{*} \in \argmin_{x} f(x) + \frac{1}{2\eta} \|x - y_{k}\|^{2} $. 
Note that the minimizer represents the mode of $\pi_{\eta}^{X|Y}(\cdot, y_{k})$.
We can then sample a Gaussian proposal $x_{p} \sim \mathcal{N}(x^{*}, t I_{d})$ 
for suitable $t$ centered at the mode $x^{*}$ and perform rejection sampling.
For more details, see, for example, \cite{chewi2023log}.

On a Riemannian manifold with $\nu$ denoting the heat kernel, to sample from $\pi_{\eta}^{X|Y}(\cdot, y_{k}) \propto e^{-f(x)} \nu(\eta, x, y_{k})$ through rejection sampling, we need evaluations of $f(x) - \log \nu(\eta, x, y_{k}) $. But in general, we cannot evaluate the heat kernel exactly, hence we seek for certain heat kernel approximations. Hence, we use the truncated heat kernel $\nu_{l}$ to replace $\nu$, 
and perform rejection sampling, see Algorithm \ref{Inexact_Rejection_Sampling}.
In the rejection sampling algorithm, as mentioned previously, we use a Riemannian Gaussian distribution as the proposal for rejection sampling. 
We choose suitable step size $\eta$ and $t$ that depends on $\eta$ s.t. $g(x) - g(x^{*}) \ge \frac{1}{2t}d(x, x^{*})^{2}$. 
Such an inequality can guarantee that the acceptance rate (with Riemannian Gaussian distribution $\mu(t, x^{*}, x)$ as proposal) would not exceed one, i.e., $\frac{\exp(-g(x) + g(x^{*}))}{\mu_{u}(t, x^{*}, x)} \le 1, \forall x$. Then we see that the output of rejection sampling would follow $\hat{\pi}_{\eta}^{X|Y}(x|y_{k}) \propto \exp(f(x) - \log \nu_{l}(\eta, x, y_{k})) $. Similarly, to implement the MBI oracle, we also use rejection sampling to get a high-accuracy approximation. Specifically, Algorithm~\ref{Inexact_BM} generates inexact Brownian motion starting from $x$ with time $\eta$.

\begin{algorithm}[t]
    \begin{algorithmic}
    \STATE Find the minimizer of $g(x) := f(x) - \log \nu_{l}(\eta, x, y_{k})$, denote as $x^{*}$.
    \STATE Set suitable $t$ and constant $C_{\mathsf{RHK}}$ s.t. $V_{\mathsf{RHK}}(x) := \frac{\exp(-g(x) + g(x^{*}) + C_{\mathsf{RHK}})}{\exp(-\frac{1}{2t} d(x, x^{*})^{2})} \le 1, \forall x \in M$
    \FOR{$i=0, 1,2,...$}
    \STATE Generate proposal $x \sim \mu(t, x^{*}, \cdot)$.
    \STATE Generate $u$ uniformly on $[0, 1]$. 
    \STATE Return $x$ if $u \le V_{\mathsf{RHK}}(x)$
    \ENDFOR
    \end{algorithmic}
    \caption{RHK through Rejection Sampling}
    \label{Inexact_Rejection_Sampling} 
\end{algorithm}


\begin{algorithm}[t]
    \begin{algorithmic}
    \STATE Set suitable $t$ and $C_{\mathsf{MBI}}$
    so that $V_{\mathsf{MBI}}(y) := \frac{\exp(\log \nu_{l}(\eta, x, y) - \log \nu_{l}(\eta, x, x) + C_{\mathsf{MBI}})}{\exp(-\frac{d(x, y)^{2}}{2t})} \le 1, \forall y \in M$
    \FOR{$i=0, 1,2,...$}
    \STATE Generate proposal $y \sim \mu(t, x, \cdot)$.
    \STATE Generate $u$ uniformly on $[0, 1]$. 
    \STATE Return $y$ if $u \le V_{\mathsf{MBI}}(y)$
    \ENDFOR
    \end{algorithmic}
    \caption{MBI through Rejection Sampling}
    \label{Inexact_BM} 
\end{algorithm}

\subsection{Verification of Assumption \ref{Assumption_Oracle_TV_quality}}
We now show that Assumption \ref{Assumption_Oracle_TV_quality} is satisfied for the aforementioned inexact implementation of the Riemannian Proximal Sampler. To do so, we specifically consider the case when the manifold $M$ is compact and is a homogeneous space. Recall that $\nu_{l}$ denote the truncated heat kernel with truncation level $l$. Roughly speaking, a homogeneous space is a manifold that has certain symmetry, including Stiefel manifold, Grassmann manifold, hypersphere, and manifold of positive deﬁnite matrices. 

\begin{proposition}\label{Prop_Verify_Assumption}
    Let $M$ be a compact manifold. Assume further that $M$ is a homogeneous space. 
    With truncation implementation of inexact oracles, 
    in order for Assumption \ref{Assumption_Oracle_TV_quality} to be satisfied
    with $\zeta = \frac{\varepsilon}{\log^{2} \frac{1}{\varepsilon}}$,
    we need truncation level $l$ to be of order $\textrm{polylog}({1}/{\varepsilon})$.
\end{proposition}
\textbf{Sketch of proof:} We briefly mention the idea of proof. 
\citet[Proposition 21]{azangulov2022stationary} provided an $L_{2}$ bound on the truncation error, and by Jensen's inequality 
we get an $L_{1}$ bound as desired.
With truncation level $l$ to be of order $\Poly (\log \frac{1}{\varepsilon})$, 
we can achieve $\int_{M} |\nu(\eta, x, y) - \nu_{l}(\eta, x, y)| dV_{g}(x) = \tilde{\mathcal{O}}(\zeta)$. See Proposition \ref{Prop_truncation_1} and Proposition \ref{Prop_truncation_level} for a complete proof.

\begin{remark} 
In Appendix \ref{Subsection_inexact_rej}, we show that on hypersphere $\mathcal{S}^{d}$, when the acceptance rate $V$ in rejection sampling would possibly exceed $1$ in some unimportant region, Assumption \ref{Assumption_Oracle_TV_quality} still holds, via explicit computations.
\end{remark}

When $M$ is not a homogeneous space, to the best of our knowledge, it is unknown how to implement the truncation method. Exploring this direction to further extend the above result is an interesting direction for future work.


\section{Implementation via Varadhan's Asymptotics and Connection to Entropy-Regularized JKO Scheme}\label{Section_Proximal_point_approximation}

In this section, we consider yet another approximation scheme for implementing Algorithm \ref{Manifold_Proximal_Sampler_Ideal}, motivated by its connection with the proximal point method in optimization, where the latter is in the sense of optimization over Wasserstein space\footnote{If $M$ is a smooth compact Riemannian manifold then the Wasserstein space $\mathcal{P}_2(M)$ is the
space of Borel probability measures on $M$, equipped with the Wasserstein metric $W_2$.  We refer the reader to~\cite{villani2021topics} for background on Wasserstein spaces.}~\citep{jordan1998variational,wibisono2018sampling,chen2022improved}. Note that the proximal point method is usually called as the JKO scheme after the authors of~\cite{jordan1998variational}.

Specifically, we consider approximating the heat kernel through Varadhan's asymptotics. 
Let $\hat{\nu}(\eta, x, y) \propto_{y} \exp(-\frac{d(x, y)^2}{2\eta}) =: \hat{\nu}_{u}(\eta, x, y)$ be an inexact evaluation of heat kernel. 
According to Varadhan's asymptotics, $\lim_{\eta \to 0} \hat{\nu}(\eta, x, y) = \nu(\eta, x, y)$. 
Hence when $\eta$ is small, $\hat{\nu}$ is a good approximation of the heat kernel. 
Note that $\hat{\nu}(\eta, x, \cdot)$ in Varadhan's asymptotic is exactly the Riemannian Gaussian distribution $\mu(\eta, x, \cdot)$. Denote $\tilde{\pi}(x, y) = \exp(-f(x)-\frac{d(x, y)^2}{2\eta})$. 
With inexact MBI implemented through Riemannian Gaussian distribution and  inexact RHK implemented through rejection sampling (Algorithm~\ref{Inexact_Rejection_Sampling}) to generate $\tilde{\pi}^{X|Y}(x|y) \propto \exp(-f(x) - \frac{d(x, y)^{2}}{2\eta})$,
we obtain Algorithm \ref{Manifold_Proximal_Sampler_Gaussian}.



For the case when $M = \mathcal{S}^{d}$, we prove in Appendix \ref{Subsection_expected_rej} that to sample from $\tilde{\pi}^{X|Y}(x|y)$ through rejection sampling, with suitable parameters, the cost is $\mathcal{O}(1)$ in both dimension $d$ and step size $\eta$. Obtaining similar results for more general manifolds seems non-trivial. Numerical simulations for this approach are provided in Appendix~\ref{vardhanimplem}. Verifying Assumption~\ref{Assumption_Oracle_TV_quality} for this implementation is open.

\begin{algorithm}[t]
    \begin{algorithmic}
    \FOR{$k=0, 1,2,...$}
    \STATE From $x_{k}$, sample $y_{k} \sim \tilde{\pi}^{Y|X}(\cdot, x_{k})$ which is a Riemannian Gaussian distribution. 
    \STATE From $y_{k}$, sample $x_{k+1} \sim \tilde{\pi}^{X|Y}(\cdot, y_{k}) \propto e^{-f(x) - \frac{d(x, y_{k}^{2})}{2\eta}} $ using Algorithm~\ref{Inexact_Rejection_Sampling}.
    \ENDFOR
    \end{algorithmic}
    \caption{Inexact Manifold Proximal Sampler with Varadhan's Asymptotics}
    \label{Manifold_Proximal_Sampler_Gaussian} 
\end{algorithm}

\subsection{RHK as a proximal operator on Wasserstein space}
We first show that the inexact RHK output in Algorithm \ref{Manifold_Proximal_Sampler_Gaussian} can be viewed as a proximal operator on Wasserstein space, generalizing the Euclidean result in~\cite{chen2020fast} to the Riemannian setting. 
Recall that with a function $f$ and $d$ being a distance function, 
$\prox_{\eta f}(y) = \argmin_{x} f(x) + \frac{1}{2\eta} d(x, y)^{2}$.
The (approximated) joint distribution is $\tilde{\pi}(x, y) = \exp(-f(x) - \frac{d(x, y)^{2}}{2\eta})$.
By direct computation we have the following Lemma (proved in Appendix \ref{Proof_Theorem_Gaussian_JKO}). 

\begin{lemma}\label{Lemma_proximal_calculation}
    We have that 
    \begin{equation*}
        \tilde{\pi}^{X|Y = y} 
        = \argmin_{\rho \in \mathcal{P}_{2}(M)} H_{\tilde{\pi}^{X}}(\rho) + \frac{1}{2\eta} W_{2}^{2}(\rho, \delta_{y}) = \prox_{\eta H_{\tilde{\pi}^{X}}} (\delta_{y}),
\end{equation*}
which shows that the ineact RHK implementation is a proximal operator, i.e., $\tilde{\pi}^{X|Y = y} = \prox_{\eta H_{\tilde{\pi}^{X}}} (\delta_{y})$.
\end{lemma}

\subsection{Connection to Entropy-Regularized JKO Scheme}\label{Section_Approximation_JKO}

Observe that in Algorithm \ref{Manifold_Proximal_Sampler_Gaussian}, the Riemannian Gaussian involves distance square, which naturally relates to Wasserstein distance. Now, recall that for a function $F$ in the Wasserstein space, its Wasserstein gradient flow can be approximated through the following discrete time JKO scheme~\citep{jordan1998variational}:
\begin{equation*}
    \rho_{k+1} = \argmin_{\rho \in \mathcal{P}(\mathbb{R}^{d})} F(\rho) + \frac{1}{2\eta} W_{2}^{2} (\rho, \rho_{k}).
\end{equation*}
It was proved that as $\eta \to 0$, the discrete time sequence $\{\rho_{k}\}$ converge to the Wasserstein gradient flow of $F$.
Later, \cite{peyre2015entropic} proposed an approximation scheme through entropic smoothing of Wasserstein distance:
\begin{equation*}
    \rho_{k+1} = \argmin_{\rho \in \mathcal{P}(\mathbb{R}^{d})} F(\rho) + \frac{1}{2\eta} W_{2, \varepsilon}^{2} (\rho, \rho_{k}),
\end{equation*}
where $W_{2, \varepsilon}$ is the entropy-regularized 2-Wasserstein distance defined by (here $H$ is the negative entropy)
\begin{equation*}
    W_{2, t}^{2}(\rho_{1}, \rho_{2}) = \inf_{\gamma \in \mathcal{C}(\rho_{1}, \rho_{2})} \int d(x, y)^{2} d\gamma(x, y) + t H(\gamma).
\end{equation*}

In Euclidean space,~\cite{chen2022improved} showed that the proximal sampler can be viewed as an entropy-regularized JKO scheme.
We extend such an interpretation to Riemannian manifolds. Specifically, we show that Algorithm \ref{Manifold_Proximal_Sampler_Gaussian} which is an approximation of the exact proximal sampler (Algorithm~\ref{Manifold_Proximal_Sampler_Ideal}), can be viewed as an entropy-regularized JKO as stated in Theorem~\ref{Theorem_Gaussian_JKO} (proved in Appendix \ref{Proof_Theorem_Gaussian_JKO}). Note that on a Riemannian manifold the negative entropy is $H(\gamma) := \int_{M \times M} \gamma \log(\gamma) dV_{g}(x) dV_{g}(y) $.
\begin{theorem}\label{Theorem_Gaussian_JKO}
    Recall that $\pi^{X} \propto e^{-f}$.
    Let $x_{k}, y_{k}, x_{k+1}$ be generated by Algorithm \ref{Manifold_Proximal_Sampler_Gaussian}. 
    Let $\tilde{\rho}_{k}^{X}$, $\tilde{\rho}_{k}^{Y}$ and $\tilde{\rho}_{k+1}^{X}$ be the distribution of $x_{k}, y_{k}, x_{k+1}$, respectively. 
    Then  
    \begin{align*}
            \tilde{\rho}_{k}^{Y} = \argmin_{\chi \in \mathcal{P}_{2}(M)} \frac{1}{2\eta} W_{2, 2\eta}^{2}(\tilde{\rho}_{k}^{X}, \chi) \quad\text{and}\quad
            \tilde{\rho}_{k+1}^{X} = \argmin_{\chi \in \mathcal{P}_{2}(M)} \int f d\chi + \frac{1}{2\eta} W_{2, 2\eta}^{2}(\tilde{\rho}_{k}^{Y}, \chi).
    \end{align*}
\end{theorem}

\section{Conclusion}

We introduced the \textit{Riemannian Proximal Sampler} for sampling from densities on Riemannian manifolds. By leveraging the Manifold Brownian Increments (MBI) and the Riemannian Heat-kernel (RHK) oracles, we established high-accuracy sampling guarantees, demonstrating a logarithmic dependence on the inverse accuracy parameter (i.e., \(\text{polylog}(1/\varepsilon)\)) in the Kullback-Leibler divergence (for exact oracles) and total variation metric (for inexact oracles). Additionally, we proposed practical implementations of these oracles using heat-kernel truncation and Varadhan’s asymptotics, providing a connection between our sampling method and the Riemannian Proximal Point Method. 

Future works include: (i) characterizing the precise dependency on other problem parameters apart from $\varepsilon$, (ii) improving oracle approximations for enhanced computational efficiency and (iii) extending these techniques to broader classes of manifolds (and other metric-measure spaces). 

\bibliographystyle{abbrvnat}

\bibliography{ref}

\appendix

\clearpage

\section{Simulation Results}\label{sec:sim}

\subsection{Brownian Motion Approximation via Geodesic random walk}\label{georw}

In our experiments, to compare against the Riemannian Langevin Algorithm, we used the geodesic random walk algorithm to simulate the MBI oracle following \cite{cheng2022efficient, de2022riemannian,schwarz2023efficient}; see Algorithm \ref{BM_GRW}. More efficient implementation is a topic of great interest in the literature; see, for example,~\citep{schwarz2023efficient}. 
\begin{algorithm}[H]
    \begin{algorithmic}
    \STATE Input $x \in M, t > 0$.
    \STATE Sample $\xi$ being a Euclidean Brownian increment with time $t$ in the tangent space $T_{x}M$.
    \STATE Output $y = \exp_{x}(\xi)$.
    \end{algorithmic}
    \caption{Approximation of Manifold Brownian Motion Using Geodesic Random Walk} 
    \label{BM_GRW} 
\end{algorithm}

While it is well-known that geodesic random walks converge asymptotically to the Brownian motion on the manifold, non-asymptotic rates of convergence in various metrics of interest is largely unknown. A basic non-asymptotic error bound for geodesic random walk is available in Wasserstein distance (see~\citet[Lemma 7]{cheng2022efficient}). Mixing time results are provided in~\cite{mangoubi2018rapid}. However, such a result is not immediately applicable to establish high-accuracy guarantees for the Riemannian proximal sampler, when the MBI oracle is implemented via geodesic random walk. An important and interesting future work is establishing rates of convergence for geodesic random walk in various metrics of interest so that those results could be leveraged to obtain high-accuracy guarantees for the Riemannian proximal sampler.

\subsection{Numerical Experiments for Algorithms \ref{Inexact_Rejection_Sampling} and~\ref{Inexact_BM}: von Mises-Fisher distribution on Hyperspheres}\label{hkimplem}
In this experiment, we test the performance of Algorithms \ref{Inexact_Rejection_Sampling} and~\ref{Inexact_BM} for sampling from the von Mises-Fisher distribution on hyperspheres and compare it with the Riemannian LMC method. 
In this case, we have $f(x) = - \kappa \mu^{T} x$.
Note that this $f(x)$ has a unique minimizer on $\mathcal{S}^{d}$.
This implies that LSI is satisfied, see \cite[Theorem 3.4]{li2023riemannian}. We demonstrate the performance of our Algorithm on $\mathcal{S}^{2} \subseteq \mathbb{R}^{3}$ with $\mu = (10, 0.1, 2)^{T}$ and $\kappa = 10$, and on $\mathcal{S}^{5}$ with $\mu = (5, 0.1, 2, 1, 1, 1)^{T}$ and $\kappa = 10$. For the purpose of numerical demonstration, we sample the Riemannian Gaussian distribution through rejection sampling. 

To evaluate the performance, we estimate $\mathbb{E}[d(x, x^{*})^{2}]$, where $x^{*}$ is the minimizer of $f$, representing the mode of the distribution, and plot it as a function of iterations. Note that the quantity $\mathbb{E}[d(x, x^{*})^{2}]$ is referred to as Fr{\'e}chet variance \cite{frechet1948elements,dubey2019frechet}. For this, we generate $10000$ samples (by generating samples independently via different runs) and compute $\frac{1}{10000}\sum_{i = 1}^{10000} d(x_{i}, x^{*})^{2}$. We use rejection sampling to generate unbiased samples and get an estimation of the true value. 
Due to the biased nature of the Riemannian LMC method, to achieve a high accuracy we need a small step size. Contrary to the Riemanian LMC method, the proximal sampler is unbiased, and it can achieve an accuracy while using a large step size and a smaller number of iterations; see Figures~\ref{fig:figures}-(a) and (b). 

\begin{figure}
\centering
\subfigure
[\texorpdfstring{$\mathcal{S}^{2}$}{S}, \texorpdfstring{$f(x)=\exp(\kappa \mu^{T}x)$}{f}]{
        \label{fig_experiment_a}
        \includegraphics[width=0.31\textwidth]{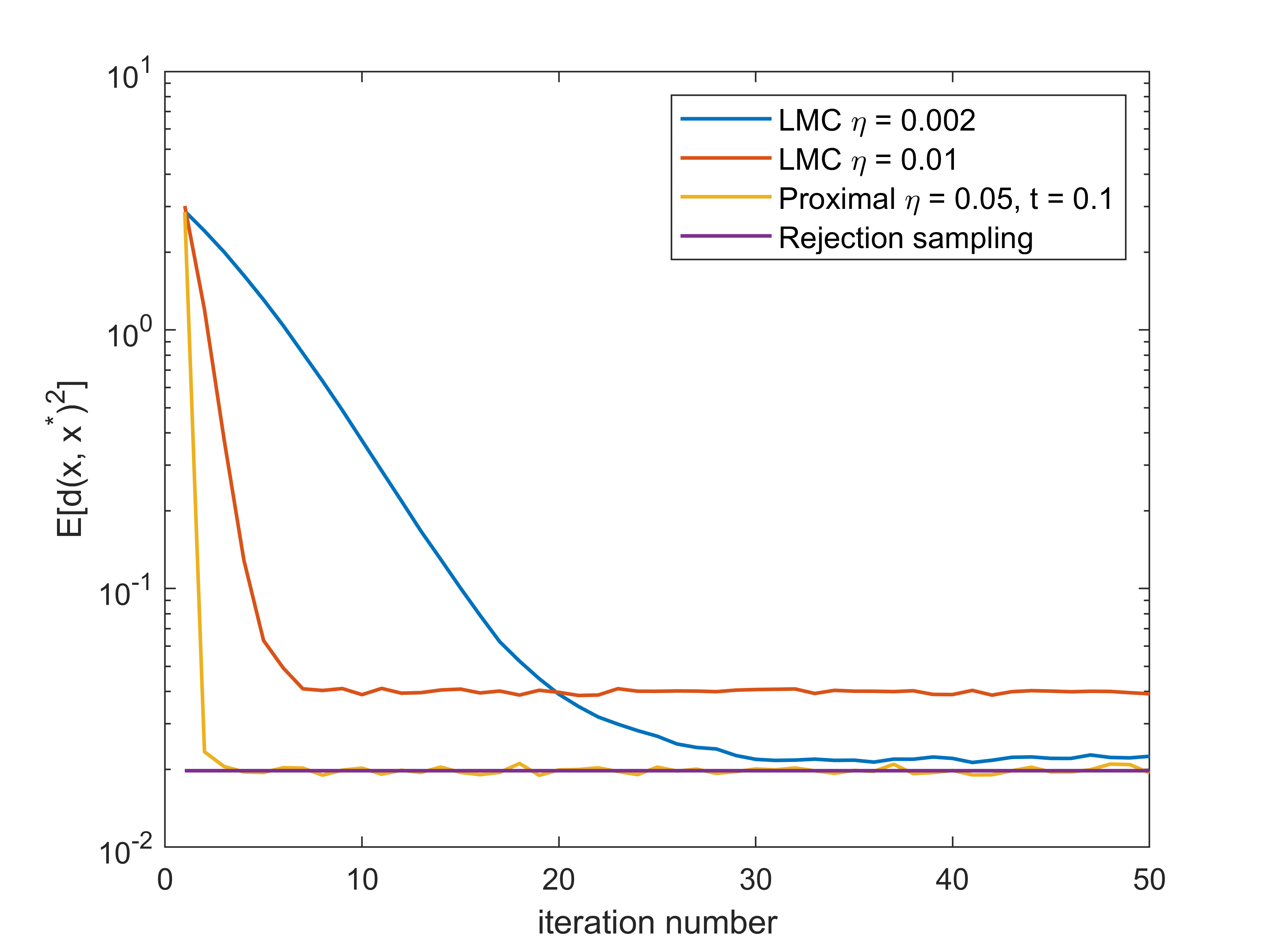}
    }
\subfigure
[\texorpdfstring{$\mathcal{S}^{5}$}{S}, \texorpdfstring{$f(x) = \exp(\kappa \mu^{T}x)$}{f}]{
        \label{fig_experiment_b}
        \includegraphics[width=0.31\textwidth]{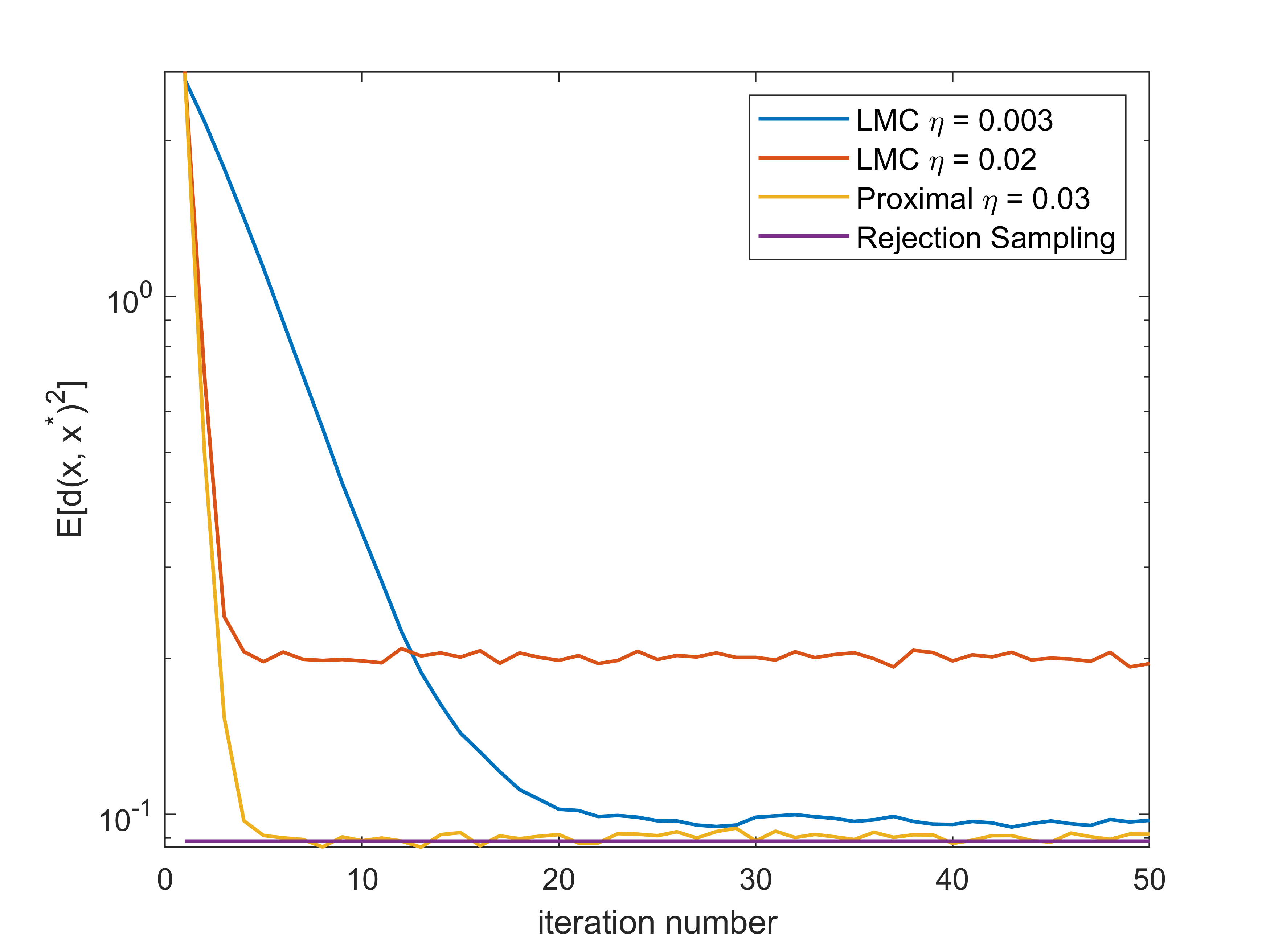}
    }
\subfigure
[\texorpdfstring{$P_{3}$}{P}, \texorpdfstring{$f(X)=\exp(\frac{-d(X, I)^{4}}{2\sigma^{2}})$}{f}, ]{
        \label{fig_experiment_c}
        \includegraphics[width=0.31\textwidth]{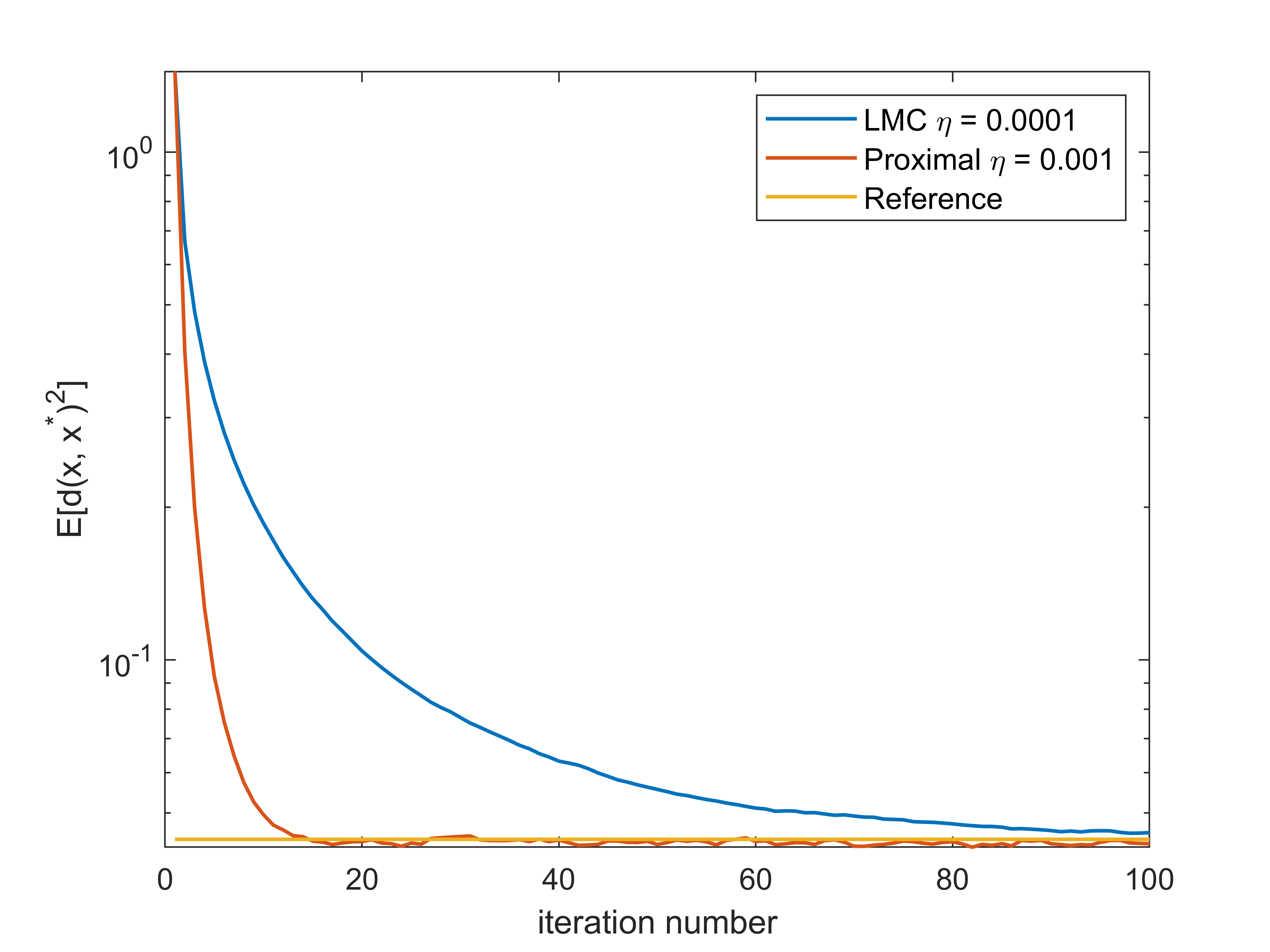}
    }   
\caption{Frech\'et variance (i.e., \texorpdfstring{$\mathbb{E}[d(x, x^{*})^{2}]$}{E} versus number of iterations. Left and Middle figure correspond to the implementation via Algorithm~\ref{Inexact_Rejection_Sampling} and~\ref{Inexact_BM}. Right figure corresponds to implementation via  Algorithm~\ref{Manifold_Proximal_Sampler_Gaussian}. }
\label{fig:figures}
\end{figure}

\subsection{Numerical Experiments for Algorithm~\ref{Manifold_Proximal_Sampler_Gaussian}: Manifold of Positive Definite Matrices}\label{vardhanimplem}

In this subsection we illustrate the performance of Algorithm~\ref{Manifold_Proximal_Sampler_Gaussian}  for sampling on the manifold of positive definite matrices.
Let $P_{m} = \{ X \in GL(m): X^{T} = X \quad\text{and} \quad y^{T} X y > 0, \forall y \in \mathbb{R}^{m} \}$
be the set of symmetric positive definite matrices. 
According to \citet[Section 6.2]{bharath2023sampling}, we can choose $g(U, V) = \tr(X^{-1}UX^{-1}V)$ 
and make $(P_{m}, g)$ a Riemannian manifold. 
It is a non-compact manifold with non-positive sectional curvature, geodesically complete 
and is a homogeneous space of general linear group $GL(m)$. Additional details are provided in Appendix~\ref{psddetails}.

We test the performance of Algorithm~\ref{Manifold_Proximal_Sampler_Gaussian} when the potential function $f(X) = \frac{1}{2\sigma^{2}} d(X, I_{m})^{4}$, $m = 3, \sigma = 0.03$, following~\cite{bharath2023sampling}. Note that $f$ is not gradient Lipschitz. In the Figure~\ref{fig:figures}-c, we estimate $\mathbb{E}[d(x, x^{*})^{2}]$ and plot it as function of iterations, 
where $x^{*} = I_{3}$ is the minimizer of $f$, representing the mode of the distribution. For a baseline comparison, we run Riemannian Langevin Monte Carlo for $200$ iterations with decreasing step size to get a reference value of $\mathbb{E}[d(x, x^{*})^{2}]$, which serves as the true $\mathbb{E}[d(x, x^{*})^{2}]$. Similar to the previous experiment, we generate $1000$ samples from independent run, and compute $\frac{1}{1000}\sum_{i = 1}^{1000} d(x_{i}, x^{*})^{2}$ for each method. For the Riemannian Langevin Monte Carlo method, we find that if we set step size to $0.001$ instead of $0.0001$, after a few iterations the algorithm diverges (potentially due to lack of gradient Lipschitz condition). But for the proximal sampler (which is an unbiased algorithm), even with a large step size as illustrated in the plots, 
the approximation scheme still works well and can achieve a higher accuracy than the Riemannian LMC algorithm. 


\section{Additional Preliminaries}\label{sec:addprelim}

\subsection{Divergence}

We will briefly discuss divergence for the manifold setting. 
More details can be found in \cite{lee2018introduction}.
Recall that in Euclidean space, for a vector field $F = (F_{1}, ..., F_{n})$ in $\mathbb{R}^{n}$, 
divergence of $F$ is defined as $\nabla \cdot F = \sum_{i = 1}^{n} \frac{\partial F_{i}}{\partial x_{i}}$. It has a natural generalization to the manifold setting using interior multiplication and exterior derivative. 

The Riemannian divergence is defined as the function such that $d(i_{X} (dV_{g})) = (\Div X)dV_{g}$, 
where $X$ is any smooth vector field on $M$, $i$ denotes interior multiplication and $d$ denotes exterior derivative. See for example \citet[Appendix B]{lee2018introduction} for more details. 
On a Riemannian manifold, recall the volumn form is 
$dV_{g} = \sqrt{\det(g_{ij})} dx^{1} \wedge ... \wedge dx^{n} $.
Let $Y = \sum_{i = 1}^{n} Y^{i} \frac{\partial}{\partial x_{i}}$. We can compute the interior multiplication as
\begin{align*}
        i_{Y} (dV_{g}) &= \sqrt{\det(g_{ij})}\sum_{j = 1}^{n} ((-1)^{j+1} dx^{j}(Y)) 
        dx^{1} \wedge ... \wedge d\hat{x}^{j} \wedge ... \wedge dx^{n} \\
        &= \sqrt{\det(g_{ij})}\sum_{j = 1}^{n} ((-1)^{j+1} Y^{j}) 
        dx^{1} \wedge ... \wedge d\hat{x}^{j} \wedge ... \wedge dx^{n}.
\end{align*}
We can then compute its exterior derivative as
\begin{align*}
        d(i_{Y} (dV_{g})) 
        &= \sum_{j = 1}^{n} ((-1)^{j+1} \frac{\partial (Y^{j} \sqrt{\det(g_{ij})})}{\partial x_{j}} dx^{j}) 
    dx^{1} \wedge ... \wedge d\hat{x}^{j} \wedge ... \wedge dx^{n} \\
    &= \frac{1}{\sqrt{\det(g_{ij})}}\sum_{j = 1}^{n} \frac{\partial (Y^{j} \sqrt{\det(g_{ij})})}{\partial x_{j}} 
    \sqrt{\det(g_{ij})} dx^{1} \wedge ... \wedge dx^{n} \\
    &= \frac{1}{\sqrt{\det(g_{ij})}}\sum_{j = 1}^{n} \frac{\partial (Y^{j} \sqrt{\det(g_{ij})})}{\partial x_{j}} dV_{g}.
\end{align*}
Hence we get $\Div(Y) = \frac{1}{\sqrt{\det(g_{ij})}}\sum_{j = 1}^{n} \frac{\partial (Y^{j} \sqrt{\det(g_{ij})})}{\partial x_{j}}$.
In Euclidean space, this reduces to $\Div(Y) = \sum_{j = 1}^{n} \frac{\partial Y^{j}}{\partial x_{j}}$.

For $u \in C^{\infty}(M)$ and $X \in \mathfrak{X}(M)$,
the divergence operator satisfies the following product rule 
\begin{align*}
    \Div (uX) = u \Div(X) + \langle \grad u, X\rangle_{g}.
\end{align*}
Furthermore, we have the ``integration by parts'' formula 
(with $\tilde{g}$ denote the induced Riemannian metric on $\partial M$)
\begin{align*}
    \int_{M} \langle \grad u, X \rangle_{g} dV_{g} 
    = \int_{\partial M} u \langle X, N \rangle_{g} dV_{\tilde{g}}
    - \int_{M} u \Div X dV_{g}.
\end{align*}
When $M$ does not have a boundary, $\partial M = \emptyset$. So we have 
\begin{align*}
    \int_{M} \langle \grad u, X \rangle_{g} dV_{g} = - \int_{M} u \Div X dV_{g}.
\end{align*}


\subsection{Normal coordinates}

\paragraph{Riemannian normal coordinates.}
Let $x \in M$. There exist a neighborhood $V$ of the origin in $T_{x}M$
and a neighborhood $U$ of $x$ in $M$ such that the exponential map $\exp_{x}: V \to U$ 
is a diffeomorphism. The set $U$ is called a normal neighborhood of $x$.
Given an orthonormal basis $(z_{i})$ of $T_{x}M$, there is a basis isomorphism from $T_{x}M$ to $\mathbb{R}^{d}$.
The exponential map can be combined with the basis isomorphism to get a smooth coordinate map $\varphi: U \to \mathbb{R}^{d}$. 
Such coordinates are called normal coordinates at $x$. 
Under normal coordinates, the coordinates of $x$ is $0 \in \mathbb{R}^{d}$. For more details see for example \citet[Chapter 5]{lee2018introduction}

\paragraph{Cut locus and injectivity radius.}
Consider $v \in T_{x}M$ and let $\gamma_{v}$ be the maximal geodesic starting at $x$ with initial velocity $v$. 
Denote $t_{cut}(x, v) = \sup \{t > 0: \text{ the restriction of } \gamma_{v} \text{ to } [0, t] \text{ is minimizing}\}$
The cut point of $x$ along $\gamma_{v}$ is $\gamma_{v}(t_{cut}(x, v))$ provided $t_{cut}(x, v) < \infty$.
The cut locus of $x$ is denoted as 
$\Cut(x) = \{q \in M: q \text{ is the cut point of } x \text{ along some geodesic.}\}$.
The injectivity radius at $x$ is the distance from $x$ to its cut locus if the
cut locus is nonempty, and infinite otherwise \cite[Proposition 10.36]{lee2018introduction}.
When $M$ is compact, the injectivity radius is positive \cite[Lemma 6.16]{lee2018introduction}.

\begin{theorem}
    {\cite[Theorem 10.34]{lee2018introduction}}
    Let $M$ be a complete, connected Riemannian manifold and $x \in M$.
    Then 
    \begin{enumerate}
        \item The cut locus of $x$ is a closed subset of $M$ of measure zero.
        \item The restriction of $\exp_{x}$ to $\overline{\ID(x)}$ is surjective.
        \item The restriction of $\exp_{x}$ to $\ID(x)$ is a diffeomorphism onto $M \backslash \Cut(x)$.
    \end{enumerate}
    Here $\ID(x) = \{v \in T_{x}M: |v| < t_{cut}(x, \frac{v}{|v|})\}$ is the injectivity domain of $x$.
\end{theorem}
Then for any $p \in M$, under normal coordinates, for all well-behaved $f$, we have 
\begin{align*}
    \int_{M} f dV_{g}
    = \int_{M \backslash \Cut(p)} f dV_{g}
    = \int_{\varphi(M \backslash \Cut(p)) \subseteq \mathbb{R}^{d}} f(\varphi^{-1}(x))\sqrt{\det(g)} dx.
\end{align*}

\subsection{Additional details for manifold of positive definite matrices}\label{psddetails}

We briefly mention some properties of $P_{m}$.
The inverse of the exponential map is globally defined and the cut locus of every point is empty.
For symmetric matrix $S \in \mathbb{R}^{m \times m}$, 
\begin{align*}
        \exp_{X}(tS) &= X^{1/2} \Exp (t X^{-1/2} S X^{-1/2}) X^{1/2}, \\
        \gamma(t) &= X_{1}^{1/2} \Exp (t \Log (X_{1}^{-1/2} X_{2} X_{1}^{-1/2}) ) X_{1}^{1/2} \text{ is a geodesic that connect } X_{1}, X_{2},\\
        d(X_{1}, X_{2}) &= \sqrt{\sum_{i = 1}^{m} (\log (r_{i}))^{2}} \text{ with } r_{i} \text{ being eigenvalues of } X_{1}^{-1} X_{2},\\
        \exp_{X_{1}}^{-1}(X_{2}) &= \gamma'(0) = X_{1}^{1/2} \Log (X_{1}^{-1/2} X_{2} X_{1}^{-1/2}) X_{1}^{1/2}.
\end{align*} 

We have the following fact. 
\begin{lemma}
    Let $\phi(x) = d(x, y)^{2}$ with $y \in M$ being fixed. 
    We have $\grad \phi(x) = -2 \exp_{x}^{-1}(y)$.
\end{lemma}


\section{Proof of Main Theorem}\label{Sec_Proof_Main_Theorem}

For a given $\phi$, 
define the $\phi$-divergence to be $\Phi_{\pi}(\rho) = \mathbb{E}_{\pi} [\phi(\frac{\rho}{\pi})]$.
Define the following dissipation functional
\begin{align*}
    D_{\pi}(\rho) := \mathbb{E}_{\rho} \left[ \langle \grad (\phi' \circ \frac{\rho}{\pi}) , 
        \grad \log \frac{\rho}{\pi} \rangle\right].
\end{align*}
We can now compute the time derivative of the $\phi$-divergence along certain flow.

Let $\mu_{t}^{X}$ be the law of the continuous-time 
Langevin diffusion with target distribution $\pi^{X} \propto e^{-f(x)}$.
That is, we have the following SDE, $dX_{t} = - \grad f(X_{t}) dt + \sqrt{2} dB_{t}$. Then, $\mu_{t}^{X}$ satisfies the following Fokker-Planck equation (see Lemma \ref{Lemma_SDE_F_P} for a proof).
\begin{align*}
        \frac{\partial}{\partial t} \mu_{t}^{X} 
        &= \Div (\mu_{t}^{X}  \grad f(X_{t}) + \grad \mu_{t}^{X} ) 
        = \Div ( \grad \mu_{t}^{X} - \mu_{t}^{X} \frac{ \grad \pi^{X} }{\pi^{X}} ) \\
        &= \Div ( \mu_{t}^{X} \grad \log \frac{\mu_{t}^{X}}{\pi^{X}}).
\end{align*}
We now show that $D_{\pi^{X}}(\mu_{t}^{X}) = - \partial_{t} \Phi_{\pi^{X}}(\mu_{t}^{X})$.
\begin{lemma}\label{lem:lem1}
    We have that 
    \begin{align*}
        D_{\pi^{X}}(\mu_{t}^{X}) 
        := \mathbb{E}_{\mu_{t}^{X}} [ 
            \langle \grad (\phi' \circ \frac{\mu_{t}^{X}}{\pi^{X}}) , 
            \grad \log \frac{\mu_{t}^{X}}{\pi^{X}} \rangle]
        = - \partial_{t} \Phi_{\pi^{X}}(\mu_{t}^{X}).
    \end{align*}
\end{lemma}
\begin{proof}[Proof of Lemma~\ref{lem:lem1}]
    By using the fact that $\frac{\partial}{\partial t} \mu_{t}^{X} = \Div ( \mu_{t}^{X} \grad \log \frac{\mu_{t}^{X}}{\pi^{X}})$,
    we have 
    \begin{align*}
            \frac{\partial}{\partial t} \Phi_{\pi^{X}}(\mu_{t}^{X})
        =& \frac{\partial}{\partial t} \int_{M} \pi^{X} \phi(\frac{\mu_{t}^{X}}{\pi^{X}}) dV_{g}(x) \\
        =& \int_{M} \phi'(\frac{\mu_{t}^{X}}{\pi^{X}}) \frac{\partial}{\partial t} \mu_{t}^{X} dV_{g}(x) 
        = \int_{M} \Big(\phi'(\frac{\mu_{t}^{X}}{\pi^{X}}) \Big) \Big(\Div ( \mu_{t}^{X} \grad \log \frac{\mu_{t}^{X}}{\pi^{X}})\Big) dV_{g}(x) \\
        =& - \int_{M} \mu_{t}^{X}  \langle \grad \phi' \circ \frac{\mu_{t}^{X}}{\pi^{X}}, \grad  \log \frac{\mu_{t}^{X}}{\pi^{X}} \rangle dV_{g}(x),
    \end{align*}
    where in the last equality we used integration by parts.
\end{proof}

To get more intuition on the notion of $\phi$-divergence and dissipation functional, 
consider $\phi(x) = x \log(x)$. 
We get KL divergence and fisher information:
\begin{align*}
        \Phi_{\pi}(\rho) &= \mathbb{E}_{\pi}[\frac{\rho}{\pi} \log(\frac{\rho}{\pi})] = \mathbb{E}_{\rho} [\log(\frac{\rho}{\pi})] = H_{\pi}(\rho), \\
        D_{\pi}(\rho) &= \mathbb{E}_{\rho} [ 
            \langle \grad (\phi' \circ \frac{\rho}{\pi}) , 
            \grad \log \frac{\rho}{\pi} \rangle]
        = \mathbb{E}_{\rho} [ 
            \|\grad \log \frac{\rho}{\pi}\|^{2}]
        = \mathbb{E}_{\pi}[\frac{\pi}{\rho}\|\grad \frac{\rho}{\pi}\|^{2}]
        = J_{\pi}(\rho).
\end{align*}

Our proof is now based on generalizing the proof in \cite{chen2022improved} to the Riemannian setting. 
In Section \ref{Section_Forward} we analyze the first step of proximal sampler by viewing it as simultaneous (forward) heat flow. In Section \ref{Section_Backward} we analyze the second step of proximal sampler by viewing it as simultaneous backward flow. Combining the two steps together, we prove convergence of proximal sampler under LSI in Section \ref{Section_Conv_LSI}.

\subsection{Forward step: Simultaneous Heat Flow}\label{Section_Forward}

We can first compute the time derivative of the $\phi$-divergence along simultaneous heat flow. 
\begin{lemma}\label{lem:lemma2}
    \label{Lemma_Forward_Riemannian} Define $Q_{t}$ to describe the forward heat flow. 
    Let $\rho^{X}Q_{t}$ and $\pi^{X} Q_{t}$ evolve according to the simultaneous heat flow, satisfying
    \begin{align*}
            \partial_{t} \rho^{X} Q_{t} = \frac{1}{2} \Delta (\rho^{X} Q_{t}) \quad\text{and}\quad
            \partial_{t} \pi^{X} Q_{t} = \frac{1}{2} \Delta (\pi^{X} Q_{t}).
    \end{align*}
    Then $\partial_{t} \Phi_{\pi^{X} Q_{t}} (\rho^{X}Q_{t}) = - \frac{1}{2} D_{\pi^{X} Q_{t}} (\rho^{X} Q_{t})$.
\end{lemma}
\begin{proof}[Proof of Lemma~\ref{lem:lemma2}]
    Denote $\rho_{t}^{X} := \rho^{X}Q_{t}$ and $\pi_{t}^{X} := \pi^{X}Q_{t}$. Then, we have
    \begin{align*}
            2 \frac{\partial }{\partial t} \Phi_{\pi_{t}^{X}} (\rho_{t}^{X})
            &= 2 \frac{\partial }{\partial t} \int_{M} \pi_{t}^{X} \phi(\frac{\rho_{t}^{X}}{\pi_{t}^{X}}) dV_{g}(x) \\
            &= 2 \int_{M} \phi(\frac{\rho_{t}^{X}}{\pi_{t}^{X}})\frac{\partial}{\partial t} \pi_{t}^{X} 
            + \phi'(\frac{\rho_{t}^{X}}{\pi_{t}^{X}}) 
            \Big( \frac{\partial }{\partial t} \rho_{t}^{X} - (\frac{\partial}{\partial t} \pi_{t}^{X}) \frac{\rho_{t}^{X}}{\pi_{t}^{X}} \Big) dV_{g}(x).
    \end{align*}
Recall that by construction, 
    $$\partial_{t} \rho_{t}^{X} = \frac{1}{2} \Delta \rho_{t}^{X} = \frac{1}{2} \Div (\grad \rho_{t}^{X}) = \frac{1}{2} \Div (\rho_{t}^{X} \grad \log \rho_{t}^{X})$$
    and $\partial_{t} \pi_{t}^{X} = \frac{1}{2} \Delta \pi_{t}^{X} = \frac{1}{2} \Div (\pi_{t}^{X} \grad \log \pi_{t}^{X})$.
    Hence, we get 
    \begin{align*}
            2 \frac{\partial }{\partial t} \Phi_{\pi_{t}^{X}} (\rho_{t}^{X})
            =& 2 \int_{M} \phi(\frac{\rho_{t}^{X}}{\pi_{t}^{X}})\frac{\partial}{\partial t} \pi_{t}^{X} 
            + \phi'(\frac{\rho_{t}^{X}}{\pi_{t}^{X}}) 
            \Big( \frac{\partial }{\partial t} \rho_{t}^{X} - (\frac{\partial}{\partial t} \pi_{t}^{X}) \frac{\rho_{t}^{X}}{\pi_{t}^{X}} \Big) dV_{g}(x)\\
            =& \int_{M} \phi(\frac{\rho_{t}^{X}}{\pi_{t}^{X}})\Div (\pi_{t}^{X} \grad \log \pi_{t}^{X}) \\
            &+ \phi'(\frac{\rho_{t}^{X}}{\pi_{t}^{X}}) 
            \Big( \Div (\rho_{t}^{X} \grad \log \rho_{t}^{X}) - (\Div (\pi_{t}^{X} \grad \log \pi_{t}^{X})) \frac{\rho_{t}^{X}}{\pi_{t}^{X}} \Big) dV_{g}(x)\\
            =& \int_{M} 
            - \Big\langle \grad \phi(\frac{\rho_{t}^{X}}{\pi_{t}^{X}}), \pi_{t}^{X} \grad \log \pi_{t}^{X} \Big\rangle\\
            &- \Big\langle \grad \phi'(\frac{\rho_{t}^{X}}{\pi_{t}^{X}}) , \rho_{t}^{X} \grad \log \rho_{t}^{X} \Big\rangle 
            + \Big\langle \grad(\frac{\rho_{t}^{X}}{\pi_{t}^{X}} \phi'(\frac{\rho_{t}^{X}}{\pi_{t}^{X}})) , \pi_{t}^{X} \grad \log \pi_{t}^{X} \Big\rangle dV_{g}(x)\\
            =& \int_{M} 
            - \Big\langle \grad \phi(\frac{\rho_{t}^{X}}{\pi_{t}^{X}}), \pi_{t}^{X} \grad \log \pi_{t}^{X} \Big\rangle
            - \Big\langle \grad \phi'(\frac{\rho_{t}^{X}}{\pi_{t}^{X}}) , \rho_{t}^{X} \grad \log \rho_{t}^{X} \Big\rangle \\ & 
            + \Big\langle  \phi'(\frac{\rho_{t}^{X}}{\pi_{t}^{X}}) \grad \frac{\rho_{t}^{X}}{\pi_{t}^{X}}, \pi_{t}^{X} \grad \log \pi_{t}^{X} \Big\rangle 
            + \Big\langle \frac{\rho_{t}^{X}}{\pi_{t}^{X}}\grad \phi'(\frac{\rho_{t}^{X}}{\pi_{t}^{X}}) , \pi_{t}^{X} \grad \log \pi_{t}^{X} \Big\rangle dV_{g}(x).
    \end{align*}
  Now, notice that 
    \begin{align*}
        \Big\langle \grad \phi(\frac{\rho_{t}^{X}}{\pi_{t}^{X}}), \pi_{t}^{X} \grad \log \pi_{t}^{X} \Big\rangle
        = \Big\langle \phi'(\frac{\rho_{t}^{X}}{\pi_{t}^{X}}) \grad \frac{\rho_{t}^{X}}{\pi_{t}^{X}}, \pi_{t}^{X} \grad \log \pi_{t}^{X} \Big\rangle.
    \end{align*}
    So we get 
    \begin{align*}
            2 \frac{\partial }{\partial t} \Phi_{\pi_{t}^{X}} (\rho_{t}^{X})
            &= \int_{M} \Big\langle \frac{\rho_{t}^{X}}{\pi_{t}^{X}}\grad \phi'(\frac{\rho_{t}^{X}}{\pi_{t}^{X}}) , \pi_{t}^{X} \grad \log \pi_{t}^{X} \Big\rangle 
            - \Big\langle \grad \phi'(\frac{\rho_{t}^{X}}{\pi_{t}^{X}}) , \rho_{t}^{X} \grad \log \rho_{t}^{X} \Big\rangle dV_{g}(x)\\
            &= - \int_{M} \rho_{t}^{X} \Big\langle \grad \phi'(\frac{\rho_{t}^{X}}{\pi_{t}^{X}}) , \grad \log \frac{\rho_{t}^{X}}{\pi_{t}^{X}} \Big\rangle dV_{g}(x) \\
            &= - \mathbb{E}_{\rho_{t}^{X}} \Big\langle \grad (\phi' \circ \frac{\rho_{t}^{X}}{\pi_{t}^{X}}) , \grad \log \frac{\rho_{t}^{X}}{\pi_{t}^{X}} \Big\rangle 
            = - D_{\pi_{t}^{X}} (\rho_{t}^{X}).
    \end{align*}
\end{proof}

\subsection{Backward step: Simultaneous Backward Flow}\label{Section_Backward}
We leverage the following result.
\begin{theorem}[Theorem 3.1 in \cite{de2022riemannian}]
For a SDE $dX_{t} = b(X_{t}) dt + dB_{t}$, let $p_{t}$ denote the distribution of $X_{t}$. Denote $Y_{t} = X_{T - t}, t \in [0, T]$ to be the time-reversed diffusion. We have that $dY_{t} = (-b(Y_{t}) + \grad \log p_{T - t} (Y_{t}) ) dt + dB_{t}$.
\end{theorem}

Note that the time reversal can be understood as $(Y_{T}, Y_{0})$ has the same distribution as $(X_{0}, X_{T})$. 

Recall that $\nu(t, x, y)$ is the density of manifold Brownian motion starting from $x$ with time $t$ and evaluated at $y$,
and that $\pi(x, y) = \pi^{X}(x) \nu(\eta, x, y)$. 
We denote $\pi^{Y} = \pi^{X} Q_{\eta}$ to be the $Y$-marginal of $\pi(x, y)$.
Let $\pi_{t} := \pi^{X} Q_{t}$.
Consider the forward process $dX_{t} = dB_{t}$ with $X_{0} \sim \pi^{X}$. 
We know that the time-reversed process satisfies $dY_{t} = \grad \log \pi_{\eta - t} (Y_{t}) dt + dB_{t}$.

Define $Q_{t}^{-}$ as follows. 
Given $\rho^{Y}$, set $\rho^{Y} Q_{t}^{-}$ to be the law at time $t$, of the solution of the time-reversed SDE (with $T = \eta$).
Thus if $Y_{0} \sim \rho^{Y}$, we get $X_{T} \sim \rho^{Y}$. 
By Bayes theorem $X_{0} \sim \int_{M} \pi^{X|Y}(x|y) d\rho^{Y}(y)$, hence $Y_{T} \sim \int_{M} \pi^{X|Y}(x|y) d\rho^{Y}(y)$.
For the channel $Q_{t}^{-}$, we have
\begin{enumerate}
    \item $Q_{0}^{-}$ is the identity channel.
    \item Given input $\rho^{Y}$, the output at time $\eta$ is $\rho^{Y} Q_{\eta}^{-}(x) = \int_{M} \pi^{X|Y}(x|y) d\rho^{Y}(y)$.
    \item $\pi^{Y} Q_{t}^{-} = \pi^{X}Q_{\eta - t }$.
\end{enumerate}
Thus we see that the RHK step of proximal sampler can be viewed as going along the time reversed process. We now have the following result.

\begin{lemma}\label{Lemma_Backward_Riemannian}
    For the time-reversed process, we have 
    \begin{align*}
        \partial_{t} \Phi_{\pi^{Y} Q_{t}^{-}}(\rho^{Y} Q_{t}^{-}) = - \frac{1}{2} D_{\pi^{Y} Q_{t}^{-}} (\rho^{Y} Q_{t}^{-}).
    \end{align*}
\end{lemma}
\begin{proof}[Proof of Lemma~\ref{Lemma_Backward_Riemannian}]
Denote $\pi_{t}^{-} = \pi^{Y} Q_{t}^{-}$ and $\rho_{t}^{-} = \rho^{Y} Q_{t}^{-}$.
    The Fokker-Planck equation is 
    \begin{align*}
            \partial_{t} \pi_{t}^{-} &= - \Div (\pi_{t}^{-} \grad \log \pi_{t}^{-}) + \frac{1}{2} \Delta \pi_{t}^{-} = - \frac{1}{2} \Delta \pi_{t}^{-}, \\
            \partial_{t} \rho_{t}^{-} &= - \Div (\rho_{t}^{-} \grad \log \pi_{t}^{-}) + \frac{1}{2} \Delta \rho_{t}^{-}
            = \Div (\rho_{t}^{-} \grad \log \frac{\rho_{t}^{-}}{\pi_{t}^{-}} ) - \frac{1}{2} \Delta \rho_{t}^{-}.
    \end{align*}

    Hence 
    \begin{align*}
            2 \partial_{t} \Phi_{\pi_{t}^{-}}(\rho_{t}^{-})
            =& 2 \int_{M} \phi(\frac{\rho_{t}^{-}}{\pi_{t}^{-}})\frac{\partial}{\partial t} \pi_{t}^{-} 
            + \phi'(\frac{\rho_{t}^{-}}{\pi_{t}^{-}}) 
            \Big( \frac{\partial }{\partial t} \rho_{t}^{-} - (\frac{\partial}{\partial t} \pi_{t}^{-}) \frac{\rho_{t}^{-}}{\pi_{t}^{-}} \Big) dV_{g}(x)\\
            =& \int_{M} - \phi(\frac{\rho_{t}^{-}}{\pi_{t}^{-}}) \Delta \pi_{t}^{-} \\
            &+ \phi'(\frac{\rho_{t}^{-}}{\pi_{t}^{-}}) 
            \Big( 2\Div (\rho_{t}^{-} \grad \log \frac{\rho_{t}^{-}}{\pi_{t}^{-}} ) - \Delta \rho_{t}^{-} + \frac{\rho_{t}^{-}}{\pi_{t}^{-}} \Delta \pi_{t}^{-} \Big) dV_{g}(x)\\
            =& 2\int_{M} \phi'(\frac{\rho_{t}^{-}}{\pi_{t}^{-}}) 
            \Div (\rho_{t}^{-} \grad \log \frac{\rho_{t}^{-}}{\pi_{t}^{-}} ) dV_{g}(x)\\
            &+ \int_{M} - \phi(\frac{\rho_{t}^{-}}{\pi_{t}^{-}}) \Delta \pi_{t}^{-} 
             - \phi'(\frac{\rho_{t}^{-}}{\pi_{t}^{-}}) \Delta \rho_{t}^{-} + \phi'(\frac{\rho_{t}^{-}}{\pi_{t}^{-}})  \frac{\rho_{t}^{-}}{\pi_{t}^{-}} \Delta \pi_{t}^{-} dV_{g}(x)\\
            =& 2\int_{M} \phi'(\frac{\rho_{t}^{-}}{\pi_{t}^{-}}) 
            \Div (\rho_{t}^{-} \grad \log \frac{\rho_{t}^{-}}{\pi_{t}^{-}} )dV_{g}(x)
            + D_{\pi_{t}^{-}} (\rho_{t}^{-}) \\
            =& - 2 D_{\pi_{t}^{-}} (\rho_{t}^{-}) + D_{\pi_{t}^{-}} (\rho_{t}^{-}) = - D_{\pi_{t}^{-}} (\rho_{t}^{-}),
    \end{align*}
where we used the same steps as as in the proof Lemma \ref{Lemma_Forward_Riemannian} to obtain 
    \begin{align*}
        \int_{M} - \phi(\frac{\rho_{t}^{-}}{\pi_{t}^{-}}) \Delta \pi_{t}^{-} 
        - \phi'(\frac{\rho_{t}^{-}}{\pi_{t}^{-}}) \Delta \rho_{t}^{-} + \phi'(\frac{\rho_{t}^{-}}{\pi_{t}^{-}})  \frac{\rho_{t}^{-}}{\pi_{t}^{-}} \Delta \pi_{t}^{-} dV_{g}(x)
        = D_{\pi_{t}^{-}} (\rho_{t}^{-}),
    \end{align*}
    and used integration by parts, to obtain 
    \begin{align*}
            &2\int_{M} \phi'(\frac{\rho_{t}^{-}}{\pi_{t}^{-}}) 
            \Div (\rho_{t}^{-} \grad \log \frac{\rho_{t}^{-}}{\pi_{t}^{-}} ) dV_{g}(x) 
            = - 2 \int_{M} \langle \grad \phi'(\frac{\rho_{t}^{-}}{\pi_{t}^{-}}) , \rho_{t}^{-} \grad \log \frac{\rho_{t}^{-}}{\pi_{t}^{-}}\rangle dV_{g}(x) \\
            =& - 2 \mathbb{E}_{\rho_{t}^{-} } [\langle \grad \phi'(\frac{\rho_{t}^{-}}{\pi_{t}^{-}}) , \grad \log \frac{\rho_{t}^{-}}{\pi_{t}^{-}}\rangle ] 
            = - 2 D_{\pi_{t}^{-}} (\rho_{t}^{-}).
    \end{align*}
\end{proof}

\subsection{Convergence under LSI}\label{Section_Conv_LSI}
Now we prove the main theorem.

\begin{proof}[Proof of Theorem~\ref{Main_Theorem}]
We first prove the theorem assuming curvature is non-negative. For the general case, we only need to replace the LSI constant $\alpha_{t}, \alpha_{t}^{-}$. 
    \begin{enumerate}
        \item \textbf{The forward step.} 
        We know $\pi^{X} Q_{t}$ satisfies LSI with $\alpha_{t} := \frac{1}{t + \frac{1}{\alpha}}$.
        Using Lemma \ref{Lemma_Forward_Riemannian}, we have 
        \begin{align*}
            \partial_{t} H_{\pi^{X} Q_{t}} (\rho_{0}^{X} Q_{t}) 
            = - \frac{1}{2} J_{\pi^{X} Q_{t}} (\rho_{0}^{X} Q_{t}) 
            \le - \alpha_{t} H_{\pi^{X} Q_{t}} (\rho_{0}^{X} Q_{t}).
        \end{align*}
        This implies $ H_{\pi^{X} Q_{t}} (\rho_{0}^{X} Q_{t}) \le e^{-A_{t}} H_{\pi^{X}} (\rho_{0}^{X})$ 
        where, $A_{t} = \int_{0}^{t} \alpha_{s} ds = \log (1 + t\alpha)$.
        We also have $e^{-A_{t}} = (1 + t\alpha)^{-1}$. As a result, 
        \begin{align*}
            H_{\pi^{X} Q_{\eta}} (\rho_{0}^{X} Q_{\eta}) \le 
            \frac{H_{\pi^{X}} (\rho_{0}^{X})}{1 + \eta \alpha}.
        \end{align*}
        \item \textbf{The backward step.}
        Using Lemma \ref{Lemma_Backward_Riemannian}, we have 
        \begin{align*}
            \partial_{t} H_{\pi^{Y} Q_{t}^{-}} (\rho_{0}^{Y} Q_{t}^{-}) 
            = - \frac{1}{2} J_{\pi^{Y} Q_{t}^{-}} (\rho_{0}^{Y} Q_{t}^{-}) 
            \le - \alpha_{t}^{-} H_{\pi^{Y} Q_{t}^{-}} (\rho_{0}^{Y} Q_{t}^{-}). 
        \end{align*}
        Since $\pi^{Y} Q_{t}^{-} = \pi^{X}Q_{\eta - t }$,
        we know the LSI constant for $\pi^{Y} Q_{t}^{-}$ is $\alpha_{t}^{-} := \frac{1}{(\eta - t) + \frac{1}{\alpha}}$.
        Same as in step 1, we get $A_{\eta}^{-} = \int_{0}^{\eta} \alpha_{t}^{-} dt = \log (1 + \alpha \eta)$.
        As a result, 
        \begin{align*}
            H_{\pi^{Y} Q_{\eta}^{-}} (\rho_{0}^{Y} Q_{\eta}^{-})  \le 
            \frac{H_{\pi^{Y} } (\rho_{0}^{Y} ) }{1 + t\alpha }.
        \end{align*}
        \item \textbf{Putting together.} 
        We have $\pi^{Y} = \pi^{X} Q_{\eta}$, $\rho_{0}^{Y} = \rho_{0}^{X} Q_{\eta}$. 
        Denote $\rho_{1}^{X} = \rho_{0}^{Y} Q_{\eta}^{-}$, we get 
        \begin{align*}
            H_{\pi^{X}} (\rho_{1}^{X} ) = 
            H_{\pi^{Y} Q_{\eta}^{-}} (\rho_{0}^{Y} Q_{\eta}^{-}) \le 
            \frac{H_{\pi^{Y} } (\rho_{0}^{Y} ) }{(1 + t\alpha) }
            = \frac{H_{\pi^{X} Q_{\eta} } (\rho_{0}^{X} Q_{\eta}) }{(1 + t\alpha) }
            \le 
            \frac{H_{\pi^{X}} (\rho_{0}^{X})}{(1 + t\alpha)^{2} }.
        \end{align*}
        \item \textbf{Negative curvature.} 
        For negative curvature, we use 
        $\alpha_{t}$ as in Proposition \ref{LSI_Propagation_Heat}
        (the value to be integrated is $\frac{\kappa}{1-e^{-\kappa t} + \kappa d_0 e^{-\kappa t}}$
        where $\frac{1}{\alpha} := d_{0}$). We compute the integral
        \begin{align*}
            \int_{0}^{t} \frac{1}{((\frac{1}{\alpha} -\frac{1}{\kappa})e^{-x} + \frac{1}{\kappa})} dx 
            = \log (\alpha(e^{\kappa t} - 1) + \kappa) - \log (\kappa)
            = \log (\frac{\alpha(e^{\kappa t} - 1) + \kappa}{\kappa}).
        \end{align*}
        Hence we have $ H_{\pi^{X}} (\rho_{k}^{X} ) \le H_{\pi^{X}} (\rho_{0}^{X}) (\frac{\kappa}{\alpha(e^{\kappa \eta} - 1) + \kappa})^{2k}$.
        
        Observe that in general, for $x \in [0, 1]$ we have that $1 - \frac{x}{2} \ge e^{-x}$.
        Thus for $\eta < \frac{1}{|\kappa|}$, we have $|\kappa| \eta < 1$, hence $1 -  \frac{|\kappa|}{2}\eta \ge e^{-|\kappa| \eta}$.
        This implies $\frac{1 - e^{-|\kappa| \eta}}{|\kappa|} \ge \frac{1}{2}\eta$.
        On the other hand, we have $1 - x \le e^{-x}$, which implies $\frac{1 - e^{-|\kappa| \eta}}{|\kappa|} \le \eta$.
        So we have $\frac{\alpha(e^{\kappa \eta} - 1) + \kappa}{\kappa} = 1 + \alpha\frac{1 - e^{-|\kappa| \eta}}{|\kappa|} = \Theta(1 + \alpha \eta)$.
    \end{enumerate}

\end{proof}

\section{Proof of Theorem \ref{TV_Inexact_BM_Inexact_RHK}}\label{Sec_Proof_Inexact_Theorem}

Recall that $\rho_{k}^{X}(x)$, $\rho_{k}^{Y}(y)$ denote the distribution generated by Algorithm \ref{Manifold_Proximal_Sampler_Ideal}, 
assuming exact Brownian motion and exact RHK. This notation is applied for all $k$. 
For practical implementation, using inexact RHK and inexact Brownian motion through all the iterations, 
we denote the corresponding distribution by $\tilde{\rho}_{k}^{X}(x)$, $\tilde{\rho}_{k}^{Y}(y)$ respectively.

Note that at iteration $k-1$, we are at distribution $\tilde{\rho}_{k-1}^{X}(x)$.
Denote $\hat{\rho}_{k-1}^{Y}(y)$ to be the distribution obtained from $\tilde{\rho}_{k-1}^{X}(x)$ using exact Brownian motion. 
(Note that $\tilde{\rho}_{k-1}^{Y}(y)$ denote the distribution obtained from $\tilde{\rho}_{k-1}^{X}(x)$ using inexact Brownian motion).  

We now prove Lemma \ref{Lemma_Propagation_Error_TV}.

\begin{proof}[Proof of Lemma~\ref{Lemma_Propagation_Error_TV}]
    Using triangle inequality, we have 
    \begin{align*}
            &\|\rho_{k}^{X}(x) - \tilde{\rho}_{k}^{X}(x)\|_{TV}
            = \Big\|\int \rho_{k-1}^{Y}(y) \pi^{X|Y}(x|y) dy  - \int \tilde{\rho}_{k-1}^{Y}(y) \hat{\pi}^{X|Y}(x|y) dy\Big\|_{TV} \\
            \le& \Big\|\int \rho_{k-1}^{Y}(y) (\pi^{X|Y}(x|y) - \hat{\pi}^{X|Y}(x|y))  dy\Big\|_{TV} 
            + \Big\| \int (\tilde{\rho}_{k-1}^{Y}(y) - \rho_{k-1}^{Y}(y) ) \hat{\pi}^{X|Y}(x|y) dy\Big\|_{TV}.
    \end{align*}
    The first part can be bounded by $\zeta_{\mathsf{RHK}}$:
    \begin{align*}
            &\Big\|\int \rho_{k-1}^{Y}(y) (\pi^{X|Y}(x|y) - \hat{\pi}^{X|Y}(x|y))  dy\Big\|_{TV} 
            \le \int \rho_{k-1}^{Y}(y) \Big\|\pi^{X|Y}(x|y) - \hat{\pi}^{X|Y}(x|y) \Big\|_{TV}  dy \\
            \le & \zeta_{\mathsf{RHK}}.
    \end{align*} 
    For the second part, we have 
    \begin{align*}
            &\Big\| \int (\tilde{\rho}_{k-1}^{Y}(y) - \rho_{k-1}^{Y}(y) ) \hat{\pi}^{X|Y}(x|y) dy\Big\|_{TV} \\
            = & \frac{1}{2}\int \Big| \int (\tilde{\rho}_{k-1}^{Y}(y) - \rho_{k-1}^{Y}(y) ) \hat{\pi}^{X|Y}(x|y) dy \Big| dx
            \le \frac{1}{2}\int \Big|\tilde{\rho}_{k-1}^{Y}(y) - \rho_{k-1}^{Y}(y) \Big| \int \hat{\pi}^{X|Y}(x|y) dx dy \\
            = & \|\tilde{\rho}_{k-1}^{Y}(y) - \rho_{k-1}^{Y}(y)\|_{TV}
            \le \|\tilde{\rho}_{k-1}^{Y}(y) - \hat{\rho}_{k-1}^{Y}(y)\|_{TV} + \|\hat{\rho}_{k-1}^{Y}(y) - \rho_{k-1}^{Y}(y)\|_{TV} \\
            \le & \zeta_{\mathsf{MBI}} + \|\tilde{\rho}_{k-1}^{X}(x) - \rho_{k-1}^{X}(x)\|_{TV}.
    \end{align*}
    Here, the last inequality follows from Lemma \ref{Lemma_TV_heat}. Together, we have
    \begin{align*}
        \|\rho_{k}^{X}(x) - \tilde{\rho}_{k}^{X}(x)\|_{TV}
        \le \zeta_{\mathsf{RHK}} + \zeta_{\mathsf{MBI}} + \|\tilde{\rho}_{k-1}^{X}(x) - \rho_{k-1}^{X}(x)\|_{TV}.
    \end{align*}
    Iteratively applying this inequality and noting that $\|\tilde{\rho}_{0}^{X}(x) - \rho_{0}^{X}(x)\|_{TV} = 0$, 
    we obtain $\|\rho_{k}^{X}(x) - \tilde{\rho}_{k}^{X}(x)\|_{TV} \le k (\zeta_{\mathsf{RHK}} + \zeta_{\mathsf{MBI}})$.
\end{proof}

Recall that Pinsker's inequality states $\|\mu - \nu\|_{TV} \le \sqrt{\frac{1}{2}H_{\nu}(\mu)}$.

\begin{proof}[Proof of Theorem~\ref{TV_Inexact_BM_Inexact_RHK}]  Using Pinsker's inequality, we have 
    \begin{align*}
        \|\rho_{k}^{X} - \pi^{X}\|_{TV} 
        \le \sqrt{\frac{1}{2} H_{\pi^{X}}(\rho_{k}^{X})} 
        \le \sqrt{\frac{1}{2} \frac{H_{\pi^{X}} (\rho_{0}^{X})}{(1 + \eta\alpha)^{2k} } }
        \le \frac{1}{2}\varepsilon.
    \end{align*}
    We want to bound $\|\rho_{k}^{X} - \pi^{X}\|_{TV} \le \frac{1}{2}\varepsilon$.
    It suffices to have $\frac{H_{\pi^{X}} (\rho_{0}^{X})}{(1 + \eta\alpha)^{2k} } \le \frac{1}{2} \varepsilon^{2}$.
    Hence we need $\log (\frac{2H_{\pi^{X}} (\rho_{0}^{X})}{\varepsilon^{2}}) \le 2k\log(1 + \eta\alpha) $, i.e., 
    $k = \mathcal{O}\left(\frac{\log \frac{H_{\pi^{X}} (\rho_{0}^{X})}{\varepsilon^{2}}}{\log (1 + \eta\alpha)}\right) $.

    For small step size $\eta$, we have $\frac{1}{\log(1 + \eta \alpha)} = \mathcal{O}(\frac{1}{\eta \alpha})$. 
    Hence $k = \mathcal{O}\left(\frac{1}{\eta \alpha} \log \frac{H_{\pi^{X}} (\rho_{0}^{X})}{\varepsilon^{2}}) = \tilde{\mathcal{O}}(\frac{1}{\eta}\log \frac{1}{\varepsilon}\right)$.

    Recall that by assumption, $\frac{1}{\eta} = \tilde{\mathcal{O}}(\log \frac{1}{\zeta})$. 
    We pick $\zeta = \frac{\varepsilon}{\log^{2} \frac{1}{\varepsilon}}$
    and consequently $\frac{1}{\eta} = \tilde{\mathcal{O}}(\log \frac{\log^{2} \frac{1}{\varepsilon}}{\varepsilon}) = \tilde{\mathcal{O}}(\log \frac{1}{\varepsilon} + 2\log \log \frac{1}{\varepsilon}) = \tilde{\mathcal{O}}(\log \frac{1}{\varepsilon})$.
    It follows that 
    \begin{align*}
        k = \tilde{\mathcal{O}}(\frac{1}{\eta}\log \frac{1}{\varepsilon})
        = \tilde{\mathcal{O}}(\log^{2} \frac{1}{\varepsilon}). 
    \end{align*}
    The result then follows from triangle inequality:
    \begin{align*}
        \|\tilde{\rho}_{k}^{X} - \pi^{X}\|_{TV} \le \|\tilde{\rho}_{k}^{X} - \rho_{k}^{X}\|_{TV} + \|\rho_{k}^{X} - \pi^{X}\|_{TV} 
        \le k (\zeta_{\mathsf{RHK}} + \zeta_{\mathsf{MBI}}) + \frac{1}{2}\varepsilon = \tilde{\mathcal{O}}(\varepsilon)
    \end{align*}
    where $k \zeta = \tilde{\mathcal{O}}(\frac{\varepsilon}{\log^{2} \frac{1}{\varepsilon}} \log^{2} \frac{1}{\varepsilon}) = \tilde{\mathcal{O}}(\varepsilon)$.

\end{proof}

\section{Verification of Assumption \ref{Assumption_Oracle_TV_quality}}\label{Section_L_infty_Varadhan}

In this section, we consider implementing inexact oracles through the truncation method. Recall that we assume $M$ is a compact manifold, which is a homogeneous space. 


We use $\hat{\pi}^{Y|X}, \hat{\pi}^{X|Y}$ to denote the output of MBI oracle and RHK when rejection sampling is exact. More precisely, since we use the truncated series to approximate heat kernel, we have $\hat{\pi}^{Y|X} \propto \nu_{l}(\eta, x, y)$  and $ \hat{\pi}^{X|Y} \propto e^{-f(x)}\nu_{l}(\eta, x, y) $. When rejection sampling is not exact, i.e., there exists $z \in M$ s.t. $V(z) > 1$, we denote the output to be $\overline{\pi}^{Y|X}, \overline{\pi}^{X|Y}$ for inexact Brownian motion and inexact RHK, respectively. 

In subsection \ref{Subsection_exact_rej}, we prove Proposition \ref{Prop_Verify_Assumption}, i.e., 
$\|\hat{\pi}^{X|Y} - \pi^{X|Y}\|_{TV} = \tilde{\mathcal{O}}(\zeta)$ and $\|\hat{\pi}^{Y|X} - \pi^{Y|X}\|_{TV} = \tilde{\mathcal{O}}(\zeta)$ with $\zeta = \frac{\varepsilon}{\log^{2} \frac{1}{\varepsilon}}$. Then we know 
$\hat{\pi}^{Y|X}$ and $\hat{\pi}^{X|Y}$ satisfy Assumption \ref{Assumption_Oracle_TV_quality}.

In subsection \ref{Subsection_inexact_rej}, we consider a more general setting, where the acceptance rate is allowed to exceed $1$ at some unimportant region. 
We show that on $\mathcal{S}^{d}$, for certain choices of parameters, $\|\hat{\pi}^{X|Y} - \overline{\pi}^{X|Y}\|_{TV} = \tilde{\mathcal{O}}(\zeta)$ and $\|\hat{\pi}^{Y|X} - \overline{\pi}^{Y|X}\|_{TV} = \tilde{\mathcal{O}}(\zeta)$. This means that allowing the acceptance rate to exceed $1$ in unimportant regions would not cause a significant bias for rejection sampling. 
It then follows from triangle inequality that $\overline{\pi}^{Y|X}$ and $\overline{\pi}^{X|Y}$ satisfy Assumption \ref{Assumption_Oracle_TV_quality}.

\subsection{Exact rejection sampling}\label{Subsection_exact_rej}

We prove Proposition \ref{Prop_Verify_Assumption}, i.e., 
verify that Assumption \ref{Assumption_Oracle_TV_quality} is satisfied with $\zeta = \frac{\varepsilon}{\log^{2} \frac{1}{\varepsilon}}$ 
as required in Theorem \ref{TV_Inexact_BM_Inexact_RHK}.

\subsubsection{Analysis in total variation distance}

The first step is to bound the total variation distance, under the assumption that heat kernel evaluation is of high accuracy.
We consider the following characterization of total variation distance (see Lemma \ref{Prop_TV}):
\begin{align*}
        \|\rho_{1} - \rho_{2}\|_{TV}
        = \frac{1}{2} \int_{M} |\rho_{1}(x) - \rho_{2}(x)| dV_{g}(x).
\end{align*}

\begin{proposition}\label{Prop_truncation_1}
    Let $M$ be a compact manifold. Let $\zeta$ be the desired accuracy. 
    Assume for all $y \in M$ we have 
    $\int_{M} |\nu(\eta, x, y) - \nu_{l}(\eta, x, y)| dV_{g}(x) = \tilde{\mathcal{O}}(\zeta)$ 
    and for all $x \in M$ we have $\int_{M} |\nu(\eta, x, y) - \nu_{l}(\eta, x, y)| dV_{g}(y) = \tilde{\mathcal{O}}(\zeta)$.
    Then $\|\hat{\pi}^{X|Y} - \pi^{X|Y}\|_{TV} = \tilde{\mathcal{O}}(\zeta)$ 
    and $\|\hat{\pi}^{Y|X} - \pi^{Y|X}\| = \tilde{\mathcal{O}}(\zeta)$.
\end{proposition}
\begin{proof}[Proof of Proposition~\ref{Prop_truncation_1}]

    \textbf{Step 1.} Note that $A_{1} := \sup_{x \in M} e^{-f(x)}, A_{2} := \inf_{x \in M} e^{-f(x)}$ are positive constants independent of $t$.
    Denote $Z_{1} = \int_{M} e^{-f(x)} \nu_{l}(\eta, x, y) dV_{g}(x)$ 
    and $Z_{2} = \int_{M} e^{-f(x)}\nu(\eta, x, y) dV_{g}(x)$.
    We know 
    \begin{align*}
            |Z_{2} - Z_{1}| &= 
            |\int_{M} e^{-f(x)}\nu(\eta, x, y) - e^{-f(x)} \nu_{l}(\eta, x, y) dV_{g}(x)| \\
            &\le A_{1} \int_{M} |\nu(\eta, x, y) - \nu_{l}(\eta, x, y)| dV_{g}(x)
            = \tilde{\mathcal{O}}(\zeta).
    \end{align*}
    Hence, we have
    \begin{align*}
            \|\hat{\pi}^{X|Y} - \pi^{X|Y}\|_{TV}
        &\le \frac{1}{2} \int_{M} | \frac{e^{-f(x)} \nu_{l}(\eta, x, y)}{\int_{M} e^{-f(x)} \nu_{l}(\eta, x, y) dV_{g}(x)} 
        - \frac{e^{-f(x)}\nu(\eta, x, y)}{\int_{M} e^{-f(x)}\nu(\eta, x, y) dV_{g}(x)} | dV_{g}(x) \\
        &\le \frac{1}{2} \int_{M} \frac{|Z_{2}e^{-f(x)} \nu_{l}(\eta, x, y) - Z_{1}e^{-f(x)}\nu(\eta, x, y)|}{Z_{1}Z_{2}} dV_{g}(x)\\
        &\le \int_{M} \frac{\min\{Z_{1}, Z_{2}\} \cdot |e^{-f(x)} \nu_{l}(\eta, x, y) - e^{-f(x)}\nu(\eta, x, y)| }{2Z_{1}Z_{2}} dV_{g}(x) \\
        &\quad + \int_{M} \frac{|Z_{2} - Z_{1}| \cdot \max\{e^{-f(x)}\nu(\eta, x, y), e^{-f(x)} \nu_{l}(\eta, x, y)\}}{2Z_{1}Z_{2}} dV_{g}(x)\\
        &\le \tilde{\mathcal{O}}(\zeta) + \tilde{\mathcal{O}}(\int_{M} |Z_{2} - Z_{1}| \cdot \max\{\nu(\eta, x, y), \nu_{l}(\eta, x, y)\} dV_{g}(x)) 
        = \tilde{\mathcal{O}}(\zeta),
    \end{align*}
    where by Lemma~\ref{Lemma_Normalizing_constant_2}, we obtain  $\frac{\min\{Z_{1}, Z_{2}\} }{2Z_{1}Z_{2}} = \tilde{\mathcal{O}}(1)$
    and $$\int_{M} \frac{\max\{e^{-f(x)}\nu(\eta, x, y), e^{-f(x)} \nu_{k}(\eta, x, y)\}}{2Z_{1}Z_{2}} dV_{g}(x) = \tilde{\mathcal{O}}(1).$$
\vspace{0.2in}

    \textbf{Step 2.}    Denote $Z_{l} = \int_{M} \nu_{l}(\eta, x, y) dV_{g}(y)$ to be the normalizaing constant for $\nu_{l}$.
    Since $\nu$ is the heat kernel, we simply have $\int_{M} \nu(\eta, x, y) dV_{g}(y) = 1$. 
    It holds that 
    \begin{align*}
        \hat{\pi}^{Y|X} = \frac{\nu_{l}(\eta, x, y)}{\int_{M} \nu_{l}(\eta, x, y) dV_{g}(y) }\quad\text{and}\quad \pi^{Y|X} = \nu(\eta, x, y).
    \end{align*}
    Then,
    \begin{align*}
            \|\hat{\pi}^{Y|X} - \pi^{Y|X}\|_{TV}
        &\le \frac{1}{2} \int_{M} | \frac{\nu_{l}(\eta, x, y)}{\int_{M} \nu_{l}(\eta, x, y) dV_{g}(y) } 
        - \nu(\eta, x, y)| dV_{g}(y) \\
        &\le \frac{1}{2} \int_{M} \frac{| \nu_{l}(\eta, x, y) - Z_{l} \nu(\eta, x, y)|}{Z_{l}} dV_{g}(y)\\
        &\le \int_{M} \frac{\min\{Z_{l}, 1\} \cdot | \nu_{l}(\eta, x, y) -  \nu(\eta, x, y)| }{2Z_{l}} dV_{g}(y) \\
        &\quad + \int_{M} \frac{|1 - Z_{l}| \cdot \max\{ \nu(\eta, x, y), \nu_{l}(\eta, x, y)\}}{2Z_{l}} dV_{g}(y)\\
        &= \tilde{\mathcal{O}}(\zeta).
    \end{align*}
\end{proof}

\begin{theorem}[Theorem 5.3.4 in \cite{hsu2002stochastic}]\label{Thm534_Hsu} Let $M$ be a compact Riemannian manifold. There exist positive constants $C_{1}, C_{2}$ 
    such that for all $(t, x, y) \in (0, 1) \times M \times M$, 
    \begin{align*}
        \frac{C_{1}}{t^{d/2}} e^{-\frac{d(x, y)^{2}}{2t}} \le \nu(t, x, y) \le \frac{C_{2}}{t^{(2d-1)/2}} e^{-\frac{d(x, y)^{2}}{2t}}.
    \end{align*}
\end{theorem}

\begin{lemma}\label{Lemma_Normalizing_constant_2}
    We have $1/\int_{M} e^{-f(x)}\nu(\eta, x, y) dV_{g}(x) = \tilde{\mathcal{O}}(1)$.
\end{lemma}
\begin{proof}[Proof of Lemma~\ref{Lemma_Normalizing_constant_2}]
    Using lower bound of heat kernel  from Theorem \ref{Thm534_Hsu}, we have 
    \begin{align*}
            &\int_{M} e^{-f(x)}\nu(\eta, x, y) dV_{g}(x) \\
            \ge& A_{2} \int_{M} \frac{C_{1}}{\eta^{d/2}} \exp(-\frac{d(x, y)^{2}}{2\eta}) dV_{g}(x) 
            = \frac{A_{2}C_{1}}{\eta^{d/2}} \int_{M} \exp(-\frac{d(x, y)^{2}}{2\eta}) dV_{g}(x)  \\
            \ge& \frac{A_{2}C_{1}}{\eta^{d/2}} \frac{\eta^{d/2}}{C_{4}}
            = \frac{A_{2}C_{1}}{C_{4}}.
    \end{align*}
    Hence
    \begin{align*}
        1/\int_{M} e^{-f(x)}\nu(\eta, x, y) dV_{g}(x) = \tilde{\mathcal{O}}(1).
    \end{align*}
\end{proof}

\subsubsection{Analysis of truncation error}

Now we discuss the truncation level needed to guarantee a high accuracy evaluation of heat kernel 
as required in Proposition \ref{Prop_truncation_1}.
\begin{proposition}\label{Prop_truncation_level}
    Let $M$ be a compact manifold, and assume $M$ is a homogeneous space. 
    With $\frac{1}{\eta} = \tilde{\mathcal{O}}(\log \frac{1}{\varepsilon})$ and 
    $\zeta = \frac{\varepsilon}{\log^{2} \frac{1}{\varepsilon}}$, 
    to reach $\|\nu(\eta, x, y) - \nu_{l}(\eta, x, y)\|_{L^{2}}^{2} = \tilde{\mathcal{O}}(\zeta)$
    we need $l = \Poly (\log \frac{1}{\varepsilon})$. 
    Consequently, to achieve 
    \begin{align*}
            &\int_{M} |\nu(\eta, x, y) - \nu_{l}(\eta, x, y)| dV_{g}(x) = \tilde{\mathcal{O}}(\zeta) \qquad\text{and}\\
          & \int_{M} |\nu(\eta, x, y) - \nu_{l}(\eta, x, y)| dV_{g}(y) = \tilde{\mathcal{O}}(\zeta),
    \end{align*}
    we need $l = \Poly (\log \frac{1}{\varepsilon})$. 
\end{proposition}

\begin{proof}[Proof of Proposition~\ref{Prop_truncation_level}]
    Following \citet[Proof of Proposition 21]{azangulov2022stationary} we have 
    \begin{align*}
        \|\nu(\eta, x, y) - \nu_{l}(\eta, x, y)\|_{L^{2}}^{2}
        \le C' l \frac{1}{\eta^{2}}e^{-\frac{\eta^{2} l^{2/d}}{C}}.
    \end{align*}
    Take $\frac{1}{\eta} = \log \frac{1}{\varepsilon}$.
    Recall that in Theorem \ref{TV_Inexact_BM_Inexact_RHK} we require $\zeta = \frac{\varepsilon}{\log^{2} \frac{1}{\varepsilon}}$.
    Requiring $ C' l \frac{1}{\eta^{2}}e^{-\frac{\eta^{2} l^{2/d}}{C}} = \tilde{\mathcal{O}}(\zeta)$
    is equivalent to 
    \begin{align*}
        C' l \log^{2}\frac{1}{ \varepsilon}e^{-\frac{\frac{1}{\log^{2}\frac{1}{\varepsilon}} l^{2/d}}{C}} \le \zeta = \frac{\varepsilon}{\log^{2} \frac{1}{\varepsilon}}.
    \end{align*}
    Take log on both sides, we get $-\frac{\frac{1}{\log^{2}\frac{1}{\varepsilon}} l^{2/d}}{C} \le \log \frac{\varepsilon }{C' l \log^{4}\frac{1}{ \varepsilon}} $.
    This further implies 
    \begin{align*}
            l^{2/d} &\ge - \log \frac{\varepsilon }{C' l \log^{4}\frac{1}{ \varepsilon}} C\log^{2}\frac{1}{ \varepsilon}
            = (4\log \log \frac{1}{\varepsilon} + \log \frac{1}{\varepsilon} + \log l + C'') C\log^{2}\frac{1}{ \varepsilon}.
    \end{align*}
    It suffices to take $l = \Poly (\log \frac{1}{\varepsilon})$.
    We verify that $l = \Poly (\log \frac{1}{\varepsilon})$ can guarantee the bound:
    \begin{align*}
            C' l \frac{1}{\eta^{2}}e^{-\frac{\eta^{2} l^{2/d}}{C}}
            = C' l \log^2 \frac{1}{\varepsilon} e^{-\frac{l^{2/d}}{C \log^2 \frac{1}{\varepsilon}}} 
            = \Poly (\log \frac{1}{\varepsilon}) e^{-\Poly (\log \frac{1}{\varepsilon})} 
            = \tilde{\mathcal{O}}(\frac{\varepsilon}{\log^{2} \frac{1}{\varepsilon}}) = \tilde{\mathcal{O}}(\zeta).
    \end{align*}

    On a homogeneous space, both $\nu$ and $\nu_{l}$ are stationary \citep{azangulov2022stationary}.
    Hence $\int_{M} |\nu(\eta, x, y) - \nu_{l}(\eta, x, y) | dV_{g}(x) $ does not depend on $y$.
    and $\int_{M} |\nu(\eta, x, y) - \nu_{l}(\eta, x, y) | dV_{g}(y)$ does not depend on $x$.
    Therefore using Jensen's inequality, 
    \begin{align*}
        \int_{M} |\nu(\eta, x, y) - \nu_{l}(\eta, x, y) | dV_{g}(x) 
           & = \tilde{\mathcal{O}} (\|\nu(\eta, x, y) - \nu_{l}(\eta, x, y)\|_{L_{1}})\\
            &\le \tilde{\mathcal{O}} (\|\nu(\eta, x, y) - \nu_{l}(\eta, x, y)\|_{L_{2}}).
    \end{align*}
    Note that the same holds for $\int_{M} |\nu(\eta, x, y) - \nu_{l}(\eta, x, y) | dV_{g}(y)$.
    Hence we get the desired bound, i.e., 
    $\int_{M} |\nu(\eta, x, y) - \nu_{l}(\eta, x, y)| dV_{g}(x) = \tilde{\mathcal{O}}(\zeta)$ 
    and $\int_{M} |\nu(\eta, x, y) - \nu_{l}(\eta, x, y)| dV_{g}(y) = \tilde{\mathcal{O}}(\zeta)$.
\end{proof}

\subsection{Truncation method on hypersphere}\label{Subsection_inexact_rej}

Let $M$ be a hypersphere. 
In the last subsection, we discussed some existing results which provided a bound on the $L_{2}$ norm of $\nu_{l} - \nu$. For hypersphere, we can derive a bound in $L_{\infty}$ norm, see subsection \ref{Subsection_Sphere_L_infty}. We also consider the situation that the acceptance rate in rejection sampling might exceed $1$, and show that for such a situation, rejection sampling can still produce a high-accuracy sample. 

Let $V_{\mathsf{MBI}}(y), V_{\mathsf{RHK}}(x)$ denote the acceptance rate in rejection sampling.
Recall that $\hat{\pi}^{Y|X} \propto \nu_{l}(\eta, x, y)$ 
and $ \hat{\pi}^{X|Y} \propto e^{-f(x)}\nu_{l}(\eta, x, y) $. In the actual rejection sampling implementation, if for example in Brownian motion implementation, it happens that there exists $y \in M$, s.t. $V(y) > 1$, then the output for rejection sampling will not be equal to $\hat{\pi}^{Y|X}$. 
For such situations, denote $\overline{V}_{\mathsf{MBI}}(y) = \min\{1, V_{\mathsf{MBI}}(y)\}$ and $\overline{V}_{\mathsf{RHK}}(x) = \min\{1, V_{\mathsf{RHK}}(x)\}$. Note that $\overline{V}_{\mathsf{MBI}}(y)$ and $\overline{V}_{\mathsf{RHK}}(x)$ are the actual acceptance rate in rejection sampling.
we denote the corresponding rejection sampling output by $\overline{\pi}^{Y|X}, \overline{\pi}^{X|Y}, $ respectively.

Intuitively, the region $B_{x}(r)$ near $x$ carries most of the probability for both Riemannian Gaussian distribution $\mu(t, x, y)$ as well as Brownian motion $\nu(t, x, y)$, when the variable $t$ is suitably small.
Thus instead of choosing parameter to guarantee $V_{\mathsf{RHK}}(x), V_{\mathsf{MBI}}(y)\le 1, \forall x, y \in M$, it suffices to guarantee $V_{\mathsf{RHK}}(x) \le 1, \forall x \in B_{y}(r)$
and $V_{\mathsf{MBI}}(y)\le 1, \forall y \in B_{x}(r)$ for some $r$. Define
\begin{align*}
    V_{\mathsf{MBI}}(y) &:= \frac{\exp(\log \nu_{l}(\eta, x, y) - \log \nu_{l}(\eta, x, x) + C_{\mathsf{MBI}})}{\exp(-\frac{d(x, y)^{2}}{2(s\eta)})} \le 1, &\forall y \in B_{x}(r), \\
        V_{\mathsf{RHK}}(x) &:= \frac{\exp(-f(x) + \log \nu_{l}(\eta, x, y) + f(y) - \log \nu_{l}(\eta, y, y) + C_{\mathsf{RHK}})}{\exp(-\frac{1}{2t} d(x, y)^{2})} \le 1, &\forall x \in B_{y}(r).
\end{align*}
\begin{proposition}\label{Prop_Verify_Assumption_inexact_rej}
    Let $M = \mathcal{S}^{d}$ be a hypersphere. with $\frac{1}{\eta }= L_{1}^{2} d \log \frac{1}{\varepsilon} $, $\frac{1}{t} = L_{1}^{2}(d-2)  \log \frac{1}{\varepsilon}$ and truncation level $l = \Poly(\log \frac{1}{\zeta})$, there exists parameters $t, C_{\mathsf{MBI}}, C_{\mathsf{RHK}}$ s.t. 
    $\overline{\pi}^{X|Y}, \overline{\pi}^{Y|X}$ satisfy Assumption \ref{Assumption_Oracle_TV_quality}.
\end{proposition}
\begin{proof}
    See Proposition \ref{Prop_BM_inexact_rej} and Proposition \ref{Prop_RHK_inexact_rej}.
\end{proof}

\subsubsection{Proof of Proposition \ref{Prop_Verify_Assumption_inexact_rej}}

\begin{proposition}\label{Prop_BM_inexact_rej}
    Let $M$ be hypersphere $\mathcal{S}^{d}$ so that the truncation error bound can be proved in $L_{\infty}$.
    Consider Algorithm \ref{Inexact_BM} with $ t = \eta s $ where $s > 1$ is a constant that does not depend on $\eta, \varepsilon$. 
    For small $\varepsilon$, the error for inexact rejection sampling with $\nu_{l}$ is of order $\zeta$, i.e., 
    $\|\hat{\pi}^{Y|X} - \overline{\pi}^{Y|X}\|_{TV} = \tilde{\mathcal{O}}(\zeta)$.
    Hence by triangle inequality, 
    $\|\pi^{Y|X} - \overline{\pi}^{Y|X}\|_{TV} = \tilde{\mathcal{O}}(\zeta)$.
\end{proposition}

\begin{proof}[Proof of Proposition~\ref{Prop_BM_inexact_rej}]
    Recall that we require $\frac{1}{\eta} = \tilde{\mathcal{O}}(\log \frac{1}{\varepsilon} )$.
    Write $ \frac{1}{\eta} = C_{\eta}\log \frac{1}{\varepsilon} $ where $C_{\eta}$ is some constant that does not depend on $\eta$.
    Then we can write $e^{- \frac{1}{\eta C_{\eta}}} = \varepsilon$.

    \begin{enumerate}
        \item \textbf{Step 1: Locate a centered neighborhood where the acceptance rate is bounded by $1$. }
        
        Let $T > 0$ be fixed. Consider $\eta \in (0, T]$. 
        As in \citet[Proof of Lemma 5.4.2]{hsu2002stochastic}, there exists some constant $C_{5}$ that depends on $T$, s.t. 
        \begin{align*}
            \nu(\eta, x, y) \le \frac{C_{5}}{\eta^{\frac{d}{2}}} \exp (-\frac{d(x, y)^{2}}{2\eta})
        \end{align*}
        for all $d(x, y) \le r$. 
        Here without loss of generality, the variable $r$ satisfies $\frac{r^{2}}{2} = \frac{2-\frac{1}{s}}{C_{\eta}} $.

        Then for all $y$ s.t. $\frac{d(x, y)^{2}}{2} \le \frac{1}{C_{\eta}} + \frac{\eta d}{2} \log \frac{1}{\eta} =: \frac{ r_{0}^{2}}{2}$, we have 
        \begin{align*}
            \nu_{l}(\eta, x, y) \le \nu(\eta, x, y) + \varepsilon \le \frac{C_{5}}{\eta^{\frac{d}{2}}} \exp (-\frac{d(x, y)^{2}}{2\eta}) + e^{- \frac{1}{\eta C_{\eta}}}
            = \frac{C_{5}}{\eta^{\frac{d}{2}}} \exp (-\frac{d(x, y)^{2}}{2\eta}) (1 + \delta(x, \eta))
        \end{align*}
        with $\delta(x, \eta) := \frac{e^{- \frac{1}{\eta C_{\eta}}}}{\frac{C_{5}}{\eta^{\frac{d}{2}}} \exp (-\frac{d(x, y)^{2}}{2\eta})}$ 
        satisfying
        \begin{align*}
            \delta(x, \eta) = \frac{e^{- \frac{1}{\eta C_{\eta}}}}{\frac{C_{5}}{\eta^{\frac{d}{2}}} \exp (-\frac{d(x, y)^{2}}{2\eta})}
            = \eta^{\frac{d}{2}}\frac{ \exp (\frac{1}{\eta}(\frac{d(x, y)^{2}}{2} - \frac{1}{C_{\eta}}) )}{C_{5} }
            \le \frac{1}{C_{5} }, \forall y \in B_{x}(r_{0}),
        \end{align*}
        which further implies
        \begin{align*}
                - \log \nu_{l}(\eta, x, y) &\ge - \log C_{5} + \frac{d}{2}\log \eta + \frac{d(x, y)^{2}}{2\eta} - \log (1+\delta(x, \eta)) \\
                &\ge - \log C_{5} + \frac{d}{2}\log \eta + \frac{d(x, y)^{2}}{2s\eta} - \log (1+\frac{1}{C_{5} } ).
        \end{align*} 

        For all $\frac{1}{C_{\eta}}\le \frac{d(x, y)^{2}}{2} \le \frac{2-\frac{1}{s}}{C_{\eta}}$, we have 
        \begin{align*}
            \delta(x, \eta) = \frac{e^{- \frac{1}{\eta C_{\eta}}}}{\frac{C_{5}}{\eta^{\frac{d}{2}}} \exp (-\frac{d(x, y)^{2}}{2\eta})}
            \le \eta^{\frac{d}{2}}\frac{ \exp (\frac{1- \frac{1}{s}}{\eta C_{\eta}})}{C_{5} }
            \le \frac{1}{C_{5}\varepsilon^{1-\frac{1}{s}} },
        \end{align*}
        so that when $\varepsilon$ is small, for some $C_{6}$ we have
        \begin{align*}
                &\log (1 + \delta(x, \eta)) 
            \le C_{6} + \log \delta(x, \eta)
            \le C_{6} + \log \frac{1}{C_{5}} + (1-\frac{1}{s})\log \frac{1}{\varepsilon} \\
            =& C_{6} + \log \frac{1}{C_{5}} + \frac{1-\frac{1}{s}}{\eta C_{\eta}}
            \le C_{6} + \log \frac{1}{C_{5}} + \frac{1-\frac{1}{s}}{\eta}\frac{d(x, y)^{2}}{2},
            \forall \frac{1}{C_{\eta}}\le \frac{d(x, y)^{2}}{2} \le \frac{2-\frac{1}{s}}{C_{\eta}},
        \end{align*}
        which further implies
        \begin{align*}
                - \log \nu(\eta, x, y) &\ge - C_{6} - \log C_{5} + \frac{d}{2}\log \eta + \frac{d(x, y)^{2}}{2\eta} - \log (1+\delta(x, \eta)) \\
                &\ge - C_{6} - \log C_{5} + \frac{d}{2}\log \eta + \frac{d(x, y)^{2}}{2\eta} - \log \frac{1}{C_{5}} - \frac{1-\frac{1}{s}}{\eta}\frac{d(x, y)^{2}}{2} \\
                &\ge - C_{6} - \log C_{5} + \frac{d}{2}\log \eta + \frac{d(x, y)^{2}}{2(s\eta)} - \log \frac{1}{C_{5}}.
        \end{align*} 

    \item \textbf{Step 2:} Recall that 
        \begin{align*}
            \nu(\eta, x, y) \ge \frac{C_{1}}{\eta^{\frac{d}{2}}} \exp(-\frac{d(x, y)^{2}}{2\eta}).
        \end{align*}
        When $x = y$, we have $\nu(\eta, x, x) \ge \frac{C_{1}}{\eta^{\frac{d}{2}}}$
        and consequently 
        \begin{align*}
            \nu_{l}(\eta, x, x) \ge \frac{C_{1}}{\eta^{\frac{d}{2}}} - \varepsilon 
            = \frac{C_{1}}{\eta^{\frac{d}{2}}} - e^{- \frac{1}{\eta C_{\eta}}}
            \approx \frac{C_{1}}{\eta^{\frac{d}{2}}}.
        \end{align*}

        Thus for all $y \in B_{x}(r)$, for some constant $C$ we have 
        \begin{align*}
                - \log \nu_{l}(\eta, x, y) + \log \nu_{l}(\eta, x, x) 
                \ge \frac{d(x, y)^{2}}{2(s\eta)} + C.
        \end{align*}
    
        Therefore there exists some $C_{7}$ s.t. 
        \begin{align*}
            V_{\mathsf{MBI}}(y) := \frac{\exp(\log \nu_{l}(\eta, x, y) - \log \nu_{l}(\eta, x, x) + C_{7})}{\exp(-\frac{d(x, y)^{2}}{2(s\eta)})} \le 1, \forall y \in B_{x}(r).
        \end{align*}
        \item \textbf{Step 3:  Analyze the error of rejection sampling when the acceptance rate could possibly exceed $1$. }
        
        Recall that $\mu$ denote the density for Riemannian Gaussian distribution. 
        We compute 
        \begin{align*}
                \mu(s\eta, x, y) \frac{V_{\mathsf{MBI}}(y)}{\mathbb{E}_{\mu(s\eta, x, y)} V_{\mathsf{MBI}}(y)} 
                =& \frac{V_{\mathsf{MBI}}(y)}{\int_{M} V_{\mathsf{MBI}}(y) \mu(s\eta, x, y) dV_{g}(y)} \mu(s\eta, x, y)\\
                = & \frac{V_{\mathsf{MBI}}(y)\mu(s\eta, x, y) }{\int_{M}  \frac{\exp(\log \nu_{l}(\eta, x, y) - \log \nu_{l}(\eta, x, x) + C_{7})}{\exp(-\frac{d(x, y)^{2}}{2(s\eta)})} \mu(s\eta, x, y) dV_{g}(y)} \\
                = &\frac{\exp(\log \nu_{l}(\eta, x, y) - \log \nu_{l}(\eta, x, x) + C_{7}) }{\int_{M} \exp(\log \nu_{l}(\eta, x, y) - \log \nu_{l}(\eta, x, x) + C_{7}) dV_{g}(y)} \\
                = & \frac{\nu_{l}(\eta, x, y)}{\int_{M} \nu_{l}(\eta, x, y) dV_{g}(y)}
                =: \hat{\pi}^{Y|X}.
        \end{align*}
        Thus the desired rejection sampling output can be written as $$\hat{\pi}^{Y|X} = \mu(s\eta, x, y) \frac{V_{\mathsf{MBI}}(y)}{\mathbb{E}_{\mu(s\eta, x, y)} V_{\mathsf{MBI}}(y)}.$$ 
        
        On the other hand we denote $\overline{V_{\mathsf{MBI}}}(y) = \min\{1, V_{\mathsf{MBI}}(y)\}$,
        and the actual rejection sampling output is $\overline{\pi}^{Y|X} = \mu(s\eta, x, y) \frac{\overline{V_{\mathsf{MBI}}}(y)}{\mathbb{E}_{\mu(s\eta, x, y)} \overline{V_{\mathsf{MBI}}}(y)} $.
        Following \citet[Proof of Theorem 6]{fan2023improved}, we get 
        \begin{align*}
                \|\hat{\pi}^{Y|X} - \overline{\pi}^{Y|X}\|_{TV} 
                &\le \mathbb{E}_{\exp(-\frac{d(x, y)^{2}}{2(s\eta)})}[|\frac{V_{\mathsf{MBI}}}{\mathbb{E}[V_{\mathsf{MBI}}]} - \frac{\overline{V_{\mathsf{MBI}}}}{\mathbb{E}[\overline{V_{\mathsf{MBI}}}]}|] \\
                &\le \frac{2\mathbb{E}[|V_{\mathsf{MBI}} - \overline{V_{\mathsf{MBI}}}|]}{|\mathbb{E}[\overline{V_{\mathsf{MBI}}}]|}.
        \end{align*}
        We show that $|\mathbb{E}[\overline{V_{\mathsf{MBI}}}]|$ is lower bounded by some positive constant that does not depend on $\eta, \varepsilon$.
        \begin{align*}
                \mathbb{E}[|\overline{V_{\mathsf{MBI}}}|]
                &\ge 
                \int_{B_{x}(r)}\frac{\exp(\log \nu_{l}(\eta, x, y) - \log \nu_{l}(\eta, x, x) + C_{7})}{\exp(-\frac{d(x, y)^{2}}{2(s\eta)})}
                \frac{\exp(-\frac{d(x, y)^{2}}{2(s\eta)})}{\int_{M}\exp(-\frac{d(x, y)^{2}}{2(s\eta)})dV_{g}(y)} dV_{g}(y) \\
                &\ge \frac{e^{C_{7}}}{\int_{M}\exp(-\frac{d(x, y)^{2}}{2(s\eta)})dV_{g}(y)}
                \int_{B_{x}(r)}\frac{\nu_{l}(\eta, x, y)}{\nu_{l}(\eta, x, x)} dV_{g}(y) \\
                &= \tilde{\Omega}(\int_{B_{x}(r)}\nu(\eta, x, y) dV_{g}(y) - \zeta)\\
                &= \tilde{\Omega}(\frac{1}{\eta^{\frac{d}{2}}} (\frac{\sin r}{r})^{d-1} (2\pi\eta)^{\frac{d}{2}}(1 - \exp(-\frac{1}{2}( \frac{1}{\eta} \frac{2-\frac{1}{s}}{C_{\eta}} - d)))- \zeta)\\
                &= \tilde{\Omega}(1-\zeta) = \tilde{\Omega}(\zeta).
        \end{align*}

        \item \textbf{Step 4: Show that the error is of order $\varepsilon$.}
        
We now have
        \begin{align*}
                \mathbb{E}[|V_{\mathsf{MBI}} - \overline{V_{\mathsf{MBI}}}|] = &\mathbb{E}_{\mu(s\eta, x, y)}[\frac{\exp(\log \nu_{l}(\eta, x, y) - \log \nu_{l}(\eta, x, x) + C_{7})}{\exp(-\frac{d(x, y)^{2}}{2(s\eta)})} 1_{V_{\mathsf{MBI}}(y) > 1}] \\
                = & \frac{1}{\int_{M} \exp(-\frac{d(x, y)^{2}}{2(s\eta)}) dV_{g}(x)} e^{C_{7}}\int_{\{V_{\mathsf{MBI}}(y) > 1\}} \frac{\nu_{l}(\eta, x, y)}{\nu_{l}(\eta, x, x)} dV_{g}(x)\\
                \le & \frac{C_{4}}{\eta^{\frac{d}{2}}} e^{C_{7}} 
                \frac{\eta^{d/2}}{C_{1}}
                \int_{\{V_{\mathsf{MBI}}(y) > 1\}} \frac{\eta^{\frac{1}{2}} C_{2}}{\eta^{d}} \exp(-\frac{d(x, y)^{2}}{2\eta}) + \zeta dV_{g}(x)\\
                \le & e^{C_{7}} \frac{C_{4}}{C_{1}}
                \int_{\{d(x, y)^{2}/2 > \frac{2-\frac{1}{s}}{C_{\eta}}\}} \frac{\eta^{\frac{1}{2}} C_{2}}{\eta^{d}} \exp(-\frac{d(x, y)^{2}}{2\eta}) + \zeta dV_{g}(x)\\
                \le & \tilde{\mathcal{O}}(\zeta)+ \Poly(\frac{1}{\eta}) \exp(-\frac{2-\frac{1}{s}}{\eta C_{\eta}})
                = \tilde{\mathcal{O}}(\zeta) + \Poly(\frac{1}{\eta}) \varepsilon^{2-\frac{1}{s}}
                = \tilde{\mathcal{O}}(\zeta)
        \end{align*}
        in the last equality note that $\varepsilon^{1-\frac{1}{s}} \Poly(\frac{1}{\eta})$ is of constant order.
    \end{enumerate}
\end{proof}

\begin{proposition}\label{Prop_RHK_inexact_rej}
    Let $M$ be hypersphere $\mathcal{S}^{d}$.
    Consider Algorithm  with $\frac{1}{\eta }= L_{1}^{2} d \log \frac{1}{\varepsilon} $
    and $\frac{1}{t} = L_{1}^{2}(d-2)  \log \frac{1}{\varepsilon}$.
    There exists $K$ s.t. 
    for small $\varepsilon$, the error for inexact rejection sampling with $\nu_{l}$ is of order $\zeta$, i.e., 
    $\|\hat{\pi}^{X|Y} - \overline{\pi}^{X|Y}\|_{TV} = \tilde{\mathcal{O}}(\zeta)$.
\end{proposition}
\begin{proof}[Proof of Proposition~\ref{Prop_RHK_inexact_rej}]
    The proof is similar to the proof of Proposition~\ref{Prop_BM_inexact_rej}. For simplicity, we provide a sketch.
    \begin{enumerate}
        \item \textbf{Step 1:} We follow  exactly the same proof as in Proposition \ref{Prop_BM_inexact_rej}, with parameters chosen as 
        $s = \frac{d}{d-1}$, $\frac{1}{\eta }= L_{1}^{2} d \log \frac{1}{\varepsilon} $, $C_{\eta} = L_{1}^{2} d$, 
        $r^{2}/2 = \frac{2-\frac{d-1}{d}}{L_{1}^{2} d}$ and 
        $\frac{1}{t} = L_{1}^{2}(d-2) \log \frac{1}{\varepsilon}$. Note that $t =\frac{d}{d-2} \eta $.

        We know, for all $x \in B_{r}(y)$, we have (for some constant $C$)
        \begin{align*}
                - \log \nu_{l}(\eta, x, y) + \log \nu_{l}(\eta, y, y) 
                \ge \frac{d(x, y)^{2}}{2(s\eta)} - C.
        \end{align*}
        We want to find $t$ being the variable for proposal distribution so that $f(x) - \log \nu_{l}(\eta, x, y) - f(y) + \log \nu_{l}(\eta, y, y) \ge \frac{1}{2t} d(x, y)^{2}$ holds for all $x \in B_{r}(y)$, 
        hence we require
        \begin{align*}
                &\frac{d(x, y)^{2}}{2(\frac{d}{d-1}\eta)} - \frac{d(x, y)^{2}}{2t} - L_{1}d(x, y) - C \\
                =& \frac{d(x, y)^{2}}{2}(L_{1}^{2}\log \frac{1}{\varepsilon})- L_{1}d(x, y) - C
                \ge -\frac{1}{2 \log \frac{1}{\varepsilon}} - C,
        \end{align*}
        where in the last inequality we take $d(x, y) = \frac{1}{L_{1}\log \frac{1}{\varepsilon}}$. 
        Also note that when $\varepsilon$ is small, $|-\frac{1}{2 \log \frac{1}{\varepsilon}}|$ is small.
        Hence there exists constant $C$ s.t. for all $x \in B_{y}(r)$, 
        $f(x) - \log \nu_{l}(\eta, x, y) - f(y) + \log \nu_{l}(\eta, y, y) + C \ge \frac{1}{2t} d(x, y)^{2}$.
        
        Denote $$V_{\mathsf{RHK}}(x) = \frac{\exp(-f(x) + \log \nu_{l}(\eta, x, y) + f(y) - \log \nu_{l}(\eta, y, y) + C)}{\exp(-\frac{1}{2t} d(x, y)^{2})}$$ 
        and $\overline{V}_{\mathsf{RHK}}(x) = \min\{1, V_{\mathsf{RHK}}(x)\}$. 
        Recall that the desired rejection sampling output can be written as $\hat{\pi}^{X|Y} = \mu(t, x, y) \frac{V_{\mathsf{RHK}}(x)}{\mathbb{E}_{\mu(t, x, y)} V_{\mathsf{RHK}}(x)}$. 
        On the other hand the actual rejection sampling output is $\overline{\pi}^{X|Y} = \mu(t, x, y) \frac{\overline{V_{\mathsf{RHK}}}(x)}{\mathbb{E}_{\mu(t, x, y)} \overline{V_{\mathsf{RHK}}}(x)} $.
        Following \citet[Proof of Theorem 6]{fan2023improved}, we get 
        \begin{align*}
                \|\hat{\pi}^{X|Y} - \overline{\pi}^{X|Y}\|_{TV} 
                &\le \mathbb{E}_{\exp(-\frac{d(x, y)^{2}}{2t})}[|\frac{V_{\mathsf{RHK}}}{\mathbb{E}[V_{\mathsf{RHK}}]} - \frac{\overline{V_{\mathsf{RHK}}}}{\mathbb{E}[\overline{V_{\mathsf{RHK}}}]}|] \\
                &\le \frac{2\mathbb{E}[|V_{\mathsf{RHK}} - \overline{V_{\mathsf{RHK}}}|]}{|\mathbb{E}[\overline{V_{\mathsf{RHK}}}]|}.
        \end{align*}

        \item \textbf{Step 2: Verify that $\mathbb{E}[|\overline{V}|]$ is lower bounded by a constant.}

        We start with the following bound.
        \begin{align*}
                &-f(x) + f(y) + \log \nu(\eta, x, y) + C + \frac{1}{2}d(x, y)^{2}(L_{1}^{2}(d+1)\log \frac{1}{\varepsilon})\\
                \ge& -L_{1}d(x, y) + \log C_{1} - \frac{d}{2}\log \eta -\frac{d(x, y)^{2}}{2\eta} + \frac{1}{2}d(x, y)^{2}(L_{1}^{2}(d+1)\log \frac{1}{\varepsilon}) \\
                \ge& -L_{1}d(x, y) + \log C_{1} - \frac{d}{2}\log \eta + \frac{d(x, y)^{2}}{2}L_{1}^{2} \log \frac{1}{\varepsilon} \\
                \ge& -\frac{1}{2\log \frac{1}{\varepsilon}} + \log C_{1} - \frac{d}{2}\log \eta.
        \end{align*}
        Hence, we have
        \begin{align*}
                \mathbb{E}[|\overline{V}|]
                &\ge 
                \int_{B_{x}(r)}\frac{\exp(-f(x)+f(y)+\log \nu_{l}(\eta, x, y) - \log \nu_{l}(\eta, y, y) + C)}{\exp(-\frac{d(x, y)^{2}}{2t})}\\
                &\quad \frac{\exp(-\frac{d(x, y)^{2}}{2t})}{\int_{M}\exp(-\frac{d(x, y)^{2}}{2t})dV_{g}(y)} dV_{g}(y) \\
                &\ge \frac{e^{C}}{\int_{M}\exp(-\frac{d(x, y)^{2}}{2t})dV_{g}(y)}
                \int_{B_{x}(r)}\exp(-f(x)+f(y))\frac{\nu_{l}(\eta, x, y)}{\nu_{l}(\eta, y, y)} dV_{g}(y) \\
                &= \tilde{\Omega} (
                \int_{B_{x}(r)}\exp(-f(x)+f(y))\nu_{l}(\eta, x, y) dV_{g}(y)) \\
                &= \tilde{\Omega} (\int_{B_{x}(r)}\exp(-f(x)+f(y))\nu(\eta, x, y) dV_{g}(y) - \zeta)\\
                &= \tilde{\Omega} (\frac{1}{\eta^{\frac{d}{2}}}\int_{B_{x}(r)}\exp(-\frac{1}{2}d(x, y)^{2}(L_{1}^{2}(d+1)\log \frac{1}{\varepsilon})) dV_{g}(y) - \zeta)\\
                &= \tilde{\Omega} (\frac{1}{\eta^{\frac{d}{2}}}\int_{B_{x}(r)}\exp(-\frac{d(x, y)^{2}}{2\eta\frac{d}{d+1}}) dV_{g}(y) - \zeta)\\
                &= \tilde{\Omega} (1 - \exp(-\frac{1}{2}( \frac{d+1}{\eta d} r^{2}/2 - d)) - \zeta)
                = \tilde{\Omega} (1 - \exp(-\frac{1}{2}( \frac{d+1}{\eta d} \frac{2-\frac{d-1}{d}}{L_{1}^{2} d} - d)) - \zeta)\\
                &= \tilde{\Omega} (1).
        \end{align*}
    
        \item \textbf{Step 3: Verify that $\mathbb{E}[|\overline{V} - V|]$ is of order $\zeta$.} 
        
        We need a sharper bound for distant points. 
        With $\frac{1}{T} = L_{1}^{2}(d-0.5)\log\frac{1}{\varepsilon}$, we have
        \begin{align*}
                &-f(x) + f(y) + \log \nu(\eta, x, y) + C + \frac{1}{2T}d(x, y)^{2} \\
                \le& L_{1}d(x, y) + \log C_{2} - d\log \eta -\frac{d(x, y)^{2}}{2\eta} + \frac{1}{2}d(x, y)^{2}(L_{1}^{2}(d-0.5)\log \frac{1}{\varepsilon}) \\
                \le& L_{1}d(x, y) + \log C_{2} - d\log \eta - \frac{d(x, y)^{2}}{2} \frac{1}{2}L_{1}^{2} \log \frac{1}{\varepsilon} \\
                \le& \frac{1}{\log \frac{1}{\varepsilon}} + \log C_{2} - d\log \eta,
        \end{align*}
        where in the last inequality we set $d(x, y) = \frac{2}{L_{1} \log \frac{1}{\varepsilon}}$
        \begin{align*}
                \mathbb{E}[|V - \overline{V}|] = &\mathbb{E}_{\mu(t, x, y)}[\frac{\exp(-f(x) + f(y) + \log \nu_{l}(\eta, x, y) - \log \nu_{l}(\eta, y, y) + C)}{\exp(-\frac{d(x, y)^{2}}{2t})} 1_{V(y) > 1}] \\
                = & \tilde{\mathcal{O}}( \frac{1}{\int_{M} \exp(-\frac{d(x, y)^{2}}{2t}) dV_{g}(x)}\int_{\{V(y) > 1\}} 
                \exp(-f(x) + f(y))\frac{\nu_{l}(\eta, x, y)}{\nu_{l}(\eta, y, y)} dV_{g}(x))\\
                \le & \tilde{\mathcal{O}}( \frac{1}{(\frac{d}{d-2}\eta)^{\frac{d}{2}}} 
                \eta^{d/2}
                \int_{\{V(y) > 1\}} \exp(-f(x) + f(y) + \log \nu(\eta, x, y)) dV_{g}(x) + \zeta )\\
                \le & \tilde{\mathcal{O}}( \zeta  + \frac{1}{\eta^{d}}
                \int_{\{V(y) > 1\}} \exp(-\frac{d(x, y)^{2}}{2T}) dV_{g}(x) )\\
                \le & \tilde{\mathcal{O}}( \zeta + 
                \frac{1}{\eta^{d}} \int_{\{d(x, y)^{2}/2 > \frac{2-\frac{d-1}{d}}{L_{1}^{2} d} \}} \exp(-\frac{d(x, y)^{2}}{2}(L_{1}^{2}(d-0.5)\log\frac{1}{\varepsilon})) dV_{g}(x)) \\
                \le & \tilde{\mathcal{O}}(\zeta + \frac{1}{\eta^{d}} \exp(-(\frac{d+1}{d})(\frac{d-0.5}{d}\log\frac{1}{\varepsilon}))) \\
                \le & \tilde{\mathcal{O}}(\zeta  +  \frac{1}{\eta^{d}}\varepsilon^{\frac{d^{2}+0.5d-0.5}{d^{2}}})
                = \tilde{\mathcal{O}}(\zeta).
        \end{align*}
        Here we used the fact that $\frac{d^{2}+0.5d-0.5}{d^{2}} > 1$.
    \end{enumerate}
    
\end{proof}


\subsubsection{Heat kernel truncation: hypersphere}\label{Subsection_Sphere_L_infty}

In this subsection, we show that on hyperspheres $\mathcal{S}^{d}$, 
the truncation error bound $\|\nu - \nu_{L}\|_{\infty} = \tilde{\mathcal{O}}(\zeta)$ 
can be achieved with truncation level $L = \tilde{\mathcal{O}}(\Poly (\log \frac{1}{\varepsilon}))$. As proved in \cite{zhao2018exact}, the heat kernel on $\mathcal{S}^{d}$ 
can be written as the following uniformly convergent series 
(with $\varphi := \langle x, y \rangle_{\mathbb{R}^{d+1}}$)
\begin{align*}
    \nu(\eta, x, y) = \sum_{k = 0}^{\infty} \exp(-\frac{k(k+d-1)t}{2})\frac{2k+d - 1}{(d-1) A_{\mathcal{S}^{d}}} C_{k}^{(d-1)/2}(\cos (\varphi)),
\end{align*}
where $C_{l}^{\alpha}$ are the Gegenbauer polynomials. 
Define \begin{align*}
    M_{l} = \frac{\Gamma(\frac{l+d-1}{2})}{\Gamma(\frac{d-1}{2})\Gamma(\frac{l}{2}+1)}
    + \left|\frac{\Gamma(l+d-1)}{\Gamma(d-1)\Gamma(l+1)} - \frac{\Gamma(\frac{l+d-1}{2})}{\Gamma(\frac{d-1}{2})\Gamma(\frac{l}{2}+1)}\right|.
\end{align*}
Such $M_{l}$ is constructed to be an upper bound for Gegenbauer polynomials; see \citet[Proof of Theorem 1]{zhao2018exact}.  The following proposition is directly implied by \citet[Theorem 1]{zhao2018exact}, 
and we provide a proof for completeness. 

\begin{proposition}\label{Zeta_dependency_hypersphere_truncation}
    Let $M = \mathcal{S}^{d}$ be a hypersphere. 
    For truncation level $L = \tilde{\mathcal{O}}(\Poly (\log \frac{1}{\varepsilon}))$, 
    we can achieve $|\nu(\eta, x, y) - \nu_{L}(\eta, x, y)| = \tilde{\mathcal{O}}(\zeta), \forall x, y \in \mathcal{S}^{d}$.
\end{proposition}

\begin{proof}[Proof of Proposition~\ref{Zeta_dependency_hypersphere_truncation}]
    Throughout the proof, we denote $\varphi = \langle x, y \rangle_{\mathbb{R}^{d+1}}$.
    The parameters $M_{l}$ satisfies $|C_{l}^{\frac{d-2}{2}} (x)| \le M_{l}$ according to \citet[Proof of Theorem 1]{zhao2018exact}. 
    Hence, we have
    \begin{align*}
           & |\nu(\eta, x, y) - \nu_{L}(\eta, x, y)| \\
        =& |\sum_{l = L+1}^{\infty} \exp(-\frac{l(l+d-1)\eta}{2})\frac{2l+d - 1}{(d-1) A_{\mathcal{S}^{d}}} C_{l}^{(d-1)/2}(\cos (\varphi))| \\
        \le &\sum_{l = L+1}^{\infty} 
        \exp(-\frac{l(l+d-1)\eta}{2}) \frac{(2l+d - 1) M_{l}}{(d-1) A_{\mathcal{S}^{d}}}.
    \end{align*}

    Observe that for all $l \ge L+1$, since for large $L$ (that depends on dimension) 
    we have $ \frac{M_{l+1}}{M_{l}} = \mathcal{O}(1)$; see, also, \citet[Proof of Theorem 1]{zhao2018exact}. Hence,
    \begin{align*}
            Q_{l} := &\frac{\exp(-\frac{(l+1)(l+1+d-1)\eta}{2}) (2l+d + 1) M_{l+1}}{\exp(-\frac{l(l+d-1)\eta}{2}) (2l+d - 1) M_{l}} 
        = \frac{2l+d+1}{2l+d-1} \exp(-\frac{(d + 2l)\eta}{2}) \frac{M_{l+1}}{M_{l}} \\
        =& \mathcal{O}(\exp(-\frac{(2l + d)\eta}{2}) \frac{M_{l+1}}{M_{l}} ) \\
        =& \mathcal{O}(\exp(-\frac{(2L + 2 + d)\eta}{2}) )
        = \mathcal{O}(\exp(-L\eta)) = \mathcal{O}(\zeta).
    \end{align*}
    For the last line, note that with $L = \Poly (\log \frac{1}{\zeta})$ and $\eta = \frac{1}{C\log \frac{1}{\varepsilon}}$,
    we have that $L\eta = \Poly (\log \frac{1}{\zeta})$. 
    This implies $\exp(-L\eta) = \mathcal{O}(\exp (\log \zeta)) = \mathcal{O}(\zeta)$. 
    Now we compute the truncation error.
    \begin{align*}
            &|\nu(\eta, x, y) - \nu_{L}(\eta, x, y)|\\ 
            =& |\sum_{l = L+1}^{\infty} \exp(-\frac{l(l+d-1)\eta}{2})\frac{2l+d - 1}{(d-1) A_{\mathcal{S}^{d}}} C_{l}^{(d-1)/2}(\cos (\varphi))| \\
         \le & \frac{1}{(d-1) A_{\mathcal{S}^{d}}} \sum_{l = L+1}^{\infty} 
        \exp(-\frac{l(l+d-1)\eta}{2}) (2l+d - 1) M_{l} \\
         \le & \frac{1}{(d-1) A_{\mathcal{S}^{d}}} 
        \exp(-\frac{(L+1)(L+1+d-1)\eta}{2}) (2L+d + 1) \frac{(L+d-1)^{d-2}}{(d-2)!} \frac{1}{1-Q} \\
        = & \tilde{\mathcal{O}}(\exp(-\frac{(L+1)(L+d)\eta}{2}) (2L+d + 1) \frac{(L+d-1)^{d-2}}{(d-2)!}) \\
        =& \tilde{\mathcal{O}}(\exp(-L^2 \eta)\Poly (L) )
        = \tilde{\mathcal{O}}(\exp(-\Poly (\log \frac{1}{\zeta}))\Poly (\log \frac{1}{\zeta}) )
        = \tilde{\mathcal{O}}(\zeta).
    \end{align*}

\end{proof}

\section{Proofs for Entropy-regularized JKO scheme}\label{Proof_Theorem_Gaussian_JKO}

\begin{proof}[Proof of Lemma~\ref{Lemma_proximal_calculation}] Note that, we have
    \begin{align*}
        H_{\pi^{X}}(\rho) &= \int_{M} \rho(x) \log \frac{\rho(x)}{\pi^{X}} dV_{g}(x)
        = \int_{M} \rho(x) \log \frac{\rho(x)}{\exp(-f(x) - \frac{d(x, y)^{2}}{2\eta})\exp(\frac{d(x, y)^{2} }{2\eta}) C} dV_{g}(x) \\
        &= \int_{M} \rho(x) (\log \frac{\rho(x)}{C\exp(-f(x) - \frac{d(x, y)^{2}}{2\eta})} + \log \frac{1}{\exp(\frac{d(x, y)^{2}}{2\eta})} ) dV_{g}(x) \\
        &= \int_{M} \rho(x) \log \frac{\rho(x)}{C'(y) \tilde{\pi}^{X|Y}(x|y)} dV_{g}(x)
        - \int_{M} \rho(x) \frac{1}{2\eta} d(x, y)^{2} dV_{g}(x) \\
        &= H_{\tilde{\pi}^{X|Y = y}}(\rho) - \frac{1}{2\eta} \int_{M} d(x, y)^{2}d\rho + C(y),
\end{align*}
where $C= \frac{1}{\int_{M} e^{-f(x)} dV_{g}(x) }$,
$C'(y)$ and $C(y)$ are some constants that only depends on $y$. 
The above computation implies 
\begin{align*}
        \tilde{\pi}^{X|Y = y} &= \argmin_{\rho \in \mathcal{P}_{2}(M)} H_{\tilde{\pi}^{X|Y = y}}(\rho)
        = \argmin_{\rho \in \mathcal{P}_{2}(M)} H_{\tilde{\pi}^{X}}(\rho) + \frac{1}{2\eta} \int_{M} d(x, y)^{2}d\rho + C(y) \\
        &= \argmin_{\rho \in \mathcal{P}_{2}(M)} H_{\tilde{\pi}^{X}}(\rho) + \frac{1}{2\eta} W_{2}^{2}(\rho, \delta_{y}) = \prox_{\eta H_{\tilde{\pi}^{X}}} (\delta_{y}).
\end{align*}
\end{proof}

\begin{lemma}\label{lemtemp2}
    The minimization problem 
    \begin{align*}
        \min_{
            \gamma \in \mathcal{P}_{2}(M \times M),
            \gamma^{X} = \rho^{X}
        } \int_{M \times M}\frac{1}{2\eta} d(x, y)^{2} \gamma (x, y) dV_{g}(x) dV_{g}(y) 
        + H(\gamma),
    \end{align*}
    where the constraint means $\int_{M} \gamma(x, y) dV_{g}(y) = \rho^{X}(x)$, 
    has solution of the form 
    \begin{align*}
        \gamma(x, y) \propto \rho^{X}(x) \tilde{\pi}^{Y|X}(y|x).
    \end{align*}
\end{lemma}

\begin{proof}[Proof of Lemma~\ref{lemtemp2}]
    Since $\int_{M} \gamma(x, y) dV_{g}(y) = \rho^{X}(x)$, we have 
    \begin{align*}
        \int_{M} \Big( \int_{M} \gamma(x, y) dV_{g}(y) - \rho^{X}(x) \Big) \beta(x) dV_{g}(x) = 0, \forall \beta,
    \end{align*}
    we can construct the following Lagrangian
    \begin{align*}
        \int_{M \times M}\frac{1}{2\eta} d(x, y)^{2} \gamma (x, y) dV_{g}(x) dV_{g}(y)
        + H(\gamma) - \int_{M} \Big( \int_{M} \gamma(x, y) dV_{g}(y) - \rho^{X}(x) \Big) \beta(x) dV_{g}(x).
   \end{align*}
    Recall that $H(\gamma) = \int_{M \times M} \gamma \log(\gamma) dV_{g}(x) dV_{g}(y) $. We have,
    \begin{align*}
            \lim_{t \to 0} \frac{H(\gamma + t \varphi) - H(\gamma)}{t}
            &= \lim_{t \to 0} \frac{\int_{M \times M} (\gamma + t \varphi) \log(\gamma + t \varphi) - \gamma \log(\gamma) dV_{g}(x) dV_{g}(y) }{t} \\
            &= \lim_{t \to 0} \int_{M \times M} \varphi \log(\gamma + t \varphi) + \varphi dV_{g}(x) dV_{g}(y) \\
            &= \int_{M \times M} \varphi (\log(\gamma) + 1)  dV_{g}(x) dV_{g}(y).
    \end{align*}
For any function $f$, denote $I_{f}(\gamma) = \int_{M \times M} \gamma(x, y) f(x, y) dV_{g}(x) dV_{g}(y)$. We then have 
    \begin{align*}
        \lim_{t \to 0} \frac{I_{f}(\gamma + t \varphi) - I_{f}(\gamma)}{t}
        &= \lim_{t \to 0} \frac{\int_{M \times M} (\gamma + \varphi t)f - \gamma f dV_{g}(x) dV_{g}(y)}{t}\\
        &= \int_{M \times M} \varphi f dV_{g}(x) dV_{g}(y).
    \end{align*}
    Thus the variation of Lagrangian is given by
    \begin{align*}
        \int_{M \times M} \varphi \cdot \Big(\frac{1}{2\eta} d(x, y)^{2} + \log(\gamma) + 1 - \beta(x) \Big) dV_{g}(x) dV_{g}(y).
    \end{align*}
    We want the above to be zero for all $\varphi$. 
    Thus we need $\frac{1}{2\eta} d(x, y)^{2} + \log(\gamma) + 1 - \beta(x) = 0$ 
    which is equivalent to 
    \begin{align*}
        \gamma(x, y) = e^{\beta(x) - \frac{1}{2\eta} d(x, y)^{2} - 1}.
    \end{align*}
    This implies $\gamma(x, y) \propto e^{\beta(x) - \frac{1}{2\eta} d(x, y)^{2}}$
    Integrating with respect to the $y$ variable, we get 
    \begin{align*}
        \rho^{X}(x) = \int_{M} \gamma(x, y) dV_{g}(y) = \int_{M} e^{\beta(x) - \frac{1}{2\eta} d(x, y)^{2} - 1} dV_{g}(y) \propto e^{\beta(x) } \int_{M} e^{ - \frac{1}{2\eta} d(x, y)^{2}} dV_{g}(y).
    \end{align*}
    It then follows that 
    \begin{align*}
        \gamma(x, y) \propto \rho^{X}(x) \frac{e^{- \frac{1}{2\eta} d(x, y)^{2}}}{\int_{M} e^{ - \frac{1}{2\eta} d(x, y)^{2}} dV_{g}(y)}
        = \rho^{X}(x) \tilde{\pi}^{Y|X}(y|x).
    \end{align*}
\end{proof}

\begin{lemma}\label{lemtemp3}
    The minimization problem 
    \begin{align*}
        \min_{
            \gamma \in \mathcal{P}_{2}(M \times M),
            \gamma^{Y} = \rho^{Y}
        } \int_{M \times M} (f(x) + \frac{1}{2\eta} d(x, y)^{2}) \gamma (x, y) dV_{g}(x) dV_{g}(y) 
        + H(\gamma),
    \end{align*}
    where the constraint means $\int_{M} \gamma(x, y) dV_{g}(x) = \rho^{Y}(y)$, 
    has solution of the form 
    \begin{align*}
        \gamma(x, y) \propto \rho^{Y}(y) \tilde{\pi}^{X|Y}(x|y).
    \end{align*}
\end{lemma}

\begin{proof}[Proof of Lemma~\ref{lemtemp3}]
The proof follows similarly to that of Lemma~\ref{lemtemp2}. Since $\int_{M} \gamma(x, y) dV_{g}(x) = \rho^{Y}(y)$, we have 
    \begin{align*}
        \int_{M} \Big( \int_{M} \gamma(x, y) dV_{g}(x) - \rho^{Y}(y) \Big) \beta(y) dV_{g}(y) = 0, \forall \beta.
    \end{align*}
    We first constructing the following Lagrangian:
    \begin{align*}
        \int_{M \times M} (f(x) + \frac{1}{2\eta} d(x, y)^{2}) \gamma (x, y) dV_{g}(x) dV_{g}(y)
        + H(\gamma) - \int_{M} \Big( \int_{M} \gamma(x, y) dV_{g}(x) - \rho^{Y}(y) \Big) \beta(y) dV_{g}(y).
    \end{align*}
Recall that $H(\gamma) = \int_{M \times M} \gamma \log(\gamma) dV_{g}(x) dV_{g}(y) $. Then, we have
    \begin{align*}
            \lim_{t \to 0} \frac{H(\gamma + t \varphi) - H(\gamma)}{t}
            &= \lim_{t \to 0} \frac{\int_{M \times M} (\gamma + t \varphi) \log(\gamma + t \varphi) - \gamma \log(\gamma) dV_{g}(x) dV_{g}(y) }{t} \\
            &= \lim_{t \to 0} \int_{M \times M} \varphi \log(\gamma + t \varphi) + \varphi dV_{g}(x) dV_{g}(y) \\
            &= \int_{M \times M} \varphi (\log(\gamma) + 1)  dV_{g}(x) dV_{g}(y).
    \end{align*}
For any function $f$, denote $I_{f}(\gamma) = \int_{M \times M} \gamma(x, y) f(x, y) dV_{g}(x) dV_{g}(y)$. We have 
    \begin{align*}
        \lim_{t \to 0} \frac{I_{f}(\gamma + t \varphi) - I_{f}(\gamma)}{t}
        &= \lim_{t \to 0} \frac{\int_{M \times M} (\gamma + \varphi t)f - \gamma f dV_{g}(x) dV_{g}(y)}{t}\\
        &= \int_{M \times M} \varphi f dV_{g}(x) dV_{g}(y).
    \end{align*}
    Thus the variation of Lagrangian is 
    \begin{align*}
        \int_{M \times M} \varphi \cdot \Big(f(x) + \frac{1}{2\eta} d(x, y)^{2} + \log(\gamma) + 1 - \beta(y) \Big) dV_{g}(x) dV_{g}(y).
    \end{align*}
    We want the above to be zero for all $\varphi$. 
    Thus we need $f(x) + \frac{1}{2\eta} d(x, y)^{2} + \log(\gamma) + 1 - \beta(y) = 0$ 
    which is equivalent to 
    \begin{align*}
        \gamma(x, y) = e^{-f(x) + \beta(y) - \frac{1}{2\eta} d(x, y)^{2} - 1}.
    \end{align*}
    This implies $\gamma(x, y) \propto e^{\beta(y) - f(x) - \frac{1}{2\eta} d(x, y)^{2}}$. Hence we can integrate with respect to the $x$ variable and get 
    \begin{align*}
        \rho^{Y}(y) = \int_{M}e^{-f(x) + \beta(y) - \frac{1}{2\eta} d(x, y)^{2} - 1} dV_{g}(x) \propto e^{\beta(y)} \int_{M}e^{-f(x) - \frac{1}{2\eta} d(x, y)^{2}} dV_{g}(x).
    \end{align*}
    Therefore, we obtain
    \begin{align*}
        \gamma(x, y) 
        \propto \rho^{Y}(y) \frac{e^{- f(x) - \frac{1}{2\eta} d(x, y)^{2}}}{\int_{M}e^{-f(x) - \frac{1}{2\eta} d(x, y)^{2}} dV_{g}(x)} 
        = \rho^{Y}(y) \tilde{\pi}^{X|Y}(x|y).
    \end{align*}

\end{proof}

\begin{proof}[Proof of Theorem~\ref{Theorem_Gaussian_JKO}]
    By definition we have 
    \begin{align*}
            &\argmin_{\chi \in \mathcal{P}_{2}(M)} \frac{1}{2\eta} W_{2, 2\eta}^{2}(\tilde{\rho}_{k}^{X}, \chi) \\
        =& \argmin_{\chi \in \mathcal{P}_{2}(M): \gamma \in C(\tilde{\rho}_{k}^{X}, \chi)} 
        \int_{M \times M} \frac{1}{2\eta} d(v, w)^{2} \gamma(v, w) dV_{g}(v) dV_{g}(w) 
        +  H(\gamma). 
    \end{align*}
    By Lemma~\ref{lemtemp3}, we know the solution of 
    \begin{align*}
        \min_{
            \gamma \in \mathcal{P}_{2}(M \times M),
            \gamma^{X} = \tilde{\rho}_{k}^{X}
        } \int_{M \times M}\frac{1}{2\eta} d(x, y)^{2} \gamma (x, y) dV_{g}(x) dV_{g}(y) 
        + H(\gamma)
    \end{align*}
    is $\gamma(x, y) \propto \tilde{\rho}_{k}^{X}(x) e^{- \frac{1}{2\eta} d(x, y)^{2}}$.
    Hence the $Y$-marginal of inexact proximal sampler satisfies
    \begin{align*}
        \chi(y) = \int_{M} \tilde{\rho}_{k}^{X}(x) e^{- \frac{1}{2\eta}d(x, y)^{2}} dV_{g}(x)
        = \int_{M} \tilde{\rho}_{k}^{X}(x) \tilde{\pi}^{Y|X}(y|x) dV_{g}(x) 
        = \tilde{\rho}_{k}(y).
    \end{align*}
Similarly,
    \begin{align*}
            &\argmin_{\chi \in \mathcal{P}_{2}(M)} 
            \frac{1}{2\eta} W_{2, 2\eta}^{2}(\tilde{\rho}_{k}^{Y}, \chi) + \int f d\chi \\
        =& \argmin_{\chi \in \mathcal{P}_{2}(M): \gamma \in C(\tilde{\rho}_{k}^{X}, \chi)}
        \int_{M \times M} (f(x) + \frac{1}{2\eta} d(v, w)^{2}) \gamma(v, w) dV_{g}(v) dV_{g}(w) 
        +  H(\gamma), 
    \end{align*}
    and its solution is $\chi(x) = \int_{M} \tilde{\rho}^{Y}(y) \tilde{\pi}^{X|Y}(x|y) dV_{g}(y) = \tilde{\rho}_{k+1}^{X}(x)$.
\end{proof}

\section{Auxiliary Results}

\subsection{Diffusion Process on Manifold}

It is well known that the law of the following SDE $dX_{t} = - b(X_{t}) dt + dB_{t}$ is related to the Fokker-Planck equation $\partial_{t} \rho_{t} = \Div (\rho_{t} b(X_{t}) + \frac{1}{2} \grad \rho_{t}) $. Here we provide a proof for completeness.
\begin{lemma}\label{Lemma_SDE_F_P}
    Let $B_{t}$ denote Brownian motion on a Riemannian manifold $M$.
    For SDE $dX_{t} = - b(X_{t}) dt + dB_{t}$, the corresponding Fokker-Planck equation is 
    \begin{align*}
        \partial_{t} \rho_{t} = \Div (\rho_{t} b(X_{t}) + \frac{1}{2} \grad \rho_{t}).
    \end{align*}
\end{lemma}

\begin{proof}
    The infinitesimal generator of the SDE is $Lf = - \langle \grad f, b \rangle + \frac{1}{2} \Delta f$ \cite{cheng2022efficient}.
    We compute the adjoint of $L$ which is defined by $\int_{M} f L^{*} h dV_{g} = \int_{M} h L f dV_{g}$.
    By divergence theorem, we have 
    \begin{align*}
            \frac{1}{2} \int_{M} \Div (\grad f) h dV_{g}
            &= - \frac{1}{2} \int_{M} \langle \grad h, \grad f \rangle dV_{g}
            = \frac{1}{2} \int_{M} \Div (\grad h) f dV_{g} \\
            - \int_{M} \langle \grad f, b h \rangle dV_{g} &= \int_{M} \Div (bh) f dV_{g}.
    \end{align*}
    Hence 
    \begin{align*}
            \int_{M} h L f dV_{g}
            &= \int_{M} - h \langle \grad f, b \rangle + \frac{1}{2} h \Delta f dV_{g} \\
            &= \int_{M} f \Div (bh) + \frac{1}{2} f \Delta h dV_{g}
            = \int_{M} f (\Div (bh) + \frac{1}{2}  \Delta h) dV_{g}.
    \end{align*}
    Thus we obtained $L^{*}h = \Div (bh) + \frac{1}{2} \Delta h$.
    By Kolmogorov forward equation {\cite[Equation 1.5.2]{bakry2014analysis}}, we get 
    \begin{align*}
        \partial_{t} \rho_{t} 
        = L^{*} \rho_{t}
        = \Div (b(X_{t})\rho_{t}) + \frac{1}{2} \Delta \rho_{t}
        = \Div (b(X_{t})\rho_{t} + \frac{1}{2} \grad \rho_{t}).
    \end{align*}
\end{proof}

We briefly mention some properties of Markov semigroup.
The following results are from \citet[Section 1.2]{bakry2014analysis}.
\begin{definition}
    \begin{enumerate}
        \item Given a markov process, the assoicated markov semigroup $(P_{t})_{t \ge 0}$ is 
        defined as (for suitable $f$)
        \begin{align*}
            P_{t}f(x)= \mathbb{E}[f(X_{t}) | X_{0} = x], \forall t \ge 0.
        \end{align*}
        \item Let $\rho$ be the law of $X_{0}$, then $P_{t}^{*}\rho$ is the law of $X_{t}$. We have 
        \begin{align*}
            \int_{M} P_{t} f d\rho = \int_{M} f d(P_{t}^{*}\rho).
        \end{align*}
        \item Markov operators $(P_{t})_{t \ge 0}$ can be represented by 
        kernels corresponding to the transition probabilities of the associated 
        Markov process:
        \begin{align*}
            P_{t} f(x) = \int_{M} f(y) p_{t} (x, y) dV_{g}(y), \forall t \ge 0, x \in M.
        \end{align*}
    \end{enumerate}
\end{definition}

Thus by definition, we have
\begin{align*}
    \mathbb{E}[f(X_{t}) | X_{0} = x] = P_{t} f(x) = \int_{M} f(y) p_{t} (x, y) dV_{g}(y).
\end{align*}

\subsection{Log-Sobolev Inequality and Heat flow}\label{LSI_Heat}

In the sampling literature, the log-Sobolev inequality is usually written in the following form: 
\begin{align*}
        \int_{M} f^{2} \log f^{2} d\nu - \int_{M} f^{2} d\nu \log \int_{M} f^{2} d\nu
        &\le \frac{2}{\alpha} \int_{M} \|\grad f\|^{2} d\nu , \forall f  \\
        H_{\nu}(\rho) &\le \frac{1}{2\alpha} J_{\nu}(\rho), \forall \rho.
\end{align*}

In the Euclidean setting, we know if $\mu_{1}, \mu_{2}$ satisfy $\alpha_{1}, \alpha_{2}$-$\mathsf{LSI}$ respectively, 
then their convolution $\mu_{1} * \mu_{2}$ satisfies LSI with constant $ \frac{1}{\frac{1}{\alpha_{1}} + \frac{1}{\alpha_{2}}}$, {see \citet[Proposition 2.3.7]{chewi2023log}}.
In particular, if we take one of $\mu$ to be $\nu(t, x, y)$ (which is a Gaussian in the Euclidean setting), since the Gaussian density satisfies LSI, we have the following result. 
\begin{fact}
    Consider Euclidean space.
    Let $\mu$ be a probability measure that satisfies $\alpha$-$\mathsf{LSI}$. 
    Then its propogation along heat flow, denoted by $\mu_{t} = \mu * \nu_{t}$, 
    also satisfies LSI with constant $\frac{1}{\frac{1}{\alpha} + t} = \frac{\alpha}{1 + t\alpha}$. Here $\nu_{t}$ denote the probability measure corresponding to heat flow for time $t$.
\end{fact}
On a Riemannian manifold, the density for Brownian motion satisfies LSI. 
\begin{theorem}{\cite[Theorem 3.1]{hsu1997logarithmic}}
    Suppose $M$ is a complete, connected manifold with $\Ric_{M} \ge -c$. Here $c \ge 0$.
    Then for any smooth function on $M$, we have 
    \begin{align*}
        \int_{M} f^{2} \log |f| d\nu_{o, s} 
        \le \frac{e^{cs} - 1}{c} \|\grad f\|_{\nu_{o, s}}^{2} 
        + \|f\|_{\nu_{o, s}}^{2} \log \|f\|_{\nu_{o, s}}.
    \end{align*}
    With $\kappa = -c$, we know the Brownian motion density for time $t$ satisfies LSI with constant
    $\alpha = \frac{\kappa}{1 - e^{-\kappa t}}$.
\end{theorem}

As a special case $M = \mathbb{R}^{d}$, we have $c = 0$. Hence, the LSI constant became $\lim_{c \to 0} \frac{e^{ct} - 1}{c} = t$.
That is, (with $\nu$ representing the measure for Brownian motion with time $t$) 
$H_{\nu}(\rho) \le \frac{t}{2} I_{\nu}(\rho), \forall \rho$.
So the LSI constant for Brownian motion is $\alpha_{\nu} = \frac{1}{t}$.

In the following, we prove that on a Riemannian manifold, such a fact is still true. 
We follow the idea by {\citet[Theorem 4.1]{collet2008logarithmic}}.
For notations, we denote $\Gamma(f) = \Gamma(f, f) = \|\grad f\|_{g}^{2}$. We also require the following intermediate result. 
\begin{lemma}[Theorem 5.5.2 in \cite{bakry2014analysis}]    For Markov triple with semigroup $(P_{t})_{t \ge 0}$, 
    the followings are equivalent:
    \begin{enumerate}
        \item $\sqrt{\Gamma(P_{s}f)} \le e^{- \beta s} P_{s} \sqrt{\Gamma(f)}$.
        \item $ P_{s}(f \log f) -  P_{s} f \log (P_{s} f) \le c(s) P_{s} (\frac{\Gamma(f)}{f}) $ where $c(s) = \frac{1 - e^{-2\beta s}}{2\beta}$.
    \end{enumerate}
\end{lemma}
\begin{corollary}\label{Cor_LSI_Markov}
    With $P_{t}$ denote manifold Brownian motion, we have 
    \begin{enumerate}
        \item $\sqrt{\Gamma(P_{t}f)} \le e^{- \frac{\kappa}{2}t} P_{t} \sqrt{\Gamma(f)}$.
        \item $ P_{t}(f \log f) -  P_{t} f \log (P_{t} f) \le c(t) P_{t} (\frac{\Gamma(f)}{f}) $ 
        where $c(t) = \frac{1 - e^{-\kappa t}}{2\kappa}$.
    \end{enumerate}
\end{corollary}
\begin{proof}[Proof of Corollary~\ref{Cor_LSI_Markov}]
    For the second item, we can replace $f$ by $g^{2}$ for some $g$. 
    \begin{align*}
            \int_{M} g^{2} \log g^{2} d\nu_{s} - \int_{M} g^{2}d\nu_{s} \log(\int_{M} g^{2}d\nu_{s}) 
            &\le \frac{1 - e^{-2\beta s}}{2\beta} \int_{M}  \frac{(2g)^{2}\|\grad g\|^{2}}{g^{2}} d\nu_{s} \\
            &= 4\frac{1 - e^{-2\beta s}}{2\beta} \int_{M} \|\grad g\|^{2} d\nu_{s}.
    \end{align*}
    Now we already know the manifold Brownian motion density $\nu_{t}$ satisfies $\frac{\kappa}{1 - e^{-\kappa t}}$-$\mathsf{LSI}$, i.e., 
    \begin{align*}
        \int_{M} f^{2} \log f^{2} d\nu_{t} - \int_{M} f^{2} d\nu_{t} \log \int_{M} f^{2} d\nu_{t}
            \le 2\frac{1 - e^{-\kappa t}}{\kappa} \int_{M} \|\grad f\|^{2} d\nu_{t} , \forall f.
    \end{align*}
    So we know, with $P_{t}$ representing manifold Brownian motion, 
    $ P_{t}(f \log f) -  P_{t} f \log (P_{t} f) \le c(t) P_{t} (\frac{\Gamma(f)}{f})$, where 
    \begin{align*}
        2\frac{1 - e^{-\kappa t}}{\kappa} = 4\frac{1 - e^{-2\beta s}}{2\beta}.
    \end{align*}
    Hence we know $\beta$ can be taken as $\kappa$, $s$ corresponds to $\frac{1}{2}t$. 
    So we get $\sqrt{\Gamma(P_{t}f)} \le e^{- \frac{\kappa}{2}t} P_{t} \sqrt{\Gamma(f)}$.
\end{proof}


\begin{proposition}\label{LSI_Propagation_Heat}
    Let $M$ be a Riemannian manifold with Ricci curvature bounded below by $\kappa$.
    Let $\rho_{0}$ be any initial distribution.
    Assume $\rho_{0}$ satisfies LSI with constant $\frac{1}{d_{0}}$:
    \begin{align*}
        \int_{M} g^{2} \log g^{2} d\rho_{0} - \int_{M} g^{2} d\rho_{0} \log \int_{M} g^{2} d\rho_{0} 
        \le 2 d_{0} \int_{M} \|\grad g\|^{2} d\rho_{0}, \forall g.
    \end{align*}
    Then the propagation of $\rho_{0}$ along heat flow, denoted as $\rho_{t}$, 
    satisfies LSI with constant 
    \begin{align*}
        \frac{1}{2c(t) + d_{0}e^{-\kappa t}} 
        = \frac{\kappa}{1 - e^{-\kappa t} + \kappa d_{0}e^{-\kappa t}} 
        = \frac{\kappa e^{\kappa t}}{e^{\kappa t} - 1 + \kappa d_{0}},
    \end{align*}
    where $c(t) = \frac{1 - e^{-\kappa t}}{2\kappa} $.
    If $\kappa \ge 0$, we have $c(t) \le \frac{1}{2} t$.
    \begin{align*}
        \frac{1}{2c(t) + d_{0}e^{-\kappa t}} 
        \ge \frac{1}{t + d_{0}e^{-\kappa t}} \ge \frac{1}{t + d_{0}}.
    \end{align*}
\end{proposition}
\begin{proof}[Proof of Proposition~\ref{LSI_Propagation_Heat}]
    Since $\rho_{0}$ satisfies LSI with constant $\frac{1}{d_{0}}$, 
    equivalently with $f$ replace $g^{2}$, we get 
    \begin{align*}
        \int_{M} f \log f d\rho_{0} - \int_{M} f d\rho_{0} \log \int_{M} f d\rho_{0} 
        \le 2 d_{0} \int_{M} \|\grad \sqrt{f}\|^{2} d\rho_{0}
        = \frac{d_{0}}{2} \int_{M} \frac{\|\grad f \|^{2}}{f} d\rho_{0}.
    \end{align*}

    For $g > 0$, using Corollary \ref{Cor_LSI_Markov}, 
    we know the manifold Brownian motion (here represented by $P_{t}$) satisfies 
    \begin{align*}
         P_{t}(g \log g) -  P_{t} g \log (P_{t} g) 
            \le c(t) P_{t} (\frac{\Gamma(g)}{g}),
    \end{align*}
    where $c(t) = \frac{1 - e^{-\kappa t}}{2\kappa}$.
    Using property of markov semigroup, we have 
    \begin{align*}
        \int_{M} g \log g d\rho_{t} = \int_{M} P_{t}(g \log g) d\rho_{0}  
            \le c(t) \int_{M} P_{t} (\frac{\Gamma(g)}{g}) d\rho_{0}
            + \int_{M} P_{t} g \log (P_{t} g) d\rho_{0}.
    \end{align*}
    Hence
    \begin{align*}
            &\int_{M} g \log g d\rho_{t} - \int_{M} g d\rho_{t} \log \int_{M} g d\rho_{t} \\
            \le & c(t) \int_{M} P_{t} (\frac{\Gamma(g)}{g}) d\rho_{0}
            + \int_{M} P_{t} g \log (P_{t} g) d\rho_{0}
            - \int_{M} P_{t}(g) d\rho_{0} \log \int_{M} P_{t}(g) d\rho_{0} \\
            \le & c(t) \int_{M} P_{t} (\frac{\Gamma(g)}{g}) d\rho_{0}
            + \frac{d_{0}}{2}\int_{M} \frac{\Gamma(P_{t} g)}{P_{t} g} d\rho_{0} \\
            \le & c(t) \int_{M} P_{t} (\frac{\Gamma(g)}{g}) d\rho_{0}
            + \frac{d_{0}}{2}e^{-\kappa t}\int_{M} \frac{(P_{t}\sqrt{\Gamma(g)})^{2}}{P_{t} g} d\rho_{0} \\
            \le & c(t) \int_{M} P_{t} (\frac{\Gamma(g)}{g}) d\rho_{0}
            + \frac{d_{0}}{2}e^{-\kappa t}\int_{M} P_{t}(\frac{\Gamma(g)}{g}) d\rho_{0} \\
            = & \Big(c(t) + \frac{d_{0}}{2}e^{-\kappa t}\Big) \int_{M} P_{t}(\frac{\Gamma(g)}{g}) d\rho_{0}
            = \Big(c(t) + \frac{d_{0}}{2}e^{-\kappa t}\Big) \int_{M} \frac{\Gamma(g)}{g} d\rho_{t},
    \end{align*}
    where the third inequality is due to Corollary \ref{Cor_LSI_Markov}, 
    and in the last inequality we used Cauchy-Schwarz inequality: 
    \begin{align*}
            P_{t}(\frac{\Gamma(f)}{f}) P_{t}(f) 
            &= \mathbb{E}[\frac{\|\grad f\|^{2}}{f}] \mathbb{E}[f] \ge \mathbb{E}\Big[\sqrt{\frac{\|\grad f\|^{2}}{f} f} \Big]^{2} 
            = \mathbb{E}[\|\grad f\|]^{2} \\
            &= (P_{t}\sqrt{\Gamma(f)})^{2}.
    \end{align*}
    Hence we know $\rho_{t}$ satisfies LSI with constant 
    \begin{align*}
        \frac{1}{2c(t) + d_{0}e^{-\kappa t}} 
        = \frac{\kappa}{1 - e^{-\kappa t} + \kappa d_{0}e^{-\kappa t}}.
    \end{align*}
\end{proof}

Note that we have $\lim_{\kappa \to 0} \frac{1}{2c(t) + d_{0}e^{-\kappa t}} = \frac{1}{t + d_{0}}$. This means that, we can recover the result for Euclidean space in the limit.

\subsection{Total Variation Distance}

\begin{lemma}\label{Prop_TV}
    For TV distance, we have 
    \begin{align*}
            \|\rho^{(1)} - \rho^{(2)}\|_{TV} := \sup_{A \subseteq M} |\rho^{(1)}(A) - \rho^{(2)}(A)|
            = \frac{1}{2} \int_{M} |\frac{d\rho^{(1)}}{dV_{g}} - \frac{d\rho^{(2)}}{dV_{g}}| dV_{g},
    \end{align*}
    and 
    \begin{align*}
        \|\rho^{(1)} - \rho^{(2)}\|_{TV} = \frac{1}{2} \sup_{f: M \to [-1, 1]} |\int f d\rho^{(1)} - \int f d\rho^{(2)}|. 
    \end{align*}
\end{lemma}
\begin{proof}[Proof of Lemma~\ref{Prop_TV}]
    Denote the set at which supremum is achieved to be $A_{*} = \{x \in M: \rho^{(2)}(x) \ge \rho^{(1)}(x) \}$.
    Denote $\rho^{(2)}, \rho^{(1)}$ to be the measure, or corresponding probability density function 
    with respect to the Riemannian volumn form, when appropriate.
    \begin{align*}
            &\int_{M} |\rho^{(2)}(x) - \rho^{(1)}(x)| dV_{g}(x)
            = \int_{A} |\rho^{(2)}(x) - \rho^{(1)}(x)| dV_{g}(x) + \int_{A^{c}} |\rho^{(2)}(x) - \rho^{(1)}(x)| dV_{g}(x) \\
            =& \int_{A_{*}} |\rho^{(2)}(x) - \rho^{(1)}(x)| dV_{g}(x) + \int_{A_{*}^{c}} |\rho^{(2)}(x) - \rho^{(1)}(x)| dV_{g}(x) \\
            =& \int_{A_{*}} \rho^{(2)}(x) - \rho^{(1)}(x) dV_{g}(x) + \int_{A_{*}^{c}} \rho^{(1)}(x) - \rho^{(2)}(x) dV_{g}(x) \\
            =& \rho^{(2)}(A_{*}) - \rho^{(2)}(A_{*}^{c}) - \rho^{(1)}(A_{*}) + \rho^{(1)}(A_{*}^{c}) 
            = 2\rho^{(2)}(A_{*}) - 1 - 2\rho^{(1)}(A_{*}) + 1 \\
            =& 2(\rho^{(2)}(A_{*}) - \rho^{(1)}(A_{*})) 
            = 2\|\rho^{(2)} - \rho^{(1)}\|_{TV}.
    \end{align*}
    Now we prove the second equation. 
    \begin{align*}
            \Big| \int_{M} f d\rho^{(2)} - \int_{M} f d\rho^{(1)} \Big|
            &\le \int_{M} \Big| f(x) (\rho^{(2)}(x) - \rho^{(1)}(x)) \Big| dV_{g}(x) \\
            &\le \sup_{x \in M } |f(x)| \int_{M} | \rho^{(2)}(x) - \rho^{(1)}(x) | dV_{g}(x) \\
            &= \int_{M} | \rho^{(2)}(x) - \rho^{(1)}(x) | dV_{g}(x) = 2\|\rho^{(2)} - \rho^{(1)}\|_{TV}.
    \end{align*}
    When $f = 1_{A_{*}} - 1_{A_{*}^{c}}$, 
    \begin{align*}
            \Big| \int_{M} f d\rho^{(2)} - \int_{M} f d\rho^{(1)} \Big|
            &= \Big| \rho^{(2)}(A_{*}) - \rho^{(2)}(A_{*}^{c}) - \rho^{(1)}(A_{*}) + \rho^{(1)}(A_{*}^{c}) \Big| \\
            &= 2\rho^{(2)}(A_{*}) - 2 \rho^{(1)}(A_{*}) = 2\|\rho^{(2)} - \rho^{(1)}\|_{TV}.
    \end{align*}
\end{proof}

\begin{lemma}\label{Lemma_TV_heat}
    Let $\rho^{(1)}, \rho^{(2)}$ be probability measures. 
    Let $\rho^{(1)}_{t}, \rho^{(2)}_{t}$ denote propagation of $\rho^{(1)}, \rho^{(2)}$ along heat flow on $M$, with $\rho^{(1)}_{0} = \rho^{(1)}$, $\rho^{(2)}_{0} = \rho^{(2)}$.
    We have 
    \begin{align*}
        \|\rho^{(1)}_{t} - \rho^{(2)}_{t}\|_{TV} \le \|\rho^{(1)} - \rho^{(2)}\|_{TV}.
    \end{align*}
\end{lemma}
\begin{proof}[Proof of Lemma~\ref{Lemma_TV_heat}]
    By definition we have that for all $f$, 
    \begin{align*}
        \mathbb{E}[f(X_{t}) | X_{0} = x] = P_{t} f(x) = \int_{M} f(y) p_{t} (x, y) dV_{g}(y).
    \end{align*}
    Assuming $X_{0} \sim \rho^{(1)}$, we get 
    \begin{align*}
            \int_{M} f(x) \rho^{(1)}_{t}(x) dV_{g}(x) &= \mathbb{E}[f(X_{t})] 
            = \int P_{t} f(x) d\rho^{(1)}(x) \\
            &= \int_{M} \int_{M} f(y) p_{t} (x, y) dV_{g} (y)\rho^{(1)}(x) dV_{g}(x) \\
            &= \int_{M} g(x) \rho^{(1)}(x) dV_{g}(x).
    \end{align*}
    Where we denote $\int_{M} f(y) p_{t} (x, y) dV_{g}(y) = g(x)$.
    Note that 
    \begin{align*}
        |g(x)| \le \int_{M} |f(y)| p_{t} (x, y) dV_{g}(y) \le \int_{M} p_{t} (x, y) dV_{g}(y) = 1.
    \end{align*}
    Hence 
    \begin{align*}
            \|\rho^{(1)}_{t} - \rho^{(2)}_{t}\|_{TV} 
            =& \frac{1}{2} \sup_{f: M \to [-1, 1]} |\int f d\rho^{(1)}_{t} - \int f d\rho^{(2)}_{t} | \\
            =& \frac{1}{2} \sup_{f: M \to [-1, 1]} |\int_{M} \int_{M} f(y) p_{t} (x, y) dV_{g}(y) \rho^{(1)}(x) dV_{g}(x) \\ 
            & - \int_{M} \int_{M} f(y) p_{t} (x, y) dV_{g}(y) \rho^{(2)}(x) dV_{g}(x) | \\
            =& \frac{1}{2} \sup_{f: M \to [-1, 1]} |\int_{M} g(x) \rho^{(1)}(x) dV_{g}(x) - \int_{M} g(x) \rho^{(2)}(x) dV_{g}(x) | \\
            \le& \frac{1}{2} \sup_{g: M \to [-1, 1]} |\int_{M} g(x) \rho^{(1)}(x) dV_{g}(x) - \int_{M} g(x) \rho^{(2)}(x) dV_{g}(x) | \\
            =& \|\rho^{(1)} - \rho^{(2)}\|_{TV}.
    \end{align*}

\end{proof}

\subsection{Expected number of rejections on hypersphere}\label{Subsection_expected_rej}

We consider the approximation scheme introduced in Section \ref{Section_Approximation_JKO} 
using Varadhan's asymptotics. 
Let $\varphi(x) = \frac{1}{2\eta} d(x, y)^{2}$. 
Intuitively, we want to see how the function $\varphi$ can improve the convexity of $f + \varphi$.

On a manifold with positive curvature, we consider the situation that we cannot compute the minimizer of $g(x) = f(x) + \frac{1}{2\eta}d(x, y)^{2} $, and instead use $y$ as the approximation of it. 
Notice that when $\eta$ is small, since $f(x)$ is uniformly bounded, the function $g(x)$ is dominated by $\frac{1}{2\eta}d(x, y)^{2}$, thus the minimizer of $g$ will be close to $y$. Therefore it is reasonable to use $y$ as an approximation of the mode of $e^{-g(x)}$.
Then in rejection sampling, we use $\mu(t, y, x)$ as the proposal.

Let $L_{1}$ be the Lipschitz constant of $f$. 
In the next proposition, we show that for some constant $C_{\varepsilon}$, with certain choices of $\eta$ and $t$, it holds that 
\begin{align*}
    f(x) + \frac{1}{2\eta} d(x, y)^{2} - f(y) + \frac{C_{\varepsilon}}{2} \ge \frac{1}{2t} d(x, y)^{2}.
\end{align*} 
Consequently, the acceptance rate defined by 
    \begin{align*}
        V(x) := \frac{\exp(- f(x) + f(y) - \frac{1}{2\eta} d(x, y)^{2} - \frac{C_{\varepsilon}}{2})}{\exp(-\frac{1}{2t} d(x, y)^{2})},
    \end{align*}
is guaranteed to be bounded by $1$. Then, in Proposition \ref{Prop_expected_rej} we show that the expected number of rejections is $\mathcal{O}(1)$ in dimension $d$ and step size $\eta$.

\begin{proposition}\label{prortt}
    Let $f$ be $L_{1}$-Lipschitz and $C_{\varepsilon}$ be some constant. Take $\eta = \frac{C_{\varepsilon}}{L_{1}^{2} d}$.
    With $T = \frac{C_{\varepsilon}}{L_{1}^{2}(d+1)}$ and $t = \frac{C_{\varepsilon}}{L_{1}^{2}(d-1)}$,
    it holds that 
    \begin{align*}
        \frac{1}{2T} d(x, y)^{2} + C_{\varepsilon}
        \ge f(x) + \frac{1}{2\eta} d(x, y)^{2} - f(y) + \frac{C_{\varepsilon}}{2}
        \ge \frac{1}{2t} d(x, y)^{2}.
    \end{align*}
    Consequently, the acceptance rate is bounded by $1$, i.e., $V(x) \le 1, \forall x \in M$.
\end{proposition}

\begin{proof}[Proof of Proposition~\ref{prortt}]    Since $f$ is $L_{1}$-Lipschitz, we have $\|\grad f(x)\| \le L_{1}$. 
    Then we have $L_{1} d(x, y) \ge f(x) - f(y) \ge -L_{1} d(x, y)$.

    \begin{enumerate}
        \item \textbf{The lower bound: } The goal is to find some $t > 0$ and constant $C$ such that 
        \begin{align*}
            f(x) + \frac{1}{2\eta} d(x, y)^{2} - f(y) + C \ge \frac{1}{2t} d(x, y)^{2}.
        \end{align*} 
        It suffices to find $t, C$ such that
        \begin{align*}
            \frac{1}{2\eta} d(x, y)^{2} - \frac{1}{2t} d(x, y)^{2} - L_{1} d(x, y) + C \ge 0.
        \end{align*}
        The left hand side can be viewed as a quadratic function of $d(x, y)$. 
        When $d(x, y) = \frac{L_{1}}{\frac{1}{\eta} - \frac{1}{t}}$, 
        the left hand side is minimized, and the mimimum is 
        $- \frac{1}{2}\frac{L_{1}^{2}}{\frac{1}{\eta} - \frac{1}{t}} + C$.
        Hence we can take $C = \frac{1}{2}\frac{L_{1}^{2}}{\frac{1}{\eta} - \frac{1}{t}}$. Take $\eta = \frac{C_{\varepsilon}}{L_{1}^{2} d}$ and $t = \frac{C_{\varepsilon}}{L_{1}^{2} (d-1)}$.
        Then we have $C = \frac{1}{2}\frac{L_{1}^{2}}{\frac{1}{\eta} - \frac{1}{t}} = \frac{C_{\varepsilon}}{2}$.    
        \item \textbf{The upper bound: } For an upper bound, we want some $T \le \eta$ for which we want to  show that
        \begin{align*}
            f(x) + \frac{1}{2\eta} d(x, y)^{2} - f(y) - \frac{C_{\varepsilon}}{2}
            \le L_{1} d(x, y) + \frac{1}{2\eta} d(x, y)^{2} - \frac{C_{\varepsilon}}{2}
            \le \frac{1}{2T} d(x, y)^{2}.
        \end{align*}
        Similar as before, it suffices to show 
        \begin{align*}
            (\frac{1}{2\eta} - \frac{1}{2T}) d(x, y)^{2} + L_{1} d(x, y) - \frac{C_{\varepsilon}}{2} 
            \le 0.
        \end{align*}
        The left hand side is maximized at $d(x, y) = \frac{L_{1}}{\frac{1}{T} - \frac{1}{\eta}}$, 
        with maximum $\frac{1}{2} \frac{L_{1}^{2}}{\frac{1}{T} - \frac{1}{\eta}} - \frac{C_{\varepsilon}}{2} $.
        Take $T = \frac{C_{\varepsilon}}{L_{1}^{2}(d+1)}$. We can then verify that 
        \begin{align*}
            \frac{1}{2} \frac{L_{1}^{2}}{\frac{1}{T} - \frac{1}{\eta}} - \frac{C_{\varepsilon}}{2} 
            = \frac{1}{2} \frac{L_{1}^{2}}{L_{1}^{2}/C_{\varepsilon}} - \frac{C_{\varepsilon}}{2} = 0.
        \end{align*}    
        \item \textbf{Combining the two steps: } 
        From the above two steps, we get 
        \begin{align*}
            \frac{1}{2T} d(x, y)^{2} + C_{\varepsilon}
            \ge f(x) + \frac{1}{2\eta} d(x, y)^{2} - f(y) + \frac{C_{\varepsilon}}{2}
            \ge \frac{1}{2t} d(x, y)^{2}.
        \end{align*}
    \end{enumerate}
\end{proof}

In the following proposition, we show that on a hypersphere (where the Riemannian metric in normal coordinates is well studied), the expected number of rejections which equals to 
$$
\frac{\int_{M} \exp(-\frac{1}{2t} d(x, y)^{2}) dV_{g}(x)}{\int_{M} \exp(- f(x) + f(y) - \frac{1}{2\eta} d(x, y)^{2} - \frac{C_{\varepsilon}}{2}) dV_{g}(x)},
$$ 
which is independent of dimension and accuracy.

\begin{proposition}\label{Prop_expected_rej}
    Let $M$ be hypersphere. 
    Set $C_{\varepsilon} = \frac{1}{\log \frac{1}{\varepsilon}}$.
    Assume without loss of generality that $L_{1} \ge \max\{1, \frac{d+1}{\sqrt{6}}\}$. Then with $\eta = \frac{C_{\varepsilon}}{L_{1}^{2} d}$ and $t = \frac{C_{\varepsilon}}{L_{1}^{2}(d-1)}$, for small $\varepsilon$,
    the expected number of rejections is $\mathcal{O}(1)$ in both dimension and $\varepsilon$.
\end{proposition}
\begin{proof}
    Let $T = \frac{C_{\varepsilon}}{L_{1}^{2}(d+1)}$.
    We try to bound the expected number of rejections. 
    We compute it as follows:
    \begin{align*}
            \frac{\int_{M} \exp(-\frac{1}{2t} d(x, y)^{2}) dV_{g}(x)}{\int_{M} \exp(- f(x) + f(y) - \frac{1}{2\eta} d(x, y)^{2} - \frac{C_{\varepsilon}}{2}) dV_{g}(x)}
            \le 
            \frac{\int_{M} \exp(-\frac{1}{2t} d(x, y)^{2}) dV_{g}(x)}{\int_{M} \exp(- \frac{1}{2T} d(x, y)^{2} - C_{\varepsilon}) dV_{g}(x)}.
    \end{align*}

    Using \citet[Lemma 8.2]{li2023riemannian} and \citet[Lemma C.5]{li2023riemannian}, 
    when $\beta \ge \frac{d}{R^{2}}$, using Riemannian normal coordinates we have the following lower bound on the integral: 
    \begin{align*}
            \int_{M} \exp(-\frac{\beta}{2} d(x, y)^{2}) dV_{g}(x) 
            &\ge \int_{B_{y}(R)} \exp(-\frac{\beta}{2} d(x, y)^{2}) dV_{g}(x) \\
            &\ge (\frac{\sin R}{R})^{d-1} (\frac{2\pi}{\beta})^{\frac{d}{2}}(1 - \exp(-\frac{1}{2}( \beta R^{2} - d))),
    \end{align*}
    where $B_{y}(R)$ denote the geodesic ball centered at $y$ with radius $R$.
    
    On the other hand, we have 
    \begin{align*}
        \int_{M} \exp(-\frac{t}{2} d(x, y)^{2}) dV_{g}(x)
            \le \int_{B_{\pi}(0)} \exp(-\frac{t}{2} |x|^{2}) dx
            \le (\frac{2\pi}{t})^{\frac{d}{2}}.
    \end{align*}

    We next find a suitably small $R$ which only depends on dimension, for which we have $\frac{R}{\sin R} \le 1 + \frac{1}{d}$.
    Using Taylor series for $\sin(R)$, we have $\frac{R}{\sin R} \approx \frac{R}{R-\frac{R^{3}}{6}}$. Hence for $ R^{2} \le \frac{6}{1+d}$, we have (approximately) $\frac{R}{\sin R} \le 1 + \frac{1}{d}$. Consequently we set $R = \sqrt{\frac{6}{1+d}}$, and we know $(\frac{R}{\sin R})^{d-1} = \mathcal{O}(1)$.
    
    Combining the bounds discussed previously, we have 
    \begin{align*}
            &\frac{\int_{M} \exp(-\frac{1}{2t} d(x, y)^{2}) dV_{g}(x)}
            {\int_{M} \exp(-\frac{1}{2T} d(x, y)^{2} - C_{\varepsilon})dV_{g}(x)}
            \\
            \le& e^{C_{\varepsilon}} (\frac{R}{\sin R})^{d-1}
            \frac{(2\pi t)^{\frac{d}{2}}}{ (2\pi T)^{\frac{d}{2}}
            (1 - \exp(-\frac{1}{2}( \frac{L_{1}^{2} (d+1)}{C_{\varepsilon}} R^{2} - d)))} \\
            \le& e^{C_{\varepsilon}+1} (\frac{t}{T})^{\frac{d}{2}}
            \frac{1}{ 1 - \exp(-\frac{1}{2}( \frac{L_{1}^{2} (d+1)}{C_{\varepsilon}}R^{2} - d))}
            = e^{C_{\varepsilon}+1}(\frac{d+1}{d-1})^{\frac{d}{2}}
            \frac{1}{ 1 - \exp(-\frac{1}{2}( \frac{L_{1}^{2} (d+1)}{C_{\varepsilon}}R^{2} - d))}.
    \end{align*}

    For small $\varepsilon$, we have $C_{\varepsilon} \le 1$.
    Since we assumed $L_{1} \ge 1$ and $L_{1}^{2} \ge \frac{d+1}{6}$, we have 
    $1 - \exp(-\frac{1}{2}( \frac{L_{1}^{2} (d+1)}{C_{\varepsilon}}R^{2} - d)) \ge 1 - \exp(-\frac{1}{2}( \frac{(d+1)^{2}}{6}\frac{6}{d+1} - d)) \ge 1 - \exp(-\frac{1}{2})$. As a result, we see that the expect number of rejections is of order $\mathcal{O}(1)$:
    \begin{align*}
            &\frac{\int_{M} \exp(-\frac{1}{2t} d(x, y)^{2}) dV_{g}(x)}
            {\int_{M} \exp(-\frac{1}{2T} d(x, y)^{2} - 1)dV_{g}(x)}
            \le 
            \frac{e^{2}}{ 1 - \exp(-\frac{1}{2})}(\frac{d+1}{d-1})^{\frac{d}{2}}
            \le 20 (\frac{d+1}{d-1})^{\frac{d}{2}}.
    \end{align*}
    Observe that $(\frac{d+1}{d-1})^{\frac{d}{2}} = (1 + \frac{1}{(d-1)/2})^{(\frac{d-1}{2} + \frac{1}{2})} = \mathcal{O}(1)$.

\end{proof}

\end{document}